%% file: main.tex
\documentclass[11pt]{article}
\usepackage[utf8]{inputenc}

\usepackage{mystyle}
\usepackage[top=2.5cm, bottom=2.5cm, left=2.5cm, right=2.5cm]{geometry}

\title{Ultra-fast feature learning for the training of two-layer neural networks in the two-timescale regime}
\author{
    Raphaël Barboni\\
    ENS -- PSL Université \\
    \texttt{raphael.barboni@ens.fr}
    \and
    Gabriel Peyré\\
    CNRS and ENS -- PSL Univ.\\
    \texttt{gabriel.peyre@ens.fr}
    \and
    Fran\c{c}ois-Xavier Vialard\\
    LIGM, Univ. Gustave Eiffel, CNRS\\
    \texttt{francois-xavier.vialard@univ-eiffel.fr}
}
\date{}

\usepackage[backend=bibtex]{biblatex}
\newcommand{\fmap}{\phi}
\newcommand{\Fmap}{\Phi}
\newcommand{\TV}{\mathrm{TV}}

\newcommand{\KL}{\mathrm{KL}}

\newcommand{\MMD}{\mathrm{MMD}}
\newcommand{\diam}{\mathrm{diam}}
\newcommand{\activation}{\sigma}

\newcommand{\relu}{\mathrm{ReLU}}
\newcommand{\resnet}{\mathrm{ResNet}}
\newcommand{\m}{\mathrm{m}}

\begin{document}

\maketitle

\begin{abstract}

We study the convergence of gradient methods for the training of \emph{mean-field} single-hidden-layer neural networks with square loss.
For this high-dimensional and non-convex optimization problem, most known convergence results are either qualitative or rely on a \emph{neural tangent kernel} analysis where nonlinear representations of the data are fixed.
Using that this problem belongs to the class of separable nonlinear least squares problems, we consider here a \emph{Variable Projection (VarPro)} or \emph{two-timescale learning} algorithm, thereby eliminating the linear variables and reducing the learning problem to the training of nonlinear features.
In a teacher-student scenario, we show such a strategy enables provable convergence rates for the sampling of a teacher feature distribution.
Precisely, in the limit where the regularization strength vanishes, we show that the dynamic of the feature distribution corresponds to a \emph{weighted ultra-fast diffusion equation}.
Recent results on the asymptotic behavior of such PDEs then give quantitative guarantees for the convergence of the learned feature distribution.

\end{abstract}

\section{Introduction}

Machine learning methods based on \emph{artificial neural networks} have recently experienced a significant increase in popularity due to their efficiency in solving numerous supervised or unsupervised learning tasks.
This success owes to their capacity to perform \emph{feature learning}, that is to extract meaningful representations from the data during the training process~\cite[Chap. 15]{Goodfellow-et-al-2016}, standing in contrast with \emph{kernel methods} for which feature representations are designed by hand and fixed during training~\cite{hofmann2008kernel}.
Feature learning is believed to play a fundamental role in the generalization performance of neural networks.
For example, adaptivity to low-dimensional representations of the data can prevent the \emph{curse of dimensionality}~\cite{bach2017breaking,ghorbani2020neural}.

However, the process through which features are learned remains largely misunderstood.
Indeed, adaptivity of the representations comes in neural networks at the price of a nonlinear parameterization, making the training dynamic more difficult to analyze.
Specifically, for \emph{overparameterized} neural network architectures where the dimension of the parameter space greatly exceeds the number of training samples, recent works have put forward the crucial role played by the choice of scaling w.r.t.\@ the number of parameters in the training dynamic~\cite{chizat2019lazy,liu2020linearity,yang2021tensor}.
For single-hidden-layer neural networks, the ``kernel regime'', corresponding to a scaling of $1/ \sqrt{M}$ where $M$ is the width, has been identified as a scaling for which the model is well-approximated by its linearization around initialization, therefore reducing to a kernel method~\cite{jacot_neural_2021}.
Relying on the good conditioning of the ``Neural Tangent Kernel (NTK)'', this regime provides convergence of gradient descent towards a global minimizer of the risk at a linear rate~\cite{allen2019convergence,du2019gradient,lee2019wide,zou_gradient_2020}.
However, this regime has also been shown to suffer from a ``lazy training'' behavior preventing significant modification of the feature distribution and associated to poor generalization guarantees~\cite{chizat2019lazy}.

In contrast, another line of work has been focused on the ``mean field'' regime corresponding to a scaling of $1/M$ (which we consider in \cref{eq:SHL}) for which the neural network is parameterized by a probability distribution over the space of weights~\cite{chizat2018global,mei2019mean,rotskoff2019global,sirignano2020mean}.
While such a choice of scaling has been shown to enable nonlinear feature learning behaviors~\cite{yang2021tensor}, existing convergence results are primarily qualitative, lacking explicit convergence rates.
To bridge this gap, we are interested in this work in the dynamic of the feature distribution in the training of mean-field models of shallow neural network architectures.
We study more particularly a \emph{variable projection} or \emph{two-timescale} learning strategy which allows reducing the learning problem to the training of the feature distribution.

\input{subsec1.1_alt}

\subsection{Contributions and related works}

\paragraph{Contributions}

This paper studies the convergence of the VarPro algorithm --- or two-timescale regime of gradient descent --- for the training of mean-field models of neural networks.
Precisely, we study the dynamic of the feature distribution $\mu \in \Pp(\Om)$ when trained with gradient flow for the minimization of the reduced risk $\Ll^\lambda$, for $\lambda \geq 0$.
In the teacher-student scenario defined by~\cref{ass:teacher_student}, we establish guarantees for the convergence of $\mu$ towards the teacher feature distribution~$\Bar{\mu}$:
\begin{itemize}
    \item In the case $\lambda = 0$, we show in~\cref{sec:training} that the training dynamic corresponds to an \emph{ultra-fast diffusion} equation.
    Relying on the work of~\cite{iacobelli2019weighted}, this allows stating convergence towards the teacher feature distribution $\Bar{\mu}$ (\cref{thm:diffusion_convergence}), with a linear convergence rate.
    \item At fixed $\lambda > 0$, we establish in~\cref{thm:algebraic_convergence} convergence of $\mu$ towards the teacher feature distribution $\Bar{\mu}$ with an algebraic rate.
    \item In the limit $\lambda \to 0^+$, we show that, under regularity assumptions, the dynamic of the feature distribution $\mu$ converges locally uniformly in time to the solution of the \emph{ultra-fast diffusion} equation with weights $\Bar{\mu}$ (\cref{thm:gradient_flow_approximation}).
    \item Finally, we show in~\cref{sec:numerics} that numerical results on low-dimensional learning problems with synthetic data are well-aligned with our theory.
    Overall, these experiments indicate that, when the regularization is sufficiently low, the VarPro dynamic indeed enters an ``ultra-fast diffusion regime'' where the student feature distribution converges to the teacher's at a linear rate.
    We also show with experiments on CIFAR10 that the VarPro algorithm can be adapted to the
    training of more complex architectures such as ResNets and achieves generalization on supervised learning problems with large datasets.
\end{itemize}

\paragraph{Convergence analysis for the training mean-field neural networks}

Several works have studied the convergence of gradient based methods for the training of neural network models similar to~\cref{eq:SHL} with the mean-field scaling $\frac{1}{M}$.
In \cite{bach2021gradient}, the authors show that, for two layer neural networks with a homogeneous activation, if gradient flow on the weights distribution converges then it converges towards a global minimizer of the risk.
In~\cite{rotskoff2019global}, a similar result is shown for a modification of the gradient flow dynamic where a supplementary ``birth-death'' term is added.

Several works have also analyzed the convergence of noisy gradient descent, or \emph{Langevin dynamic}, for the training of mean-field models of two layer neural networks~\cite{chizat2022mean,mei2019mean,nitanda2022convex,hu2021mean,suzuki2023feature}.
Thanks to the addition of an entropic regularization term, these works provide a convergence rate for the sampling of an invariant weight distribution.

\paragraph{Two-timescale learning}

While two-timescale learning strategies have a broad range of applications in the fields of stochastic approximation and optimization~\cite{borkar1997stochastic,borkar2008stochastic}, there has been a recent interest in these methods for the training of neural networks~\cite{marion2023leveraging,berthier2024learning,wang2024mean,bietti2023learning,takakura2024mean}.
Specifically, \cite{berthier2024learning} studies the training of two-layer neural networks and exhibits a separation of timescales and different learning phases whose respective sizes depend on the timescale parameter $\eta$.
In~\cite{marion2023leveraging}, the authors study two-timescale gradient descent for a simple model of $1$-dimensional neural network and show that the teacher network is recovered as soon as both the number of neurons of the student and the timescale parameter are sufficiently large.
In~\cite{bietti2023learning}, a multi-index regression problem is considered.
Relying on the assumption of high dimensional Gaussian data, the authors consider a linear layer composed with a  nonparameteric model whose projection can be computed in the Hermite basis.
They show this VarPro algorithm results in a saddle-to-saddle dynamic on the linear layer and establish guarantees for the recovery of the teacher model.

Finally, \cite{takakura2024mean,wang2024mean} study the training of mean-field models of neural networks in the two-timescale limit with noisy gradient descent.
In contrast with these works, we do not consider here additional entropic or $L^2$-regularization on the feature weights.

\paragraph{Wasserstein gradient flows of statistical distances}

Under our~\cref{ass:teacher_student}, \cref{eq:L_mmd} shows $\Ll^\lambda$ is an infimal convolution of statistical divergences between the feature distribution $\mu$ and the teacher $\Bar{\nu}$, interpolating between the $\chi^2$-divergence $\chi^2(\Bar{\nu}|\mu)$ --- or more generally a $f$-divergence $\D_f(\Bar{\nu}|\mu)$ --- when $\lambda \to 0^+$ and a (squared) kernel discrepancy $\MMD(\Bar{\nu}, \mu)^2$ when $\lambda \to \infty$.
In the case $\lambda \to \infty$, gradient flows of $\MMD$-discrepancies and applications to sampling were studied in several works~\cite{arbel2019maximum,sejdinovic2013equivalence,hertrich2023wasserstein,hertrich2023generative,hertrich2024wasserstein,boufadene2023global}.
Those flows are known to get trapped in local minima but discrepancies associated to non-smooth kernels have been observed to behave better in terms of convergence~\cite{hertrich2023generative,hertrich2024wasserstein}.
In the case of the coulomb kernel, \cite{boufadene2023global} proves that the discrepancy loss admits no spurious local minima and that the discrepancy flow converges towards the target measure under regularity assumptions. 

In the intermediate regime $\lambda \in (0, \infty)$, several other works have also proposed regularization of $f$-divergences based on the infimal convolution with a kernel distance.
In~\cite{glaser2021kale}, the \emph{KL Approximate Lower bound Estimator (KALE)} kernelizes the variational formulation of the $\KL$-divergence and in \cite{chen2024regularized} the \emph{(De)-regularized Maximum Mean Discrepancy (DrMMD)} kernelizes the $\chi^2$-distance.
More generally, the work of~\textcite{neumayer2024wasserstein} studied kernelized variational formulations --- or ``Moreau envelopes in a RKHS'' --- of $f$-divergences.
Similar to our~\cref{lem:gamma}, they showed $\Gamma$-convergence of these functionals towards the generating $f$-divergence when the regularization parameter $\lambda$ tends to $0$.
They also studied numerically the convergence of the associated Wasserstein gradient flow towards the target distribution.
The most notable difference between these regularized distances and the functional $\Ll^\lambda$ appearing in this work is that (w.r.t.\@ \cite[eq. (14)]{neumayer2024wasserstein}) the role of the target $\Bar{\nu}$ and parameter $\mu$, over which optimization is performed, are interchanged.
In other words, we consider optimizing over a statistical discrepancy which is the ``reverse'' of the one considered in~\cite{neumayer2024wasserstein} and for this reason, though the mathematical tools to analyze it might be similar, the gradient flow dynamics will a priori have different behaviors.

\subsection{Mathematical preliminaries and notations}

In the following, $\Omega$ will either be the $n$-dimensional torus or a closed bounded convex domain of $\RR^n$, for some $n \geq 1$.
We denote by $\Mm(\Om)$ the set of finite Borel measures over $\Om$ and by $\Pp(\Om)$ the subset of $\Mm(\Om)$ consisting of probability measures.
We will denote by $\pi \in \Pp(\Om)$ the uniform distribution over $\Om$.
For a measure $\nu \in \Mm(\Om)$, $|\nu|$ is its total variation measure and $\| \nu \|_\TV$ is the total variation of $\nu$.
For $p \in [1, +\infty)$, we denote by $\Ww_p$ the Wasserstein-$p$ distance defined for two probability measures $\mu, \mu' \in \Pp(\Om)$ by:
\begin{align*}
    \Ww_p(\mu, \mu') \eqdef \min_{\gamma \in \Gamma(\mu, \mu')} \left( \int_{\Om \times \Om} \| \om-\om' \|^2 \d \gamma(\om, \om') \right)^{1/p} \, ,
\end{align*}
where $\Gamma(\mu, \mu') \subset \Pp(\Om \times \Om)$ is the set of couplings between $\mu$ and $\mu'$.
Standard references on the properties of the Wasserstein distance are the textbooks of~\textcite{villani2009optimal} and~\textcite{santambrogio2015optimal}.
If not otherwise specified, $\Mm(\Om)$ and $\Pp(\Om)$ are endowed with the topology of \emph{narrow} convergence, that is the weak-* topology of $\Mm(\Om)$ in duality with continuous functions.
Importantly, because $\Om$ is compact, this topology on $\Pp(\Om)$ is equivalent to the $\Ww_p$-topology for any $p \in [1,+\infty)$ and $\Pp(\Om)$ is compact.

For an integer $k \geq 0$ and for $s \in (0,1]$, we denote by $\Cc^{k,s}(\Om)$ (or just $\Cc^{k,s}$) the Hölder space of $k$-times continuously differentiable real-valued functions over $\Om$ with $s$-Hölder $k^{\text{th}}$-derivative.
We denote by $\| . \|_{\Cc^{k,s}}$ the Hölder norm on $\Cc^{k,s}(\Om)$.
For a probability measure $\rho \in \Pp(\RR^d)$ and $p \in [1, +\infty]$, we denote by $L^p(\rho, \Cc^{k,s})$ the space of measurable functions $\fmap : \Om \times \RR^d \to \RR$ s.t.\@ $\fmap(.,x) \in \Cc^{k,s}(\Om)$ for $\d \rho$-a.e.\@ $x \in \Om$ and $\| \fmap \|_{L^p(\rho, \Cc^{k,s})} \eqdef \left( \int_{\RR^d} \| \fmap(., x ) \|^p_{\Cc^{k,s}} \d \rho(x) \right)^{1/p} < + \infty$. 
We will often use that if $\fmap \in L^2(\rho, \Cc^{k, s})$ and $\alpha \in L^2(\rho)$ then the Bochner integral $\int_{\RR^d} \fmap(., x) \alpha(x) \d \rho(x)$ is in $\Cc^{k,s}$ with:
\begin{align*}
    \left\| \int_{\RR^d} \fmap(., x) \alpha(x) \d \rho(x) \right\|_{\Cc^{k,s}} \leq \| \fmap \|_{L^2(\rho, \Cc^{k,s})} \| \alpha \|_{L^2(\rho)}. 
\end{align*}

\section{Reduced risk associated to the VarPro algorithm}

We study in this work a VarPro algorithm or two-timescale regime of gradient descent for the training of neural networks.
This strategy amounts to performing gradient descent on the \emph{reduced risk} defined as the result of a partial minimization on a regularized version of the risk.

\subsection{Primal formulation of the reduced risk}

Whereas regularizing the risk with the Euclidean square norm of the weights is a popular practice, the variable projection procedure can be used with other kinds of regularization.
Generally, for a convex function $f : \RR \to \RR$ and a regularization strength $\lambda > 0$ we consider for $\mu \in \Pp(\Om)$ and $u \in L^1(\mu)$:
\begin{align} \label{eq:R_lambda_f}
    \Rr^\lambda_f(\mu, u)
    \eqdef \frac{1}{2} \left\| F_{\mu,u} - Y \right\|^2_{L^2(\rho)} + \lambda \int_\Om f(u) \d \mu
    = \frac{1}{2} \left\| \Fmap_\mu \cdot u - Y \right\|^2_{L^2(\rho)} + \lambda \int_\Om f(u) \d \mu,
\end{align}
where we assume $\Rr^\lambda_f(\mu, u) = +\infty$ if $f(u)$ is not integrable w.r.t.\@ $\mu$.
As before we consider the \emph{reduced risk} obtained by minimizing $\Rr^\lambda_f$ w.r.t.\@ the outer weights $u$.
For every $\mu \in \Pp(\Om)$ we define:
\begin{align} \label{eq:L_lambda_f}
    \Ll^\lambda_f(\mu) \eqdef \min_{u \in L^1(\mu)} \frac{1}{\lambda} \Rr^\lambda_f(\mu,u) = \min_{u \in L^1(\mu)} \frac{1}{2 \lambda} \left\| \Fmap_\mu \cdot u - Y \right\|^2_{L^2(\rho)} + \int_\Om f(u) \d \mu
\end{align}
and this definition extends to the limiting case $\lambda \to 0^+$ by considering:
\begin{align} \label{eq:L0_f}
    \Ll^0_f(\mu) \eqdef \min_{\Fmap_\mu \cdot u = Y} \int_\Om f(u) \d \mu.
\end{align}

In the following, we always assume that $\fmap \in L^2(\rho, \Cc^0(\Om))$ (\cref{ass:feature_map}).
This in particular implies that, for any $\mu \in \Pp(\Om)$, the map $\Fmap_\mu : L^1(\mu) \to L^2(\rho)$ is weakly continuous.
We also consider the following assumption on the regularization function:
\begin{assumption} \label{ass:regularization}
    The function $f : \RR \to \RR \cup {+\infty}$ is nonnegative, strictly convex and superlinear i.e. such that $\lim_{\pm \infty} \frac{f(t)}{|t|} = + \infty$.
\end{assumption}

\noindent
By~\cref{lem:integral_functional}, this is sufficient to ensure the existence of a unique minimizer $u^\lambda_f[\mu] \in \argmin \Rr^\lambda_f(\mu, u)$, when $\lambda >0$, and $u^0_f[\mu] \in \argmin_{\Fmap_\mu \cdot u = Y} \int_\Om f(u) \d \mu$, when $\lambda = 0$.
Of particular interest in this work and more precisely in~\cref{sec:training} is the case where $f(t) = |t|^r /(r-1)$ for some $r > 1$.
In this case we denote the corresponding reduced risk by $\Ll^\lambda_r$.
In particular, for $r = 2$ we recover the ``$L^2$-regularized'' reduced risk defined in~\cref{eq:L_lambda_min} and~\cref{eq:L0}.

\begin{lem} \label{lem:integral_functional}
    Assume~\cref{ass:regularization} holds. Then, for every  $\mu \in \Pp(\Om)$, the functional
    $$\Ii_f : u \in L^1(\mu) \mapsto \int_\Om f(u) \d \mu$$ 
    is strictly convex, weakly lower semicontinuous and has weakly compact sublevel sets. In particular, \cref{eq:L_lambda_f} (and \cref{eq:L0_f} if feasible) admits a unique minimizer $u^\lambda_f[\mu]$.
\end{lem}

\begin{proof}
    Clearly $\Ii_f$ is strictly convex.
    Weak lower semicontinuity is a classical consequence of the fact that $\Ii_f$ is convex and strongly lower semicontinuous (using Fatou's lemma), hence its epigraph is convex and strongly closed and hence also weakly closed.
    For weak compacity of sublevel sets, if $(u_n)_{n\geq 0}$ is a sequence s.t.\@ $\int_\Om f(u_n) \d \mu \leq C$ for every $n\geq 0$, then using that $f$ has super-linear growth, for every $\eps > 0$ there exists a $T \geq 0$ s.t.\@ $|t| \leq \eps f(t)$ for every $ |t| \geq T$ and for every $n \geq 0$:
    \begin{align*}
        \int_{|u_n| \geq T} |u_n| \d \mu \leq \eps \int f(u_n) \d \mu \leq \eps C.
    \end{align*}
    Thus the sequence $(u_n)_{n \geq 0}$ is uniformly integrable and admits a weakly converging subsequence by Dunford-Pettis theorem.
\end{proof}

\subsection{Partial minimization on the space of measures}

The reduced risk can also be obtained as the result of partial minimization of a convex functional over the space of measures.
Whereas we have previously separated the role of the outer weights $u$ and of the feature distribution $\mu$ in~\cref{eq:SHL_mean_field}, our neural network model can equivalently be seen as a linear operator acting on the space $\Mm(\Om)$ of finite measures on $\Om$.

For $\mu \in \Pp(\Om)$ and $u \in L^1(\mu)$, we have by definition of $\Fmap \star$ in~\cref{eq:Fmap_star} and of $\Fmap_\mu$ in~\cref{eq:feature_operator} that $\fmap_\mu \cdot u = \Fmap \star \nu$ where $\nu \in \Mm(\Om)$ is s.t.\@ $\d \nu = u \d \mu$.
Also, $\int_\Om f(u) \d \mu = \int_\Om f(\frac{\d \nu}{\d \mu}) \d \mu = \D_f(\nu | \mu)$  where, for $f$ satisfying~\cref{ass:regularization}, $\D_f$ is the divergence defined by:
\begin{align} \label{eq:divergence}
    \forall (\nu, \mu) \in \Mm(\Om) \times \Pp(\Om), \quad 
    \D_f(\nu | \mu) \eqdef
    \left\{
    \begin{array}{cc}
        \int_\Om f(\frac{\d \nu}{\d \mu}) \d \mu & \text{if $\nu \ll \mu$},  \\
        +\infty & \text{otherwise}. 
    \end{array}
    \right.
\end{align}
In particular, in the case where $f$ is an \emph{entropy function} and $\nu \in \Pp(\Om)$ is a probability measure, $\D_f(\nu|\mu)$ is the standard Csiszàr $f$-divergence~\cite{liero2018optimal}.
Performing a change of variable, one can thus define the functional $\Ll^\lambda_f$ as the value resulting from a minimization problem over the space of measures.
For $\mu \in \Pp(\Om)$, minimizing over $\nu \in \Mm(\Om)$ instead of $u \in L^1(\mu)$, we get:
\begin{align} \label{eq:L_measure}
    \Ll_f^\lambda(\mu) =
    \left\{
    \begin{array}{cc}
        \min\limits_{\nu \in \Mm(\Om)} \frac{1}{2 \lambda} \| \Fmap \star \nu - Y \|^2 + \D_f(\nu | \mu) & \text{if $\lambda > 0$,} \\[10pt]
        \min\limits_{\nu \in \Mm(\Om)} \iota_{\Fmap \star \nu = Y} + \D_f(\nu | \mu) & \text{if $\lambda = 0$.}
    \end{array}
    \right.
\end{align}

As presented in~\cref{ass:teacher_student}, of particular interest  is the case where the signal $Y$ itself can be exactly represented by a neural network, that is $Y = \Fmap \star \Bar{\nu}$, for some $\Bar{\nu} \in \Mm(\Om)$.
Then in the case $\lambda = 0$, using the injectivity of $\Fmap \star$, $\Bar{\nu}$ is the only feasible solution in~\cref{eq:L_measure} and we obtain:
\begin{align} \label{eq:L0_div}
	\Ll^0_f(\mu) = \int_\Om f(\frac{\d \Bar{\nu}}{\d \mu}) \d \mu =  \D_f(\Bar{\nu} | \mu) \,.
\end{align}
In the case $\lambda > 0$, $\Ll^\lambda_f$ can be interpreted as the infimal convolution between a \emph{Maximum Mean Discrepancy (MMD)} and the divergence $\D_f$.
Indeed, naturally associated to the data distribution $\rho \in \Pp(\RR^d)$ and to the feature map $\fmap$ is a structure of Reproducing Kernel Hilbert Space (RKHS) of functions on $\Om$.
We refer to~\cref{sec:rkhs} for results on the theory of RKHSs we use in this chapter.
The RKHS $\Hh$ is defined in~\cref{eq:RKHS_characterization}, and corresponds to the kernel $\kappa :  \Om \times \Om \to \RR$  defined by:
\begin{align*}
    \forall \om, \om' \in \Om, \quad \kappa (\om, \om') \eqdef \int_\Om \fmap(\om, x) \fmap (\om', x) \d \rho(x) .
\end{align*}
It then follows from the definition of $\kappa$ and $\Hh$ that, under~\cref{ass:teacher_student}, the data attachment term in~\cref{eq:L_measure} can be interpreted as a kernel distance between $\nu$ and $\Bar{\nu}$.
By~\cref{eq:MMD} we have $\| \Fmap \star (\nu - \Bar{\nu}) \|_{L^2(\rho)} = \MMD_\kappa (\nu, \Bar{\nu})$ where $\MMD_\kappa$ is the \emph{Maximum Mean Discrepancy (MMD)} with kernel $\kappa$~\cite{muandet2017kernel,gretton2012kernel}.
For $\lambda > 0$, the functional $\Ll^\lambda_f$ can then be expressed for every $\mu \in \Pp(\Om)$ as:
\begin{align} \label{eq:L_mmd}
    \Ll^\lambda_f(\mu) = \min_{\nu \in \Mm(\Om)} \frac{1}{2 \lambda} \MMD_\kappa^2(\nu, \Bar{\nu}) + \D_f(\nu | \mu).
\end{align}
This last formulation of the functional $\Ll^\lambda_f$ resembles the notion of \emph{Moreau envelope in a RKHS} of the divergence $\D_f$ introduced by~\textcite{neumayer2024wasserstein}.
This notion encompasses the particular cases of \emph{De-regularized MMD} studied in~\cite{chen2024regularized} and \emph{KL Approximate Lower bound Estimator} studied in~\cite{glaser2021kale}.
Nonetheless, w.r.t.\@ \cite[eq. (14)]{neumayer2024wasserstein}, the role of the target measure $\Bar{\nu}$ and of the optimized measure $\mu$ are here interchanged, which is expected to play an important role in the gradient flow dynamic.

\subsection{Dual formulation of the reduced risk}

In~\cref{eq:L_lambda_f,eq:L0_f}, the objectives $\Ll^\lambda_f$ and $\Ll_f^0$ are expressed as the value of a minimization problem over the outer weights $u$. Taking the dual of those minimization problems, $\Ll_f^\lambda$ and $\Ll^0_f$ can be expressed as the value of a maximization problem over the dual variable $\alpha \in L^2(\rho)$.
In contrast with the primal formulation \cref{eq:L_lambda_f}, the dual formulation of~\cref{prop:duality} has the advantage of conveniently expressing $\mathcal{L}^\lambda_f$ for both $\lambda > 0$ and $\lambda = 0$ as the value of an optimization problem over the space $L^2(\rho)$ which is independent of $\mu$.

\begin{prop}[Dual representation] \label{prop:duality}
    Let~\cref{ass:regularization} hold and consider $\mu \in \Pp(\Om)$.
    Then we have for $\lambda > 0$:
    \begin{align} \label{eq:L_dual}
        \Ll^\lambda_f(\mu) = \max_{\alpha \in L^2(\rho)} - \int_\Om f^*(\Fmap^\top \alpha) \d \mu + \left< \alpha, Y \right>_{L^2(\rho)} - \frac{\lambda}{2} \| \alpha \|^2_{L^2(\rho)} ,
    \end{align}
    where $f^*$ is the Legendre transform of $f$ and $\Fmap^\top : L^2(\rho) \to \Cc^0(\Om)$ is defined by:
    \begin{align*}
        \forall \alpha \in L^2(\rho), \quad \Fmap^\top \alpha \eqdef \int_{\RR^d} \fmap(.,x) \alpha(x) \d \rho(x) .
    \end{align*}
    The supremum in~\cref{eq:L_dual} is attained at some $\alpha^\lambda_f[\mu] \in L^2(\rho)$ and for $u^\lambda_f[\mu] \in L^1(\mu)$ the optimizer in~\cref{eq:L_lambda_f} it holds:
    \begin{align} \label{eq:duality}
        \lambda \alpha^\lambda_f[\mu] = \Fmap_\mu \cdot u^\lambda_f[\mu] - Y \quad \text{and} \quad  f(u^\lambda_f[\mu]) + f^*(\Fmap^\top \alpha^\lambda_f[\mu]) = u^\lambda_f[\mu] (\Fmap^\top \alpha^\lambda_f[\mu]).
    \end{align}
    Moreover, \cref{eq:L_dual} also holds in the case $\lambda = 0$ under~\cref{ass:teacher_student}.
\end{prop}

When $\lambda > 0$, this result yields a convenient reformulation of the functional $\Ll^\lambda_f$.
For $\mu \in \Pp(\Om)$, $\alpha^\lambda_f[\mu] \in L^2(\rho)$ being the maximizer in~\cref{eq:L_dual} and $u^\lambda_f[\mu] \in L^1(\mu)$ the minimizer in~\cref{eq:L_lambda_f}, we have:
\begin{align} \label{eq:L_alpha_u}
    \Ll^\lambda_f(\mu) = \frac{\lambda}{2} \left\| \alpha^\lambda_f[\mu] \right\|^2_{L^2(\rho)} + \int_\Om f( u^\lambda_f[\mu] ) \d \mu.
\end{align}

\begin{proof}

Consider $\mu \in \Pp(\Om)$ and $\lambda > 0$.
First, by definition of $\Fmap^\top$ we have for every $\alpha \in L^2(\rho)$ and every $u \in L^1(\mu)$ that $\int_\Om (\Fmap^\top \alpha) u \d \mu = \left< \alpha, \Fmap_\mu \cdot u \right>_{L^2(\rho)}$ i.e. $\Fmap^\top$ is the adjoint of $\Fmap_\mu : L^1(\mu) \to L^2(\mu)$.
Also, it follows from the assumption on $f$ that the map $\Ii_f : u \in L^1(\mu) \mapsto \int_\Om f(u) \d \mu$ is a convex, weakly lower semicontinuous functional whose Legendre transform is given for $h \in L^\infty(\mu)$ by:
\begin{align*}
    \Ii_f^*(h) = \sup_{u \in L^1(\mu)} \int_\Om h u \d \mu - \int_\Om f(u) \d \mu = \int_\Om f^*(h) \d \mu ,
\end{align*}
with $f^*$ the Legendre transform of $f$ and where the supremum is attained for $u \in L^1(\mu)$ satisfying the duality relation $f(u) + f^*(h) = u h$~\cite[Thm. 2]{rockafellar1968integrals}.
Similarly, for $u \in L^1(\mu)$ we have:
\begin{align*}
    \sup_{\alpha \in L^2(\rho)} - \left< \alpha, \Fmap_\mu \cdot u - Y \right>_{L^2(\rho)} - \frac{\lambda}{2}  \| \alpha \|^2_{L^2(\rho)} =
        \frac{1}{2 \lambda} \left\| \Fmap_\mu \cdot u - Y \right\|^2_{L^2(\rho)}.
\end{align*}
where the supremum is reached at $\alpha = \lambda \left( \Fmap_\mu \cdot u - Y \right)$ when $\lambda > 0$.
Moreover, the functional $\alpha \mapsto \frac{1}{2 \lambda} \| \alpha \|^2_{L^2(\rho)}$ being continuous, we can apply~\cite[Thm. 3]{rockafellar1967duality} and~\cref{eq:L_dual} holds by strong duality.
The optimums are attained in both~\cref{eq:L_lambda_f} and~\cref{eq:L_dual} and thus~\cref{eq:duality} expresses the optimality conditions.

Finally, for the case $\lambda = 0$, when~\cref{ass:teacher_student} holds we have by~\cref{eq:L_measure} that $\Ll^0_f(\mu) = \D_f(\Bar{\nu}|\mu)$. Also, the assumptions on $f$ ensures $\dom(f^*) = \RR$ and using~\cite[Thm. 4]{rockafellar1971integrals} we obtain:
\begin{align*}
    \Ll^\lambda_f(\mu) = \D_f(\Bar{\nu}|\mu) = \sup_{h \in \Cc^0(\Om)} \int_\Om h \d \Bar{\nu} - \int_\Om f^*(h) \d \mu.
\end{align*}
The result follows as the injectivity of $\Fmap \star$ ensures $\Rg(\Fmap^\top)$ is dense in $\Cc^0(\Om)$ (\cref{lem:rkhs_universality}).

\end{proof}

Observing that $\Fmap^\top$ defines a partial isometry from $L^2(\rho)$ to the RKHS $\Hh$ (\cref{eq:RKHS_characterization}), a similar dual formulation of $\Ll^\lambda_f$ also holds in duality with $\Hh$.

\begin{prop} \label{prop:duality_mmd}
    Let~\cref{ass:regularization} and~\cref{ass:teacher_student} hold and consider $\mu \in \Pp(\Om)$.
    Then we have for $\lambda \geq 0$:
    \begin{align} \label{eq:L_mmd_dual}
        \Ll^\lambda_f(\mu) = \sup_{h \in \Hh} - \int_\Om f^*(h) \d \mu + \int_\Om h \d \Bar{\nu} - \frac{\lambda}{2} \| h \|^2_\Hh ,
    \end{align}
    where $f^*$ is the Legendre transform of $f$.
    For $\lambda > 0$, the supremum in~\cref{eq:L_mmd_dual} is attained at some $h^\lambda_f[\mu] \in \Hh$ and for $\nu^\lambda_f[\mu] \in L^1(\mu)$ the optimizer in~\cref{eq:L_mmd} it holds:
    \begin{align} \label{eq:duality_mmd}
        \lambda h^\lambda_f[\mu] = \Fmap^\top \Fmap \star(\nu^\lambda_f[\mu] - \Bar{\nu})  \quad \text{and} \quad  f(\frac{\d \nu^\lambda_f[\mu]}{\d \mu}) + f^*(h^\lambda_f[\mu]) =  h^\lambda_f[\mu] \frac{\d \nu^\lambda_f[\mu]}{\d \mu}.
    \end{align}
\end{prop}

\begin{proof}
    The formula~\cref{eq:L_mmd_dual} is directly deduced from~\cref{eq:L_dual} and the characterization of the RKHS $\Hh$ in~\cref{eq:RKHS_characterization}.
    Also~\cref{eq:duality_mmd} is a rewritting of~\cref{eq:duality} since $\nu^\lambda_f[\mu] \in \Mm(\Om)$ and $h^\lambda_f[\mu] \in \Hh$ are related to $u^\lambda_f[\mu] \in L^1(\mu)$ and $\alpha^\lambda_f[\mu] \in L^2(\rho)$ by
    $\d \nu^\lambda_f[\mu] = u^\lambda_f[\mu] \d \mu$ and $h^\lambda_f[\mu] = \Fmap^\top \alpha^\lambda_f[\mu]$.
\end{proof}

\subsection{Kernel learning in the case of quadratic regularization}

The case of a quadratic regularization is of particular interest since the partial optimization problem over $u$ admits a closed-form solution which can be efficiently obtained numerically by solving a linear system.
In this case, the task of minimizing the reduced risk is equivalent to solving a \emph{Multiple Kernel Learning} problem~\cite{bach2004multiple}.

For the $L^2$-regularization $f(t) = |t|^2$, the reduced risk $\Ll^\lambda_2[\mu]$ is the value of the \emph{ridge regression problem} in~\cref{eq:L_lambda_min} and for  $\lambda > 0$, the optimizer is given by
$u^\lambda_2[\mu] = (\Fmap_\mu^\top \Fmap_\mu + 2\lambda)^{-1} \Fmap_\mu^\top Y$
where $\Fmap_\mu^\top : L^2(\rho) \to L^2(\mu)$ is the adjoint of the operator $\Fmap_\mu$ restricted to $L^2(\mu)$.
Also, the dual problem in~\cref{eq:L_dual} here reads:
\begin{align*}
    \Ll^\lambda_2(\mu) = \sup_{\alpha \in L^2(\rho)} - \frac{1}{2} \left< \alpha, (K_\mu + 2\lambda) \alpha \right>_{L^2(\rho)}
    + \left< \alpha, Y \right>_{L^2(\rho)},
\end{align*}
where $K_\mu : L^2(\rho) \to L^2(\rho)$ is the self-adjoint operator defined by $K_\mu = \Fmap_\mu \Fmap_\mu^\top$.
The supremum is attained at $\alpha^\lambda_2[\mu] = (K_\mu + 2\lambda)^{-1} Y$ and by~\cref{eq:L_alpha_u} we obtain for every $\mu \in \Pp(\Om)$:
\begin{align} \label{eq:kernel_learning}
    \Ll^\lambda_2(\mu) = \frac{1}{2} \left< Y, (K_\mu + 2\lambda)^{-1} Y \right>_{L^2(\rho)}.
\end{align}
This is the optimal value of the kernel ridge regression problem with kernel $K_\mu$, where $K_\mu$ is parameterized by the feature distribution $\mu$. 
Moreover this parameterization is linear in $\mu \in \Pp(\Om)$ since considering for $\om \in \Om$ the rank-one self-adjoint operator $k(\om) \eqdef \fmap(\om, .) \otimes \fmap(\om, .)$ we have:
\begin{align*}
    K_\mu = \int_\Om k(\om) \d \mu(\om) .
\end{align*}
Therefore, minimizing the reduced risk $\Ll^\lambda_2$ over the feature distribution $\mu$ amounts to finding the best kernel for solving the ridge regression problem in~\cref{eq:L_lambda_min} among convex combinations of ``simple'' basis kernels $(k(\om))_{\om \in \Om}$ i.e. a \emph{Multiple Kernel Learning} task.
Other convex optimization strategies for solving such task have been studied in~\cite{lanckriet2004learning,bach2004multiple}.

\section{Properties of minimizers of the reduced risk} \label{sec:minimizer}

Before turning to the analysis of gradient methods for the minimization of the reduced risk $\Ll^\lambda_f$ in~\cref{sec:training,sec:convergence}, we study here variational properties of $\Ll^\lambda_f$.

\subsection{Existence and uniqueness of minimizers}

We first investigate existence and uniqueness of minimizers of $\Ll^\lambda_f$.
Importantly, we use here that $\Ll^\lambda_f$ is obtained as the result of a partial minimization.
Namely, for $\lambda \geq 0$ and $\mu \in \Pp(\Om)$, we have from~\cref{eq:L_measure} that $\Ll^\lambda_f(\mu) = \min_{\nu \in \Mm(\Om)} \Ee^\lambda_f(\nu, \mu)$, where $\Ee^\lambda_f$ is defined for $\nu \in \Mm(\Om)$ and $\mu \in \Pp(\Om)$ by:
\begin{align} \label{eq:E}
    \Ee^\lambda_f(\nu, \mu) \eqdef
    \left\{
    \begin{array}{cc}
        \D_f(\nu|\mu) + \frac{1}{2 \lambda} \| \Fmap \star \nu - Y \|^2 & \text{if $\lambda > 0$,}  \\[10pt]
        \D_f(\nu|\mu) + \iota_{\Fmap \star \nu = Y} & \text{if $\lambda = 0$.} 
    \end{array}
    \right.
\end{align}
In particular, it follows from variational formulations of $f$-divergences that $ \D_f$ is (jointly) convex and lower semicontinuous w.r.t.\@ its arguments $(\nu, \mu) \in \Mm(\Om) \times \Pp(\Om)$~\cite[Thm. 4]{rockafellar1971integrals}.
The following~\cref{lem:L_convexity} uses this fact to establish convexity and lower semicontinuity of $\Ll^\lambda_f$, implying the existence of minimizers.
We then discuss cases in which $\Ll^\lambda_f$ has in fact a unique minimizer.

\begin{lem} \label{lem:L_convexity}
    Assume $f$ satisfies~\cref{ass:regularization}.
    Then, for $\lambda \geq 0$, $\Ll^\lambda_f :  \Pp(\Om) \to \RR$ is a convex, lower semicontinuous function (w.r.t.\@ the narrow convergence on $\Mm(\Om)$).
\end{lem}

\begin{proof}
    By the definition of the divergence $\D_f$ in~\cref{eq:divergence} and by~\cite[Thm. 4]{rockafellar1971integrals}, we have for every $(\nu, \mu) \in \Mm(\Om) \times \Pp(\Om)$:
    \begin{align*}
        \D_f(\nu | \mu) = \sup_{h \in \Cc^0(\Om)} \int_\Om h \d \nu - \int_\Om f^*(h) \d \mu .
    \end{align*}
    Thus $\D_f$ is a (jointly) convex and lower semicontinuous function as a supremum of (jointly) convex and lower semicontinuous functions.
    As a consequence, for $\lambda \geq 0$, $\Ee^\lambda_f$ is also (jointly) convex and lower semicontinuous.
    The convexity of $\Ll^\lambda_f = \min_{\nu} \Ee^\lambda_f(\nu, .)$ follows as partial minimization preserves convexity.
    Also, if $(\mu_n)_{n \geq 0}$ is a sequence in $\Pp(\Om)$ converging narrowly to some $\mu \in \Pp(\Om)$, then we have $\Ll^\lambda_f(\mu_n) = \Ee^\lambda_f(\nu_n, \mu_n)$ for some $\nu_n \in \Mm(\Om)$. 
    Without loss of generality one can assume $\Ll^\lambda_f(\mu_n)$ is bounded, thus $\D_f(\nu_n | \mu_n)$ and then $\| \nu_n \|_\TV$ are bounded as $f$ is superlinear.
    Then, up to extraction of a subsequence, $(\nu_n)$ converges narrowly to $\nu \in \Mm(\Om)$ and we get by lower semicontinuity of $\Ee^\lambda_f$:
    \begin{align*}
        \liminf_{n \to \infty} \Ll^\lambda_f(\mu_n) = \liminf_{n \to \infty} \Ee^\lambda_f(\nu_n, \mu_n) \geq \Ee^\lambda_f(\nu, \mu) \geq \Ll^\lambda_f(\mu),
    \end{align*}
    which shows that $\Ll^\lambda_f$ is lower semicontinuous.
\end{proof}

The above result implies the existence of minimizers of the reduced risk $\Ll^\lambda_f$ for every $\lambda \geq 0$ but it does not establish uniqueness and $\Ll^\lambda_f$ may, a priori, have several minimizers.
However, there are cases in which uniqueness can be ensured.
We give two examples:
\begin{itemize}
    \item In the teacher-student setup where~\cref{ass:teacher_student} holds, if $\Bar{\nu}$ is a positive measure with $\Bar{m} = \Bar{\nu}(\Om) > 0$ and if $f$ is nonnegative, strictly convex and such that $f(\Bar{m}) = 0$ then the teacher feature distribution $\Bar{\mu} = \Bar{\nu} / \Bar{m}$ is the unique minimizer of $\Ll^\lambda_f$, whatever $\lambda \geq 0$.
    Indeed, from~\cref{eq:L_mmd} we have that $\Ll^\lambda_f(\Bar{\mu}) = 0$ and $\Ll^\lambda_f(\mu) > 0$ for every $\mu \neq \Bar{\mu}$.
    With these assumptions, we prove in~\cref{thm:algebraic_convergence} that the gradient flow of $\Ll^\lambda_f$ converges towards the teacher feature distribution $\Bar{\mu}$ with an algebraic convergence rate. 
    
    \item For general data $Y$, relying on a variational characterization of the total variation, the following~\cref{lem:uniqueness_minimizer} establishes uniqueness of a minimizer to $\Ll^\lambda_r$ in the case the regularization is of the form $f(t) = |t|^r$ for some $r > 1$.
\end{itemize}

\begin{lem} \label{lem:uniqueness_minimizer}
    Let $\lambda \geq 0$ and assume $f(t) = |t|^r$ for some $r > 1$. 
    Then $\Ll^\lambda_r$ admits a unique minimizer $\Bar{\mu}^\lambda_r \in \Pp(\Om)$.
\end{lem}

\begin{proof}
    We use arguments similar to the one of \cite[Prop. 3.3]{wang2024mean}.
    By duality $\inf_{\mu \in \Pp(\Om)} \Ll^\lambda_r(\mu) = \inf_{\nu \in \Mm(\Om)} \Gg^\lambda_r(\nu)$ where $\Gg^\lambda_r$ is defined for $\nu \in \Mm(\Om)$ by:
    \begin{align*}
        \Gg^\lambda(\nu) \eqdef \left\{
        \begin{array}{cc}
            \| \nu \|^r_{\TV} + \frac{1}{2 \lambda} \| \Fmap \star \nu - Y \|^2 & \text{if $\lambda > 0$,}  \\
            \| \nu \|^r_{\TV} + \iota_{\Fmap \star \nu = Y} & \text{if $\lambda = 0$,} 
        \end{array}
        \right.
    \end{align*}
    where we used the variational representation $\| \nu \|^r_{\TV} = \inf_{\mu \in \Pp(\Om)} \int_\Om \left| \frac{\d \nu}{\d \mu} \right|^r \d \mu$ (the case $r=2$ is used in~\cite{wang2024mean,lanckriet2004learning}).
    Indeed, for measures $\nu \in \Mm(\Om)$ and $\mu \in \Pp(\Om)$ s.t.\@ $\nu \ll \mu$ we have:
    \begin{align*}
        \int_\Om \left| \frac{\d \nu}{\d \mu} \right|^r \d \mu = \int_\Om \left( \frac{\d \mu}{ \d |\nu|} \right)^{1-r} \d |\nu|.
    \end{align*}
     The Lagrangian of the convex problem $\inf_{\mu} \int_\Om \left( \frac{\d \mu}{ \d |\nu|} \right)^{1-r} \d |\nu|$ is given by:
    \begin{align*}
        \Jj(\mu, \gamma) = \int_\Om \left( \frac{\d \mu}{ \d |\nu|} \right)^{1-r} \d |\nu| + \gamma \left( \int_\Om \d \mu -1 \right).
    \end{align*}
     The optimality condition gives that $\frac{\d \mu}{\d |\nu |}$ is constant and the minimum is attained for $\mu = |\nu| / \| \nu \|_\TV$, giving $\int_\Om \left| \frac{\d \nu}{\d \mu} \right|^r \d \mu = \| \nu \|^r_\TV$.
    Finally, the map $\nu \mapsto \| \nu \|^r_{\TV}$ is strictly convex so that $\Gg^\lambda_r$ admits a unique minimizer $\Bar{\nu}^\lambda_r \in \Mm(\Om)$, thus $\Ll^\lambda_r$ has also a unique minimizer $\Bar{\mu}^\lambda_r \in \Pp(\Om)$ and we have the duality relation $\Bar{\mu}^\lambda_r = | \Bar{\nu}^\lambda_r | / \| \Bar{\nu}^\lambda_r \|_{\TV}$.
    Notably, in the case where~\cref{ass:teacher_student} holds and $\lambda = 0$,  we have $\Bar{\nu}^0_r = \Bar{\nu}$ and $\Bar{\mu}^0_r = \Bar{\mu}$ for every $r > 1$.
\end{proof}

\subsection{Convergence of minimizers}

Of particular interest to us is the case $\lambda = 0$ for which minimizers of $\Ll^0_f$ are related to the teacher measure $\Bar{\nu}$ by~\cref{eq:L0_div}.
However, in practice, minimization of $\Ll^\lambda_f$ is easier in the presence of a regularization parameter $\lambda > 0$. For this reason, we are interested in the asymptotic behavior of minimizers of $\Ll^\lambda_f$ when $\lambda \to 0^+$.

We show here, when $\lambda \to 0^+$, that any converging sequence of minimizers to $\Ll^\lambda_f$ converges to some minimizer of $\Ll^0_f$.
In particular, if $\Ll^0_f$ has a unique minimizer $\Bar{\mu}^0$ then any sequence of minimizers to $\Ll^\lambda_f$ converges to $\Bar{\mu}^0$.
This result is a consequence of the following~\cref{lem:gamma} which states the $\Gamma$-convergence of the functionals $\Ll^\lambda_f$ to $\Ll^0_f$.
We refer to~\cite[Chap. 7]{santambrogio2023course} for an introduction to $\Gamma$-convergence.
This is in particular stronger than pointwise convergence and is the appropriate notion of convergence for studying the behavior of minimizers.
In the case where~\cref{ass:teacher_student} holds and $\Ll^\lambda_f$ admits the representation~\cref{eq:L_mmd}, a similar result was established by~\textcite{neumayer2024wasserstein}, with a notable difference being here that we prove $\Gamma$-convergence w.r.t.\@ the variable $\mu$ instead of $\Bar{\nu}$.

\begin{lem}[$\Gamma$-convergence] \label{lem:gamma}
    Assume \cref{ass:regularization,ass:teacher_student} hold, $\fmap \in L^2(\rho,\Cc^{0,1})$ and $f^* \in \Cc^{0,1}_\loc(\RR)$.
	Then the family of functionals $(\Ll^\lambda_f)_{\lambda > 0}$ $\Gamma$-converges towards $\Ll^0_f$ as $\lambda \to 0^+$ in the sense that for every sequence $(\lambda_n)_{n \geq 0}$ converging to $0^+$ and every $\mu \in \Pp(\Om)$ it holds:
	\begin{enumerate}
		\item for every sequence $(\mu_n)_{n \geq 0}$ converging narrowly to $\mu$, $\liminf_{n \to \infty} \Ll^{\lambda_n}_f(\mu_n) \geq \Ll^0_f(\mu)_f$,
		\item there exists a sequence $(\mu_n)_{n \geq 0}$ converging narrowly to $\mu$ s.t.\@ $\limsup_{n \to \infty} \Ll^{\lambda_n}_f(\mu_n) \leq \Ll^0_f(\mu)$.
	\end{enumerate}
\end{lem}

\begin{proof}
	For the second part of the result it suffices to consider the constant sequence $\mu_n = \mu$ for every $n \geq 0$. Indeed, it then directly follows from the definition of $\Ll^\lambda_f$ and $\Ll^0_f$ in~\cref{eq:L_lambda_f} and~\cref{eq:L0_f} that $\Ll^\lambda_f(\mu) \leq \Ll^0_f(\mu)$ for every $\lambda >0$.
	
	To prove the first part of the definition, consider a sequence $(\mu_n)_{n \geq 0}$ converging narrowly to $\mu$ in $\Pp(\Om)$.
	By the dual formulation of $\Ll^0_f$ in~\cref{eq:L_dual}, we can also consider a sequence $(\alpha_k)_{k \geq 0}$ in $L^2(\rho)$ such that:
	\begin{align*}
		- \int_\Om f^*(\Fmap^\top \alpha_k) \d \mu + \left< \alpha_k , Y \right> \xrightarrow{k \to \infty} \Ll^0_f(\mu).
	\end{align*}
	Then, by the dual formulation of $\Ll^\lambda_f$ in~\cref{eq:L_dual}, for every $n,k \geq 0$:
	\begin{align*}
		\Ll^{\lambda_n}(\mu_n) & \geq - \int_\Om f^*(\Fmap^\top \alpha_k) \d \mu_n + \left< \alpha_k, Y \right> - \frac{\lambda_n}{2} \| \alpha_k \|^2_{L^2{\rho}} \\
		& = - \int_\Om f^*(\Fmap^\top \alpha_k) \d (\mu_n - \mu) - \frac{\lambda_n}{2} \| \alpha_k \|^2_{L^2{\rho}} - \int_\Om f^*(\Fmap^\top \alpha_k) \d \mu + \left< \alpha_k, Y \right>  \\
		& \geq - \left\| f^*(\Fmap^\top \alpha_k) \right\|_{\Cc^{0,1}} \Ww_1(\mu_n, \mu) - \frac{\lambda_n}{2} \| \alpha_k \|^2_{L^2{\rho}} - \int_\Om f^*(\Fmap^\top \alpha_k) \d \mu + \left< \alpha_k, Y \right>  .
	\end{align*}
	But then, since $\Ww_1(\mu_n, \mu) \to 0$ and $\lambda_n \to 0$, one can find an increasing sequence $(k_n)_{n \geq 0}$ s.t.\@:
	\begin{align*}
	    \lambda_n \| \alpha_{k_n} \|^2_{L^2(\rho)} \xrightarrow{n \to +\infty} 0 \quad \text{and} \quad \left\| f^*(\Fmap^\top \alpha_{k_n}) \right\|_{\Cc^{0,1}} \Ww_1(\mu_n, \mu) \xrightarrow{n \to +\infty} 0 .
	\end{align*}
	Thus $\Ll^{\lambda_n}_f(\mu_n) \geq \Ll^0_f(\mu) + o(1)$ for every $n \geq 0$ and the result follows.
\end{proof}

It is a direct consequence of the above $\Gamma$-convergence result that the limit when $\lambda \to 0^+$ of a sequence of minimizers of $\Ll^\lambda_f$ is a minimizer of $\Ll^0_f$~\cite[Prop. 7.5]{santambrogio2023course}.
In the case where $\Ll^\lambda_f$ has a unique minimizer $\Bar{\mu}^0_f$, this implies every sequence of minimizers of $\Ll^\lambda_f$ converges to $\Bar{\mu}^0_f$ when $\lambda \to 0^+$.

\begin{prop}[Convergence of minimizers] \label{prop:convergence_minimizers}
	 Assume the result of~\cref{lem:gamma} holds.
	 For every $\lambda >  0$, consider $\Bar{\mu}^\lambda_f \in \argmin \Ll^\lambda_f$ and assume $ \Bar{\mu}^\lambda_f \xrightarrow{\lambda \to 0^+} \mu$.
	 Then $\mu \in \argmin \Ll^0_f$.
	 Notably, if $\Ll^0_f$ has a unique minimizer $\Bar{\mu}^0_f$, then $\Bar{\mu}^\lambda_f \xrightarrow{\lambda \to 0^+} \Bar{\mu}^0_f$.
\end{prop}

\section{Training with gradient flow} \label{sec:training}

In the rest of this work we consider the optimization over the feature distribution $\mu \in \Pp(\Om)$ for the minimization of the reduced risk $\Ll^\lambda_f$, for $\lambda \geq 0$.
Specifically, we consider a \emph{gradient flow} algorithm.
In the case of a finite number of features $\{ \om_i \}_{1 \leq i \leq M} \in \Om^M$ such gradient flow is defined as the solution of the equation:
\begin{align} \label{eq:empirical_gradient_flow}
    \forall i \in \lbrace 1, ..., M \rbrace, \quad \frac{\d}{\d t} \om_i(t) = - M \nabla_{\om_i} \hat{\Ll}^\lambda_f ( \lbrace \om_i(t) \rbrace_{1 \leq i \leq M} )
\end{align}
where $\Hat{\Ll}^\lambda_f(\lbrace \om_i \rbrace_{1 \leq i \leq M}) \eqdef \Ll^\lambda_f(\Hat{\mu})$ and $\Hat{\mu}$ is the empirical distribution $\Hat{\mu} = \frac{1}{M} \sum_{i=1}^M \delta_{\om_i}$.
More generally, in terms of the feature distribution $\mu \in \Pp(\Om)$, the above equation corresponds to a \emph{Wasserstein gradient flow} over the functional $\Ll^\lambda_f$, namely:
\begin{align} \label{eq:general_wasserstein_flow}
    \partial_t \mu_t - \div \left( \mu_t \nabla \frac{\delta \Ll^\lambda_f}{\delta \mu}[\mu_t] \right) = 0, \quad \text{on  $(0, \infty) \times \Om$,}
\end{align}
where for $\mu \in \Pp(\Om)$, $\frac{\delta \Ll^\lambda_f}{\delta \mu}[\mu]$ is the \emph{Fréchet differential} of $\Ll^\lambda_f$ at $\mu$~\cite[Def. 7.12]{santambrogio2015optimal}.
Importantly, \textcite{jordan1998variational} have shown that Wasserstein gradient flows can be obtained as limits of proximal update schemes when the discretization step tends to $0$.
Here, the curve $(\mu_t)_{t \geq 0}$ is the limit of the piecewise-constant curve with values $(\mu^\tau_k)_{k \geq 0}$
where, given a time-step $\tau > 0$ and an initialization $\mu^\tau_0 \in \Pp(\Om)$, the sequence $(\mu^\tau_k)_{k \geq 0}$ is defined recursively by:
\begin{align} \label{eq:proximal step}
	\forall k \geq 0, \quad \mu^\tau_{k+1} \in \argmin_{\mu \in \Pp(\Om)} \Ll^\lambda_f(\mu) + \frac{1}{2 \tau} \Ww_2(\mu, \mu^\tau_k)^2.
\end{align}
We study in this section the well-posedness of the above Wasserstein gradient flow equation by distinguishing the case where $\lambda > 0$ and the case $\lambda = 0$.
In the latter case, we show the Wasserstein gradient flow corresponds to a \emph{weighted ultra-fast diffusion equation}~\cite{iacobelli2019weighted}.

\subsection{Wasserstein gradient flows in the case $\lambda > 0$}

In the case where $\lambda > 0$, the presence of the regularization induces sufficient regularity on the objective to study the training dynamic through the lens of classical results from the theory of gradient flows in the Wasserstein space~\cite{ambrosio2008gradient,santambrogio2017euclidean}.
In particular, one can derive the gradient flow equation leveraging the dual representation of $\Ll^\lambda_f$.
Indeed, \cref{eq:L_dual} expresses $\Ll^\lambda_f$ as a maximum over linear functionals, and thus by the envelope theorem one can formally differentiate $\Ll^\lambda_f$ w.r.t.\@ $\mu$ and obtain the Fréchet differential:
\begin{align*}
	 \frac{\delta \Ll^\lambda_f}{\delta \mu}[\mu](\om) = -  f^* (\Fmap^\top \alpha^\lambda_f[\mu]) (\om)
\end{align*}
with $\alpha^\lambda_f[\mu] \in L^2(\rho)$ the maximizer in~\cref{eq:L_dual}.
We show that the gradient field of this potential indeed defines a notion of ``gradient'' for the functional $\Ll^\lambda_f$ w.r.t.\@ the Wasserstein topology on $\Pp(\Om)$.

Locally absolutely continuous curves $(\mu_t)_{t \in (0, +\infty)}$ in the space $\Pp(\Om)$, equipped with the Wasserstein distance $\Ww_2$, are characterised as solutions to the continuity equation:
\begin{align} \label{eq:continuity}
    \partial_t \mu_t + \div (\mu_t v_t) = 0 \quad \text{on $(0, +\infty) \times \Om$}
\end{align}
for some velocity field $v$ such that $\| v_t \|_{L^2(\mu_t)} \in L^1_\loc((0, +\infty))$~\cite[Thm. 5.14]{santambrogio2015optimal}.
This equation has to be understood in the sense of distributions, that is in duality with the set $\Cc^\infty_c((0,+\infty) \times \Om)$ of compactly supported test functions, i.e.:
\begin{align} \label{eq:continuity_weak}
	    \int_0^1 \int_\Om \left( \partial_t \varphi + \left< \nabla \varphi, v_t \right> \right) \d \mu_t \d t = 0, \quad \forall \varphi \in \Cc^\infty_c((0, +\infty) \times \Om). 
\end{align}
The following result shows that the functional $\Ll^\lambda_f(\mu_t)$ is differentiable along those curves and expresses its derivative in terms of the gradient field $\nabla \frac{\delta \Ll^\lambda_f}{\delta \mu}$.

\begin{lem}[Wasserstein chain rule for $\Ll^\lambda_f$] \label{lem:chain_rule}
	Assume $\fmap \in L^2(\rho,\Cc^1)$, $f$ satisfies~\cref{ass:regularization} with $f^* \in \Cc^1_\loc(\RR)$ and consider $\lambda > 0$.
	Let $(\mu_t)_{t \in (0, +\infty)}$ be a locally absolutely continuous curve in $\Pp(\Om)$ solution of the continuity equation~\cref{eq:continuity_weak} for a velocity field $v$ s.t.\@ $\| v_t \|_{L^2(\mu_t)} \in L^1_\loc((0, +\infty))$.
	Then $(\Ll^\lambda_f(\mu_t))_{t \in (0,+\infty)}$ is locally absolutely continuous and for a.e.\@ $t', t \in (0, +\infty)$:
	\begin{align*}
		\Ll^\lambda_f(\mu_{t'}) - \Ll^\lambda_f(\mu_t) = \int_t^{t'} \left< \nabla \Ll^\lambda_f [\mu_s], v_s \right>_{L^2(\mu_s)} \d s \, ,
	\end{align*}
	where for $\mu \in \Pp_2(\Om)$ the velocity field $\nabla \Ll^\lambda_f[\mu] \in L^2(\mu)$ is defined by:
	\begin{align*}
	    \forall \om \in \Om, \quad \nabla \Ll^\lambda_f[\mu](\om) \eqdef - \nabla \left( f^* (\Fmap^\top \alpha^\lambda_f[\mu]) \, , \right)(\om)
	\end{align*}
	with $\alpha^\lambda_f[\mu]$ the maximizer in~\cref{eq:L_dual}.
\end{lem}

\begin{proof}
    Consider the dual formulation of $\Ll^\lambda_f$ in~\cref{eq:L_dual}.
    For every $\mu \in \Pp(\Om)$ we have:
    \begin{align*}
        \Ll^\lambda_f(\mu) = \sup_{\alpha \in L^2(\rho)} \Vv_\alpha(\mu) - \frac{\lambda}{2} \| \alpha \|^2 + \left< \alpha , Y \right>,
    \end{align*}
    where for $\alpha \in L^2(\rho)$ we defined:
    \begin{align*}
        \Vv_\alpha(\mu) \eqdef - \int_\Om f^* (\Fmap^\top \alpha)(\om) \d \mu(\om).
    \end{align*}
    In particular, at fixed $\alpha \in L^2(\rho)$, it follows from the assumptions on $\fmap$ and $f^*$ that the potential $f^*(\Fmap^\top \alpha)$ is in $\Cc^1(\Om)$ with $\|  f^*(\Fmap^\top \alpha) \|_{\Cc^1} \leq C(\| \alpha \|_{L^2(\rho)})$ for some continuous function $C$.
    Thus, by properties of the continuity equation, $\Vv_\alpha(\mu_t)$ is locally absolutely continuous and its derivative is given for a.e.\@ $t \in (0, +\infty)$ by:
    \begin{align*}
        \frac{\d}{\d t} \Vv_\alpha(\mu_t) = - \int_\Om \left< \nabla f^* ( \Fmap^\top \alpha) , v_t \right> \d \mu_t.
    \end{align*}
    Moreover, for $\mu \in \Pp(\Om)$, using~\cref{eq:L_alpha_u} and the fact that $\Ll^\lambda_f \leq \frac{1}{2 \lambda} \| Y \|^2_{L^2(\rho)} + f(0)$ (by taking $u = 0$ in~\cref{eq:L_lambda_f}) we have at the optimum in~\cref{eq:L_dual} that
    $$\| \alpha^\lambda_f[\mu] \|_{L^2(\rho)} \leq \lambda^{-1} \left( \| Y \|^2_{L^2(\rho)} + \lambda f(0) \right)^{1/2} \defeq R_\lambda$$
    Thus, $\Ll^\lambda_f$ is equivalently defined by restricting the supremum to $\alpha \in L^2(\rho)$ such that $\| \alpha \|_{L^2(\rho)} \leq R_\lambda$.
    For such $\alpha$ we have $\left| \frac{\d}{\d t} \Vv_\alpha(\mu_t)  \right| \leq C'$ for some constant $C' = C'(f^*, \lambda)$ independent of $\alpha$.
    Thus we can apply the envelope theorem in~\cite[Thm. 2]{milgrom2002envelope}, which shows the desired result.
\end{proof}

The preceding result has defined a notion of gradient field for the functional $\Ll^\lambda_f$.
One can thus define gradient flows of $\Ll^\lambda_f$ for the $\Ww_2$ metric as the curves solution to the continuity equation:
\begin{align} \label{eq:grad_flow_lambda}
		\partial_t \mu_t - \div (\mu_t \nabla \Ll^\lambda_f[\mu_t]) = 0  \quad \text{on $(0, +\infty) \times \Om$}.
\end{align}
We make the following definition:
\begin{defn}[Gradient flow of $\Ll^\lambda_f$] \label{def:grad_flow_lambda}
	Let $\mu_0 \in \Pp_2(\Om)$. We say $(\mu_t)_{t \geq 0}$ is a gradient flow for $\Ll^\lambda_f$ starting from $\mu_0$ if it is a locally absolutely continuous curve on $(0, +\infty)$ s.t.\@ $\lim_{t \to 0^+} \mu_t = \mu_0$ and if it satisfies the continuity equation~\cref{eq:grad_flow_lambda} in the sense of distribution, i.e.:
	\begin{align} \label{eq:grad_flow_lambda_weak}
	    \int_0^\infty \int_\Om \left( \partial_t \varphi - \nabla \varphi \cdot \nabla \Ll^\lambda_f[\mu_t] \right) \d \mu_t \d t = 0, \quad \forall \varphi \in \Cc^\infty_c((0, +\infty) \times \Om). 
	\end{align}
\end{defn}

\begin{rem}[Boundary conditions]
    Note that, in the case where $\Om$ is a closed, bounded, smooth and convex domain, our definition~\cref{eq:continuity_weak} of solutions to the continuity equation enforces no-flux conditions on the boundary $\partial \Om$.
    Indeed we consider test function $\varphi \in \Cc^\infty_c((0, +\infty) \times \Om)$ that can be supported on the whole domain $\Om$ (which is always assumed closed).
    Thus, \cref{eq:continuity_weak} enforces $\left< \mu_t v_t, \Vec{n} \right> = 0$ in the sense of distribution, where $\Vec{n}$ is the outer normal vector to the boundary $\partial \Om$.
    
    In case of the gradient flow equation~\cref{eq:grad_flow_lambda_weak}, this boundary condition is for example satisfied if one assumes $ \left< \nabla_\om \fmap(\om, x), \Vec{n} \right> = 0$ for every $x \in \RR^d$ and every $\om \in \partial \Om$.
    Another way of ensuring the no-flux condition is to remove the outer part of the gradient field $\nabla \Ll^\lambda_f$ on the boundary $\partial \Om$, which can be performed by clipping the features.
\end{rem}

\paragraph{Well-posedness of the gradient flow equation}

To show the well-posedness of gradient flows, we rely on convexity properties of the functional $\Ll^\lambda_f$.
Indeed, by the dual formulation in~\cref{eq:L_dual}, we can express $\Ll^\lambda_f$ as a supremum over semiconvex functionals.
As a consequence, the~\cref{lem:semi_convexity} below shows that, for $\lambda > 0$, $\Ll^\lambda_f$ is semiconvex along (generalized) geodesics of the Wasserstein space (see \cite[Def. 9.2.4]{ambrosio2008gradient} for the definition of \emph{generalized geodesics}).
However, note that such an argument can not be extended to the case $\lambda = 0$ since the semiconvexity constant blows-up when $\lambda \to 0^+$. For example, in the case $f(t) = |t|^2$ this constant scales as $\lambda^{-2}$.

\begin{lem}[Geodesic semiconvexity] \label{lem:semi_convexity}
Assume $\fmap \in L^2(\rho,\Cc^{1, 1})$, $f$ satisfies~\cref{ass:regularization} with $f^* \in \Cc^{1,1}_\loc(\RR)$ and let $\lambda > 0$. Then $\Ll^\lambda_f$ is $C$-semiconvex along (generalized) geodesics for some constant $C = C(f^*, \lambda)$.
\end{lem}

\begin{proof}
	Consider the dual formulation of $\Ll^\lambda_f$ in~\cref{eq:L_dual}.
	For every $\mu \in \Pp(\Om)$ we have:
	\begin{align*}
		\Ll^\lambda_f(\mu) = \sup_{\alpha \in L^2(\rho)} - \int_\Om f^* (\Fmap^\top \alpha)(\om) \d \mu(\om) - \frac{\lambda}{2} \| \alpha \|^2 + \left< \alpha , Y \right>,
	\end{align*}
	Then, at fixed $\alpha \in L^2(\rho)$, it follows from the assumptions on $\fmap$ that $\Fmap^\top \alpha \in \Cc^{1,1}(\Om)$ with $$\| \Fmap^\top \alpha \|_{\Cc^{1,1}} \leq \| \alpha \|_{L^2(\rho)} \| \fmap \|_{L^2(\rho, \Cc^{1,1})}.$$
	Then, from the assumptions on $f^*$, the composition $f^*(\Fmap^\top \alpha)$ is also in $\Cc^{1, 1}(\Om)$ and by \cite[Prop.9.3.2]{ambrosio2008gradient} the functional $\mu \mapsto \int_\Om f^* (\Fmap^\top \alpha) \d \mu$ is $C$-semiconvex along generalized geodesics for some constant $C = C(f^*, \| \alpha \|_{L^2(\rho)} \| \fmap \|_{L^2(\rho, \Cc^{1,1})})$.
	Moreover, similarly as in the proof of~\cref{lem:chain_rule}, one can restrict the definition of $\Ll^\lambda_f$ to the supremum over $\alpha \in L^2(\rho)$ with $\| \alpha \|_{L^2(\rho)} \leq R_\lambda$.
    The result then follows by taking a supremum over (uniformly) semiconvex functionals.
\end{proof}

The semiconvexity of $\Ll^\lambda_f$ along generalized geodesics ensures the existence and uniqueness of gradient flows in the sense of~\cref{def:grad_flow_lambda}.
	
\begin{thm}[Well-posedness of the gradient flow equation for $\lambda > 0$] \label{thm:wellposed_lambda}
	Assume the assumptions of~\cref{lem:semi_convexity} hold.
	Then for any $\lambda > 0$ and any initialization $\mu_0 \in \Pp_2(\Om)$ there exists a unique gradient flow for $\Ll^\lambda_f$ starting from $\mu_0$ in the sense of~\cref{def:grad_flow_lambda}.
	Moreover, if $(\mu_t)_{t \geq 0}, (\mu_t')_{t \geq 0}$ are gradient flows for $\Ll^\lambda_f$ with respective initializations $\mu_0, \mu_0' \in \Pp(\Om)$ then for every $t \geq 0$:
	\begin{align*}
		\Ww_2(\mu_t, \mu_t') \leq e^{C t} \Ww_2(\mu_0, \mu_0'),
	\end{align*}
	for some constant $C = C(f^*, \lambda)$.
\end{thm}

\begin{proof}
    The chain rule formula established in~\cref{lem:chain_rule} shows that for every $\mu \in \Pp(\Om)$ the vector field $\Ll^\lambda_f[\mu]$ is a strong subdifferential of $\Ll^\lambda_f$ in the sense of~\cite[Def. 10.3.1 and eq. (10.3.12)]{ambrosio2008gradient}. Existence, uniqueness and contractivity properties of the gradient flow then follow from the geodesic semiconvexity of $\Ll^\lambda_f$ established in~\cref{lem:semi_convexity} and the application of~\cite[Def. 11.2.1]{ambrosio2008gradient}
\end{proof}

Finally, it is a classical property of weak solutions to continuity equations that gradient flows of $\Ll^\lambda_f$ can be represented in terms of push-forward by a flow map.

\begin{prop} \label{prop:flow_representation}
	Assume the assumptions of~\cref{lem:semi_convexity} hold.
	Then, for any $\lambda > 0$ and any initialization $\mu_0 \in \Pp_2(\Om)$, the gradient flow $(\mu_t)_{t \geq 0}$ of $\Ll^\lambda_f$ starting from $\mu_0$ satisfies $\mu_t = (X_t)_\# \mu_0$ for every $t \geq 0$, where $(X_t)_{t \geq 0}$ is the flow-map solution of the ODE:
	\begin{align*}
		\forall t \geq 0, \quad \frac{\d}{\d t} X_t =  - \nabla \Ll^\lambda_f[\mu_t] \circ X_t, \quad \text{with} \quad X_0 = \Id_\Om.
	\end{align*}
	In particular, if $(\om_i(t))_{t \geq 0}$ for $i \in \{1, ..., M \}$ are solutions to~\cref{eq:empirical_gradient_flow} then the empirical distribution $\Hat{\mu}_t \eqdef \frac{1}{M} \sum_{i=1}^M \delta_{\om_i(t)}$ is a gradient flow for $\Ll^\lambda_f$ in the sense of~\cref{def:grad_flow_lambda} and thus $\om_i(t) = X_t(\om_i(0))$ for $i \in \{ 1, ..., M \}$ and $t \geq 0$.
\end{prop}

\begin{proof}
	 For every $t \geq 0$, similarly as in the proof of~\cref{lem:chain_rule}, we have that the dual variable is bounded by $\| \alpha^\lambda_f[\mu_t] \|_{L^2(\rho)} \leq R_\lambda$ and from the assumption on the regularity of $\fmap$ it follows that:
	\begin{align*}
		\| f^*(\Fmap^\top \alpha^\lambda_f[\mu_t] ) \|_{\Cc^{1,1}} \leq C.
	\end{align*}
	for some constant $C = C(f^*,\lambda)$.
	Then by definition $\nabla \Ll^{\lambda}[\mu_t] = - \nabla f^*(\Fmap^\top \alpha^\lambda_f[\mu_t]) \in \Cc^{0, 1}$ and the first part of the result follows from classical results of ODE theory~\cite{hale2009ordinary} and on representation of solutions to continuity equations~\cite[Thm. 8.1.8]{ambrosio2008gradient}.
	For the second part of the result it suffices to remark that, by the definition of $\Hat{\Ll}^\lambda_f$ and $\nabla \Ll^\lambda_f$, for $\{ \om_i \}_{1 \leq i \leq M} \in \Om^M$ and $j \in \{1, ..., M \}$:
	\begin{align}
	    M \nabla_{\om_j} \Hat{\Ll}^\lambda (\{ \om_i \}_{1 \leq i \leq M}) = \nabla \Ll^\lambda_f[\Hat{\mu}](\om_j)
	\end{align}
	where $\Hat{\mu} = \frac{1}{M} \sum_{i=1}^M \delta_{\om_i}$.
	Therefore, by~\cref{eq:empirical_gradient_flow} we have for any test function \mbox{$\varphi \in \Cc^\infty_c((0, +\infty) \times \Om)$}:
	\begin{align*}
	    0 = \frac{1}{M} \sum_{i=1}^M \int_0^\infty \frac{\d}{\d t} \varphi(t, \om_i(t)) \d t = \int_0^\infty \int_\Om \left( \partial_t \varphi - \nabla \varphi \cdot \nabla \Ll^\lambda_f[\Hat{\mu}_t] \right) \d \Hat{\mu}_t \d t ,
	\end{align*}
	meaning $(\Hat{\mu}_t)_{t \geq 0}$ is a gradient flow for $\Ll^\lambda_f$ according to~\cref{def:grad_flow_lambda}.
\end{proof}

\paragraph{Particle approximation}

In the case where $\lambda > 0$, associating the contraction rate of the gradient flow obtained in~\cref{thm:wellposed_lambda} with classical results on the approximation of measure by empirical distributions we  obtain an approximation result for the minimization of $\Ll^\lambda_f$ with a finite number of features.
For conciseness, we only state the result in the case $d \geq 3$, but similar results hold for $d \in \{1, 2\}$.
\begin{cor}[Particle approximation]
    Let the assumptions of~\cref{lem:L_convexity} hold and let $d \geq 3$.
    Consider some initialization $\mu_0 \in \Pp(\Om)$ and, for some $N \geq 0$, denote by $\Hat{\mu}_0 \eqdef N^{-1} \sum_{i=1}^N \delta_{\om_i}$ the empirical measure where $\lbrace \om_i \rbrace_{1 \leq i \leq N}$ are i.i.d. samples of $\mu_0$. For $\lambda > 0$, let $(\mu^\lambda_t)_{t \geq 0}$ and $(\Hat{\mu}_t^\lambda)_{t \geq 0}$ be the gradient flow of $\Ll^\lambda_f$ starting from $\mu_0$ and $\Hat{\mu}_0$ respectively. Then there exists a constant
    $A = A(d, \Om)$  s.t.\@ for every $t \geq 0$ and every $\epsilon > 0$:
    \begin{align*}
        \PP \left( \Ww_1(\Hat{\mu}^\lambda_t, \mu^\lambda_t) \geq \epsilon \right) \leq \frac{A}{ \epsilon} N^{-1/d} e^{C t}\,,
    \end{align*}
    where $C = C(f^*, \lambda)$ is the constant in~\cref{thm:wellposed_lambda}.
\end{cor}
\begin{proof}
    Using~\cite[Thm. 1]{fournier2015rate} we obtain at initialization $t=0$:
    \begin{align*}
        \EE \left[ \Ww_1(\Hat{\mu}^\lambda_0, \mu^\lambda_0) \right] \leq A N^{-1/d}
    \end{align*}
    for some constant $A = A(d, \Om) >0$ depending on the dimension and on the domain $\Om$.
    Then using the contraction rate in~\cref{thm:wellposed_lambda} we have a constant $ C = C(f^*, \lambda) >0$ such that for every $t \geq 0$:
    \begin{align*}
        \EE \left[ \Ww_1(\Hat{\mu}^\lambda_t, \mu^\lambda_t) \right] \leq A N^{-1/d} e^{C t}.
    \end{align*}
    The result then follows by applying Markov's inequality.
\end{proof}

\subsection{Wasserstein gradient flows in the case $\lambda = 0$ and  ultra-fast diffusions}

We now consider the limit of the proximal scheme~\cref{eq:proximal step} when the step size $\tau$ tends to $0$ and $\lambda$ is set to $0$.
We focus on the case where~\cref{ass:teacher_student} holds and the regularization is of the form $f(t) = |t|^r / (r-1)$ for some $r > 1$ and recall that we use the shortcut $\Ll^0_r \eqdef \Ll^0_f$.
Then, following~\cref{eq:L0_div}, we have for $\mu \in \Pp(\Om)$:
\begin{align} \label{eq:L0_diffusion}
    \Ll^0_r(\mu) = \frac{1}{r-1} \D_r(\Bar{\nu}|\mu) = \frac{1}{r-1} \int_\Om \left| \frac{\d \Bar{\nu}}{\d \mu} \right|^r \d \mu = \frac{\| \Bar{\nu} \|^r_\TV}{r-1} \int_\Om \left| \frac{\d \Bar{\mu}}{\d \mu} \right|^r \d \mu .
\end{align}
The first variation of $\Ll^0_r$ w.r.t.\@ $\mu$ is formally given by 
$\frac{\delta \Ll^0_r}{\delta \mu}[\mu](\om) = - \| \Bar{\nu} \|^r_\TV \left( \frac{\Bar{\mu}}{\mu} \right)^r$
and thus, following~\cref{eq:general_wasserstein_flow}, the Wasserstein gradient flow of $\Ll^0_r$ is formally defined as the solution to the continuity equation:
\begin{align} \label{eq:grad_flow_L0}
    \partial_t \mu_t = - \| \Bar{\nu} \|^r_\TV \div \left( \mu_t \nabla \left( \frac{\Bar{\mu}}{\mu_t} \right)^r \right).
\end{align}
Moreover, calculating formally, $\nabla \left( \frac{\Bar{\mu}}{\mu} \right)^r = r \frac{\Bar{\mu}}{\mu} \nabla \left( \frac{\Bar{\mu}}{\mu} \right)^{r-1}$ and~\cref{eq:grad_flow_L0} can be written equivalently:
\begin{align*}
    \partial_t \mu_t = - r \| \Bar{\nu} \|^r_\TV \div \left( \Bar{\mu} \nabla \left( \frac{\Bar{\mu}}{\mu_t} \right)^{r-1} \right) \, .
\end{align*}
When the target distribution is uniform, i.e. with density $\Bar{\mu} = 1$, this corresponds to a nonlinear diffusion equation of the form~\cref{eq:nonlinear_diffusion} with the coefficient $m = 1-r < 0$, that is an \emph{ultra-fast diffusion}. Such an equation, with general inhomogeneous weights $\Bar{\mu}$ was studied in~\cite{iacobelli2016asymptotic,caglioti2016quantization,iacobelli2019weighted} in the context of particle algorithms for finding an optimal quantization of the measure $\Bar{\mu}$. We rely particularly here on the work of~\textcite{iacobelli2019weighted} which establishes the well-posedness of~\cref{eq:grad_flow_L0} as well as the convergence of the solution $\mu_t$ towards the target measure $\Bar{\mu}$.
We consider the following definition of solutions for~\cref{eq:grad_flow_L0}:
\begin{defn}[Gradient flow of $\Ll^0_r$ (Def. 1.1 in \cite{iacobelli2019weighted})] \label{def:grad_flow_L0}
    Let $\mu_0 \in \Pp(\Om)$ admit a density $\mu_0 \in L^{r+2}(\Om)$. We say $(\mu_t)_{t \geq 0}$ is a weak solution of~\cref{eq:grad_flow_L0} or a gradient flow for $\Ll^0_r$ starting from $\mu_0$ if it is a narrowly continuous curve in $\Pp(\Om)$ with $\lim_{t \to 0^+} \mu_t = \mu_0$, s.t.
    \begin{align} \label{eq:grad_flow_L0_weak}
        \int_0^\infty \int_\Om \left( \partial_t \varphi - \| \Bar{\nu} \|^r_\TV \nabla \varphi \cdot \nabla \left(\frac{\Bar{\mu}}{\mu_t} \right)^r \right) \d \mu_t \d t = 0, \quad \forall \varphi \in \Cc^\infty_c((0,\infty) \times \Om).
    \end{align}
    and satisfying:
    \begin{align*}
        \left( \frac{\mu_t}{ \Bar{\mu} } \right)^{r-1} \in L^2_{\loc}((0,\infty), H^1(\Om)), \quad \frac{ \Bar{\mu} }{\mu_t} \in L^2_{\loc}((0,\infty), H^1(\Om))\,.
    \end{align*}
\end{defn}

\paragraph{Existence and uniqueness of solutions}

In~\cite{iacobelli2019weighted}, the authors establish the existence and uniqueness of gradient flows for the functional $\Ll^0_r$. More precisely, they show that, under appropriate assumptions on the initialization $\mu_0$ and on the target $\Bar{\mu}$, the iterates of the proximal scheme in~\cref{eq:proximal step} converge towards a curve $(\mu_t)_{t \geq 0}$ that is a gradient flow of the functional $\Ll^0_r$ in the sense of~\cref{def:grad_flow_L0}.

\begin{thm}[{\cite[Thm.\@ 1.2]{iacobelli2019weighted}}] \label{thm:wellposed_diffusion}
    Assume $\mu_0$ and $\Bar{\mu}$ are absolutely continuous and have bounded log-densities.
    Then there exists a unique weak solution of~\cref{eq:grad_flow_L0} starting from $\mu_0$ in the sense of~\cref{def:grad_flow_L0}.
\end{thm}

\paragraph{Convergence towards the target distribution}

In the case $\lambda = 0$, \textcite{iacobelli2019weighted} establish a linear convergence rate of the \emph{weighted ultra-fast diffusion}~\cref{eq:grad_flow_L0} towards the target distribution $\Bar{\mu}$.
Precisely, they show convergence in the $L^2$-sense of the density $\mu_t$ towards the target density $\Bar{\mu}$.
We state their result in the following theorem.

\begin{thm}[{\cite[Thm.\@ 1.4]{iacobelli2019weighted}}] \label{thm:diffusion_convergence}
    Assume $\mu_0$ and $\Bar{\mu}$ are absolutely continuous and have bounded log-densities.
    For $\mu_0 \in \Pp(\Om)$, let $(\mu_t)_{t \geq 0}$ be a weak solution of~\cref{eq:grad_flow_L0} starting from $\mu_0$ in the sense of~\cref{def:grad_flow_L0}.
    Then in fact the log-density of $\mu_t$ is bounded uniformly over $t \geq 0$ and there exists a constant $C = C(\Om, \Bar{\mu}, \mu_0) > 0$ s.t.\@ for every $t \geq 0$:
    \begin{align*}
        \| \Bar{\mu} - \mu_t \|_{L^2(\Om)} \leq C e^{-Ct}.
    \end{align*}
\end{thm}

For completeness, we give here some of the  key arguments of the proof of the above~\cref{thm:diffusion_convergence} in the case where $r = 2$.
In this case, we have for every $\mu \in \Pp(\Om)$:
\begin{align} \label{eq:L0_chi2}
    \Ll^0_2(\mu) = \| \Bar{\nu} \|^2_\TV \int_\Om \left| \frac{\d \Bar{\mu}}{\d \mu} \right|^2 \d \mu = \| \Bar{\nu} \|^2_\TV \left( \chi^2( \Bar{\mu} | \mu) + 1 \right) ,
\end{align}
where $\chi^2$ is the chi-square divergence.
The following~\cref{lem:prox_chi2_convergence} establishes the desired linear convergence rate for the proximal scheme defined in~\cref{eq:proximal step} with the loss $\Ll^0_2$.
The result in continuous time then follows from the lower semicontinuity of the $\chi^2$-divergence as the curve $(\mu_t)_{t \geq 0}$ is obtained by taking the limit of the discrete process $(\mu^\tau_k)_{k \geq 0}$ when the discretization time $\tau$ tends to zero.

From a technical perspective, the proof of \cref{lem:prox_chi2_convergence} relies on a Poincaré inequality satisfied by $\mu_t$.
It is indeed well-known that such inequality controls the convergence rate of Fokker-Planck equations towards their stationary distribution in $\chi^2$-distance~\cite[Thm. 4.4]{pavliotis2014stochastic}.
This can for example be used to prove the convergence of sampling algorithms such as \emph{Langevin Monte Carlo}~\cite{chewi2024analysis,chewi2020svgd}.
In our case, the ultra-fast diffusion~\cref{eq:grad_flow_L0} is to be interpreted as a Wasserstein gradient flow for $\Ll^0_2$, which, by the above~\cref{eq:L0_chi2}, is the \emph{reverse} $\chi^2$-divergence between $\mu_t$ and $\Bar{\mu}$ and the convergence rate is controlled by the Poincaré constant of $\mu_t$.
This rate may thus a priori evolve and vanish during training but, crucially, \cite[Lem. 2.4]{iacobelli2019weighted} shows that it is here a property of solutions to the ultra-fast diffusion equation that the log-density ratio $\| \log \left( \frac{\Bar{\mu}}{\mu_t} \right) \|_\infty$ decreases with time.
As a consequence, it is sufficient to assume that the log-density is bounded at initialization to obtain a control over the Poincaré constant of $\mu_t$, for $t \geq 0$, by a classical perturbation argument~\cite[Thm. 3.4.1]{ane2000inegalites}.

\begin{lem} \label{lem:prox_chi2_convergence}
    Assume $\mu_0$ and $\Bar{\mu}$ are absolutely continuous with bounded log-densities.
    Let $\tau > 0$ and let $(\mu_k^\tau)_{k \geq 0}$ be the sequence defined by~\cref{eq:proximal step} with $\lambda = 0$, $f(t) = |t|^2$ ($r=2$) and initialization $\mu_0^\tau = \mu_0 \in \Pp(\Om)$.
    Then there exists a constant $C > 0$ s.t.:
    \begin{align*}
        \forall k \geq 0, \quad \chi^2(\Bar{\mu} | \mu^\tau_k) \leq (1+C\tau)^{-k} \chi^2(\Bar{\mu} | \mu_0).
    \end{align*}
\end{lem}

\begin{proof}
    From~\cite[Thm.2.1 and Lem.2.4]{iacobelli2019weighted} we know the sequence $(\mu^\tau_k)_{k \geq 0}$ is uniquely defined. Moreover $\mu^\tau_k$ is absolutely continuous w.r.t.\@ Lebesgue measure and their exists a constant $C = C(\Bar{\mu}, \mu_0) > 0$ s.t.\@ the $\log$-densities $\log(\mu^\tau_k)$ satisfy:
    \begin{align*}
        \forall k \geq 0, \quad \left\| \log(\mu^\tau_k ) \right\|_\infty \leq C.
    \end{align*}
    Then, at step $k \geq 0$, we get from  the expression of $\Ll^0_2$ in~\cref{eq:L0_chi2} and from the optimality condition in~\cref{eq:proximal step} (see e.g.~\cite[Prop.7.20]{santambrogio2015optimal}) that:
    \begin{align*}
       - \| \Bar{\nu} \|^2_\TV \left( \frac{ \Bar{\mu} }{\mu^\tau_{k+1} } \right)^2 + \frac{\varphi}{\tau} = cte, \quad \text{almost everywhere on $\Om$,}
    \end{align*}
    where $\varphi$ is the Kantorovitch potential from $\mu^\tau_{k+1}$ to $\mu^\tau_k$. Also this potential is necessarily Lipschitz, hence a.e.\@ differentiable and so is $\Bar{\nu} / \mu^\tau_{k+1}$. Then from the definition of $\mu_{k+1}^{0, \tau}$ we have:
    \begin{align*}
        \Ll^0_2(\mu^\tau_k) - \Ll^0_2(\mu^\tau_{k+1}) & \geq \frac{1}{2 \tau} \Ww_2( \mu^\tau_{k+1}, \mu^\tau_k )^2 \\
        & = \frac{1}{2 \tau} \int_\Om \left| \nabla \varphi \right|^2 \d \mu^\tau_{k+1} \\
        & = \frac{\tau \| \Bar{\nu} \|^2_\TV}{2} \int_\Om \left\| \nabla \left( \frac{ \Bar{\mu} }{\mu^\tau_{k+1}} \right)^2 \right\|^2 \d \mu^\tau_{k+1}
    \end{align*}
    where we used the definition of the potential $\varphi$ and the optimality condition. Using that $\mu^\tau_{k+1}$ has log-density bounded by $C = C(\Bar{\mu}, \mu_0)$ and that the domain $\Om$ satisfies a Poincaré inequality with constant $C_P = C_P(\Om)$, it follows from a classical perturbation argument that $\mu^\tau_{k+1}$ satisfies a Poincaré inequality with constant $e^{2C}C_P(\Om)$~\cite[Thm. 3.4.1]{ane2000inegalites}. As a consequence:
    \begin{align*}
        \int_\Om \left\| \nabla \left( \frac{\Bar{\mu}}{\mu^\tau_{k+1}} \right)^2 \right\|^2 \d \mu^\tau_{k+1}
        & \geq 4 e^{-4C} \int_\Om \left\| \nabla \left( \frac{\Bar{\mu}}{\mu^\tau_{k+1}} \right) \right\|^2 \d \mu^\tau_{k+1} \\
        & \geq 4 C_P^{-1} e^{-6C} \left( \int_\Om  \left( \frac{\Bar{\mu}}{\mu^\tau_{k+1}}  \right)^2 \d \mu^\tau_{k+1} - 1 \right) \\
        & = 4 C_P^{-1} e^{-6C} \| \Bar{\nu} \|^{-2}_\TV \left( \Ll^0_2(\mu^\tau_{k+1}) - \| \Bar{\nu} \|^2_\TV \right),
    \end{align*}
    where $\| \Bar{\nu} \|^2_\TV = \inf \Ll^0_2$. Combining this with the previous inequality finally gives:
    \begin{align*}
        \left( 1 + 2 \tau C_P^{-1} e^{-6C} \right) \left( \Ll^0_2(\mu^\tau_{k+1}) - \inf \Ll^0_2 \right) \leq \Ll^0_2(\mu^\tau_k) - \inf \Ll^0_2,
    \end{align*}
    and inductively:
    \begin{align*}
        \forall k \geq 0, \quad \Ll^0_2(\mu^\tau_k) - \inf \Ll^0_2 \leq \left( 1 +  2 \tau C_P^{-1} e^{-6C} \right)^{-k} \left( \Ll^0_2(\mu_0) - \inf \Ll^0_2 \right).
    \end{align*}
    By the definition of $\Ll^0_2$ in~\cref{eq:L0_chi2}, this is the desired result.
\end{proof}

\begin{rem}[Dependence of the convergence rate w.r.t.\@ the dimension] \label{rmk:convergence_rate}
    It follows from the proof that the convergence rate $C$ in~\cref{lem:prox_chi2_convergence} scales linearly with $C_P(\Om)^{-1}$ where $C_P(\Om)$ is the Poincaré constant of the domain $\Om$.
    For bounded, Lipschitz and convex domains of $\RR^n$ or for the flat torus $\TT^n$, this constant is in particular independent of the dimension $n$~\cite{payne1960optimal}.
    Therefore, the correspondence between the training of neural networks in the two-timescale regime and solutions to ultra-fast diffusions points towards the fact that gradient methods, with suitable hyperparameter scaling, are amenable to efficient feature learning in the training of neural networks, without suffering from the curse of dimensionality~\cite{donoho2000high}.
    Note however that the convergence rate $C$ in~\cref{lem:prox_chi2_convergence} is exponentially bad in the log-density ratio $\| \log(\Bar{\mu} / \mu_0) \|_\infty$.
    In particular the convergence rate does not hold in case the teacher feature distribution is supported on a finite number of atoms.
\end{rem}

\section{Convergence of gradient flow} \label{sec:convergence}

The main purpose of this work is to study in what extent the gradient flow dynamics defined~\cref{def:grad_flow_lambda,def:grad_flow_L0} allow recovering the teacher feature distribution $\Bar{\mu}$ associated to the observed signal $Y$ in~\cref{ass:teacher_student}.
Whereas~\cref{thm:diffusion_convergence} shows convergence of the gradient flow of $\Ll^0_f$, that is solutions to the \emph{ultra-fast diffusion}~\cref{eq:grad_flow_L0}, such dynamics are hardly numerically tractable in practice due to the absence of the regularization parameter $\lambda$.
For this reason we are interested here in the asymptotic behavior of the gradient flow of $\Ll^\lambda_f$ in the case where $\lambda > 0$.
A difficulty is that, in the  case $\lambda = 0$, the proof of~\cref{thm:diffusion_convergence} relies on the implicit behavior of~\cref{eq:grad_flow_L0} which preserves the density of solutions.
Such a behavior is a priori not expected to hold when $\lambda > 0$.
As a consequence, the results in this section hold under supplementary regularity assumptions on the solutions to~\cref{eq:grad_flow_lambda}.

\subsection{Algebraic convergence rate}

At fixed $\lambda > 0$, we are able to obtain convergence towards the minimizer $\Bar{\mu}^\lambda_f$ of $\Ll^\lambda_f$ under mild regularity assumptions on solutions to the gradient flow~\cref{eq:grad_flow_lambda}.
Specifically, for a probability measure $\mu \in \Pp(\Om)$ and a function $h \in \Cc^1$ we define the weighted Sobolev seminorm of $h$ as:
\begin{align*}
    \| h \|_{\Dot{H}^1(\mu)} \eqdef \left( \int_\Om \| \nabla h \|^2 \d \mu \right)^{1/2} .
\end{align*}
Then for a measure $\nu \in \Mm(\Om)$ s.t.\@ $\int_\Om \d \nu = 0$ the negative weighted Sobolev seminorm $\| \nu \|_{\Dot{H}^{-1}(\mu)}$ is defined by duality with $\Dot{H}^1(\mu)$:
\begin{align*}
    \|  \nu \|_{\Dot{H}^{-1}(\mu)} \eqdef \sup_{\| h \|_{\Dot{H}^1(\mu)} \leq 1} \int_\Om h \d \nu.
\end{align*}
The following~\cref{thm:algebraic_convergence} states convergence of $\Ll^\lambda_f$ towards $0$ with an algebraic convergence rate provided  $ \| \mu_t - \frac{\Bar{\nu}}{\Bar{m}} \|_{\Dot{H}^{-1}(\mu_t)}$ stays bounded along the gradient flow.
As discussed below, since the domain $\Om$ is compact, this assumption is satisfied for example when both distribution have bounded log-densities.
The arguments are similar to the one presented in~\cite{glaser2021kale}, where the authors consider an infimal convolution between a kernel discrepancy and the $\KL$-divergence.
Importantly, the obtained convergence rate depends on the bound on $\| \mu_t - \frac{\Bar{\nu}}{\Bar{m}} \|_{\Dot{H}^{-1}(\mu_t)}$ but is independent of $\lambda > 0$.

\begin{thm} \label{thm:algebraic_convergence}
    Let~\cref{ass:teacher_student} hold. Consider $\lambda > 0$ and some initialization $\mu_0 \in \Pp(\Om)$. Let $(\mu_t)_{t \geq 0}$ be the gradient flow of $\Ll_f^\lambda$, starting from $\mu_0$ (in the sense of~\cref{eq:grad_flow_lambda}). 
    Assume that:
    \begin{itemize}
        \item $\Bar{\nu}$ is a positive measure and $f$ is such that $\min f = f(\Bar{m}) = 0$ where $\Bar{m} \eqdef \Bar{\nu}(\Om) > 0$.
        \item  the gradient flow $(\mu_t)_{t \geq 0}$ is such that $\| \frac{\Bar{\nu}}{\Bar{m}} - \mu_t \|_{\Dot{H}^{-1}(\mu_t)}$ is bounded, uniformly over $t \geq 0$.
    \end{itemize}
    Then there exists a constant $C > 0$ s.t.\@ for any $t \geq 0$:
    \begin{align*}
        \Ll^\lambda_f(\mu_t) \leq \left( \Ll^\lambda_f(\mu_0) ^{-1} + Ct  \right)^{-1}.
    \end{align*}
    In particular, $\mu_t$ converges to $\Bar{\mu} = \Bar{\nu} / \Bar{m}$ when $t$ tends to $+ \infty$.
\end{thm}

\begin{proof}
    Note that it follows from the assumptions on $f$ and $\Bar{\nu}$ that $\inf \Ll^\lambda_f = 0$ and that this infimum is attained only for $\mu = \Bar{\nu} / \Bar{m}$.
    Thus, the last statement on the convergence of $\mu_t$ follows from the convergence of $\Ll^\lambda_f(\mu_t)$ to $0$ and from the lower semicontinuity of $\Ll^\lambda_f$ (see~\cref{sec:minimizer}).
    
    To obtain the convergence rate, consider $\mu \in \Pp(\Om)$ and note that by~\cref{eq:L_dual} we have:
    $$
    \Ll^\lambda_f(\mu) = \max_{\alpha} \Kk(\alpha, \mu) \, ,
    $$
    where for every $\alpha \in L^2(\rho)$ we defined:
    \begin{align*}
        \Kk (\alpha, \mu) \eqdef \int_\Om (\Fmap^\top \alpha) \d \Bar{\nu} - \int_\Om f^*(\Fmap^\top \alpha) \d \mu - \frac{\lambda}{2} \| \alpha \|^2_{L^2(\rho)},
    \end{align*}
    with $f^*$ the Legendre transform of $f$.
    Let us denote by $\alpha^\lambda = \alpha^\lambda_f[\mu]$ the maximizer of $\Kk (\alpha, \mu)$.
    Then, using the convexity of $f^*$, we have for every $\om \in \Om$:
    \begin{align*}
        f^*(0) + \partial f^*(0) (\Fmap^\top \alpha^\lambda)(\om) \leq f^*((\Fmap^\top \alpha^\lambda)(\om)).
    \end{align*}
    Also by assumption $\partial f(\Bar{m}) = 0$ and hence by properties of the Legendre transform $\partial f^*(0) = \Bar{m}$.
    Also $f^*(0) = -f(\Bar{m}) = 0$ and after integrating w.r.t.\@ $\Bar{\nu}$ :
    \begin{align*}
        \int_\Om (\Fmap^\top \alpha^\lambda) \d \Bar{\nu} \leq \int_\Om f^* (\Fmap^\top \alpha^\lambda) \frac{\d \Bar{\nu}}{\Bar{m}}.
    \end{align*}
    Then, replacing $\alpha$ by $\alpha^\lambda$ in $\Kk$ and using the previous inequality:
    \begin{align*}
       \Ll^\lambda_f(\mu) = \Kk(\alpha^\lambda, \mu) \leq \int_\Om f^* (\Fmap^\top \alpha^\lambda) \d ( \frac{\Bar{\nu}}{\Bar{m}} - \mu ) 
        \leq \left\| f^* (\Fmap^\top \alpha^\lambda) \right\|_{\Dot{H}^1(\mu)}   \| \frac{\Bar{\nu}}{\Bar{m}} - \mu \|_{\Dot{H}^{-1}(\mu)} 
    \end{align*}
    Also, by the gradient flow equation~\cref{eq:grad_flow_lambda}, the dissipation of $\Ll^\lambda_f$ along the gradient flow curve $(\mu_t)_{t \geq 0}$ is given for every $t \geq 0$ by:
    \begin{align*}
        \frac{\d}{\d t} \Ll^\lambda_f(\mu_t)
        = - \int_\Om \left\| \nabla (f^* (\Fmap^\top \alpha^\lambda_t) ) \right\|^2 \d \mu_t
        = - \left\| f^* (\Fmap^\top \alpha^\lambda_t) \right\|^2_{\Dot{H}^1(\mu_t)} ,
    \end{align*}
    where $\alpha^\lambda_t = \alpha^\lambda_f[\mu_t]$ maximizes $\Kk(\alpha, \mu_t)$.
    Thus using the previous inequality on $\Ll^\lambda_f(\mu_t)$ and that $ \| \frac{\Bar{\nu}}{\Bar{m}} - \mu_t \|_{\Dot{H}^{-1}(\mu_t)}$ is bounded, uniformly over $t \geq 0$, we get for every $t \geq 0$:
    \begin{align*}
        \frac{\d}{\d t}  \Ll^\lambda_f(\mu_t)  \leq - C^{-1} \Ll^\lambda_f(\mu_t)^2,
    \end{align*}
    for some constant $C > 0$. The desired convergence rate follows from this inequality by applying a Grönwall lemma.
\end{proof}

Let us comment on the assumptions of~\cref{thm:algebraic_convergence}.
The second assumption specifically, is automatically satisfied in case $\Bar{\nu}$ has bounded density and $\mu_t$ has bounded log-density, uniformly over $t \geq 0$.
Indeed, for $\mu \in \Pp(\Om)$ having a lower-bounded log-density, we have that the \emph{weighted} Sobolev seminorm $\|. \|_{\Dot{H}^1(\mu)}$ is lower-bounded by the classical Sobolev seminorm $\| . \|_{\Dot{H}^1(\pi)}$, where we recall that $\pi$ is the (normalized) Lebesgue measure over $\Om$. Precisely, if $\pi \ll \mu$ and $\d \pi / \d \mu \leq C_1$ then for every $f \in \Cc^1$:
\begin{align*}
    \| f \|_{\Dot{H}^1(\pi)} \leq C_1 \| f \|_{\Dot{H}^1(\mu)}.
\end{align*}
In this case, the weighted \emph{negative} Sobolev seminorm $\|. \|_{\Dot{H}^{-1}(\mu)}$ is upper-bounded by the seminorm $\|. \|_{\Dot{H}^{-1}(\pi)}$ and for every $\nu \in \Mm(\Om)$ with $\int_\Om \d \nu = 0$ we have:
\begin{align*}
    \|  \nu \|_{\Dot{H}^{-1}(\mu)} \leq  C_1 \|  \nu \|_{\Dot{H}^{-1}(\pi)}.
\end{align*}
Moreover, this last quantity can be estimated by the Wasserstein distance.
Indeed, for probability measures having bounded log-densities, the Wasserstein distance $\Ww_2$ is equivalent to the negative Sobolev seminorm $\|  . \|_{\Dot{H}^{-1}(\pi)}$.
If $\mu, \nu \in \Pp(\Om)$ are such that $\frac{\d \mu}{\d \pi}, \frac{\d \nu}{\d \pi} \leq C_2$ for some constant $C_2 > 0$ we have~\cite[Lem. 5.33 and Thm. 5.34]{santambrogio2015optimal}:
\begin{align*}
    \| \mu - \nu \|_{\Dot{H}^{-1}(\pi)} \leq C_2^{1/2} \Ww_2(\mu, \nu).
\end{align*}
Finally, the Wasserstein distance $\Ww_2(\mu, \nu)$ is always bounded by $\diam(\Om)$ which is finite, hence ensuring the second assumption of~\cref{thm:algebraic_convergence} is satisfied.

\subsection{Convergence to ultra-fast diffusion.}

The algebraic convergence rate stated in the above~\cref{thm:algebraic_convergence} in the case $\lambda > 0$ stands in contrast with the faster linear convergence stated in~\cref{thm:diffusion_convergence} in the case $\lambda = 0$.
For this reason, we are interested in comparing the gradient flow dynamics with and without regularization.

Below we assume $f(t) = |t|^r / (r-1)$ for some $r > 1$ and \cref{thm:gradient_flow_approximation} shows local uniform in time convergence of gradient flows of $\Ll^\lambda_r$ to gradient flows of $\Ll^0_r$, i.e. solutions to the ultra-fast diffusion equation~\cref{eq:L0_diffusion}, when the regularization strength $\lambda$ vanishes.
To obtain such a result we assume regularity on the density ratio $\frac{\d \Bar{\nu}}{\d \mu^\lambda_t}$.
Namely, we assume that the Legendre-conjugate $\partial f(\frac{\d \Bar{\nu}}{\d \mu^\lambda_t})$ stays bounded in the RKHS $\Hh$, defined as the image of the convolution operator $\Fmap^\top : L^2(\rho) \to \Cc^0(\Om)$ (\cref{eq:RKHS_characterization}).
Using classical results from the theory of inverse problems, such a \emph{source condition} ensures the dual variable $\alpha \in L^2(\rho)$ stays uniformly bounded for $\lambda > 0$ (\cref{lem:source}).
Therefore, provided $\Hh$ is sufficiently regular, such a regularity assumption ensures compactness of the Wasserstein gradient $\nabla \Ll^\lambda_r[\mu_t] = \nabla f^*(\Fmap^\top \alpha^\lambda_t)$ in $\Cc^1$ and allows passing to the limit in~\cref{eq:grad_flow_lambda_weak} to obtain~\cref{eq:grad_flow_L0_weak}.

\begin{thm} \label{thm:gradient_flow_approximation}
    Let~\cref{ass:teacher_student} hold with $\Bar{\nu}$ a positive measure with bounded log-density, consider the regularization function $f(t) = |t|^r / (r-1)$ for some $r > 1$ and let the assumptions of~\cref{thm:wellposed_lambda} and~\cref{thm:wellposed_diffusion} be satisfied.
    Consider some initialization $\mu_0 \in \Pp(\Om)$ s.t.\@ $\mu_0$ has bounded log-density.
    For $\lambda \geq 0$, let $(\mu^\lambda_t)_{t \geq 0}$ be the gradient flow of $\Ll^\lambda_r$ starting from $\mu_0$ in the sense of~\cref{def:grad_flow_lambda} (when $\lambda > 0$) and~\cref{def:grad_flow_L0} (when $\lambda = 0$).
    Moreover, for $\Hh$ defined by~\cref{eq:RKHS_characterization}, assume $\Hh$ is compactly embedded in $\Cc^1(\Om)$ and $\partial f (\frac{\d \Bar{\nu}}{\d \mu^\lambda_t})$ is bounded in $\Hh$, locally uniformly over $t \geq 0$ and uniformly over $\lambda > 0$.
    Then for any $T \geq 0$:
    \begin{align*}
        \lim_{\lambda \to 0^+} \sup_{t \in [0,T]} \Ww_2(\mu_t^0, \mu_t^\lambda) = 0.
    \end{align*}
\end{thm}

\begin{proof}
    For $\lambda > 0$, the curves $(\mu_t)^\lambda$ are gradient flows for the functionals $\Ll^\lambda_r$ and classical computations show that for every $t, s \geq 0$:
    \begin{align*}
        \Ww_2(\mu^\lambda_t, \mu^\lambda_s)^2
        \leq | t-s | \left| \Ll^\lambda_r(\mu^\lambda_t) - \Ll^\lambda_r(\mu^\lambda_s) \right|
        \leq | t-s | \Ll^0_r(\mu_0),
    \end{align*}
    where we used the fact that the functionals $\Ll^\lambda_r$ converge pointwise from below to $\Ll^0_r$. Thus, for $T \geq 0$, the sequence $(\mu_t^\lambda)_{t \in [0,T]}$ is uniformly equicontinuous with value in the compact space $\Pp(\Om)$ and Arzela-Ascoli's theorem ensures the existence of a subsequence $\lambda_n \to 0^+$ s.t.:
    \begin{align*}
         (\mu_t^\lambda)_{t \in [0,T]} \xrightarrow{n \to \infty} (\mu_t)_{t \in [0,T]} \in \Cc^0([0,T], \Pp(\Om)).
    \end{align*}
    To prove the result one needs to identify $\mu_t$ with $\mu_t^0$ and the supplementary regularity assumptions on $\mu_t^\lambda$ are sufficient for this purpose.
    Let us fix some $t \in [0, T]$ and denote by $u^\lambda_t = u^\lambda_f[\mu^\lambda_t]$ the minimizer in~\cref{eq:L_lambda_f}, $\nu^\lambda_t \in \Mm(\Om)$ the minimizer in~\cref{eq:L_mmd} s.t.\@ $\frac{\d \nu^\lambda_t}{\d \mu^\lambda_t} = u^\lambda_t$ and $\alpha^\lambda_t \in L^2(\rho)$ the maximizer in~\cref{eq:L_dual}.

    Then for every $\lambda > 0$, since $\mu_0$ has bounded log-density we have by the flow-map representation in~\cref{prop:flow_representation} that $\mu^\lambda_t$ has bounded log-density.
    Also, since $\Bar{\nu}$ is positive with bounded log-density and $\Fmap \star$ is injective, we have that $u^\dagger_t \eqdef \frac{\d \Bar{\nu}}{\d \mu^\lambda_t}$ is the unique solution to~\cref{eq:L0_f}.
    But then, by the characterization of the RKHS $\Hh$ in~\cref{thm:mercer}, we have that $\Fmap^\top : L^2(\rho) \to \Hh$ is a partial isometry and the assumption that $\partial f (\frac{\d \Bar{\nu}}{\d \mu^\lambda_t}) \in \Hh$ is equivalent to a source condition of the form~\cref{eq:source}.
    Hence by~\cref{lem:source} the dual variable $\alpha^\lambda_t$ is bounded in $L^2(\rho)$, uniformly over $\lambda > 0$, which implies that, up to extraction of a subsequence, $\Fmap^\top \alpha^\lambda_t$ converges to some $h_t$ in $\Cc^1(\Om)$.
    
    Also for every $\lambda > 0$, by the duality relations in~\cref{eq:duality}, we have $\frac{\d \nu^\lambda_t}{ \d \mu^\lambda_t} = \partial f^*(\Fmap^\top \alpha^\lambda_t)$ and hence --- recalling that $f(t) = |t|^r/(r-1)$ for some $r > 1$ --- $\frac{\d \nu^\lambda_t}{ \d \mu^\lambda_t} \to \partial f^*(h_t)$ in $\Cc^0(\Om)$.
    Since $\Ll^\lambda_r(\mu^\lambda_t) \leq \Ll^0_r(\mu_0)$ is bounded we have by~\cref{eq:L_mmd} that $\nu^\lambda_t \to \Bar{\nu}$ narrowly and then for every $\varphi \in \Cc^0(\Om)$:
    \begin{align*}
        \int_\Om \varphi \d \nu^\lambda_t = \int_\Om \varphi \frac{\d \nu^\lambda_t}{\d \mu^\lambda_t} \d \mu^\lambda_t \xrightarrow{\lambda \to 0^+} \int_\Om \varphi \d \Bar{\nu} = \int_\Om \varphi \partial f^*(h_t) \d \mu_t.
    \end{align*}
    This shows that $\Bar{\nu}$ is absolutely continuous w.r.t.\@ $\mu_t$ and that $\frac{\d \Bar{\nu}}{\d \mu_t} = \partial f^*(h_t)$.
    By duality, this is equivalent to $h_t = \partial f(\frac{\d \Bar{\nu}}{\d \mu_t})$, which shows that $\Fmap^\top \alpha^\lambda_t$ converges to $h_t$ in $\Cc^1(\Om)$.
    
    Finally, using the gradient flow equation in~\cref{eq:grad_flow_lambda_weak}, the previously described convergence of $\Fmap^\top \alpha^\lambda_t$ is sufficient to have for every test function $\varphi \in \Cc^\infty_c((0,T) \times \Om)$:
    \begin{align*}
        & \int_0^T \int_\Om \left( \partial_t \varphi_t - \frac{1}{2} \nabla \varphi_t \cdot \nabla f^*(\Fmap^\top \alpha^\lambda_t) \right) \d \mu^\lambda_t \d t = 0 \\
        \xrightarrow{\lambda \to 0^+}
        & \int_0^T \int_\Om \left( \partial_t \varphi_t - \frac{1}{2} \nabla \varphi_t \cdot \nabla f^*(\partial f(\frac{\d \Bar{\nu}}{\d \mu_t})) \right) \d \mu_t \d t = 0.
    \end{align*}
    Since $f(t) = |t|^r/(r-1)$ for some $r > 1$ the above equation is equivalent to~\cref{eq:grad_flow_L0_weak} which shows $\mu_t$ is the weak solution starting from $\mu_0$ of the ultra-fast diffusion equation~\cref{eq:L0_diffusion} according to~\cref{def:grad_flow_L0}, that is $\mu_t = \mu^0_t$.
\end{proof}

The proof of the above~\cref{thm:gradient_flow_approximation} relies on the following result on solutions to inverse problems with nonlinear regularization~\cite{benning2018modern}.
The following~\cref{lem:source} is similar to~\cite[Prop. 3]{iglesias2018note}.

\begin{lem} \label{lem:source}
    Let $f$ satisfy~\cref{ass:regularization} and~\cref{ass:teacher_student}.
    For $\mu \in \Pp(\Om)$, let $u^\dagger \in L^1(\mu)$ be a solution of \cref{eq:L0_f}.
    We say $u^\dagger$ satisfies a \emph{source condition} if there exists $\alpha \in L^2(\rho)$ s.t.
    \begin{align} \label{eq:source}
        \Fmap^\top \alpha \in \partial f (u^\dagger) \quad \text{in $L^1(\mu)$}.
    \end{align}
    Then in this case, noting $\alpha^\dagger \in L^2(\rho)$ the $\alpha$ of minimal norm satisfying the above source condition, we have for every $\lambda > 0$:
    \begin{align*}
        \| \alpha^\lambda_f[\mu] \|_{L^2(\rho)} \leq  \| \alpha^\dagger \|_{L^2(\rho)} \quad \text{and} \quad \alpha^\lambda_f[\mu] \xrightarrow{\lambda \to 0^+} \alpha^\dagger
    \end{align*}
    where $\alpha^\lambda_f[\mu]$ is the solution to~\cref{eq:L_dual}.
\end{lem}

\begin{proof}
    Let $u^\dagger \in L^1(\mu)$ and $\alpha^\dagger \in L^2(\rho)$ be as in the statement. By the source condition~\cref{eq:source} we have in $L^1(\mu)$:
    \begin{align*}
        - f^*(\Fmap^\top \alpha^\dagger) + (\Fmap^\top \alpha^\dagger) u^\dagger \geq f(u^\dagger) 
    \end{align*}
    and integrating w.r.t.\@ $\mu$ and using that $\int_\Om (\Fmap^\top \alpha^\dagger) u^\dagger \d \mu = \left< \alpha^\dagger, Y \right>_{L^2(\rho)}$ we obtain:
    \begin{align*}
        - \int_\Om f^*(\Fmap^\top \alpha^\dagger) \d \mu + \left< \alpha^\dagger, Y \right>_{L^2(\rho)} \geq \int_\Om f(u^\dagger) \d \mu = \inf_{\Fmap_\mu \cdot u = Y} \int_\Om f(u) \d \mu.
    \end{align*}
    Thus, $\alpha^\dagger$ achieves the supremum in~\cref{eq:L_dual} with $\lambda=0$ and we have for every other $\alpha \in L^2(\rho)$:
    \begin{align*}
        - \int_\Om f^*(\Fmap^\top \alpha^\dagger) \d \mu + \left< \alpha^\dagger, Y \right>_{L^2(\rho)} \geq - \int_\Om f^*(\Fmap^\top \alpha) \d \mu + \left< \alpha, Y \right>_{L^2(\rho)}.
    \end{align*}
    Moreover, for $\lambda > 0$, noting $\alpha^\lambda \eqdef \alpha^\lambda_f[\mu]$, we have by definition:
    \begin{align*}
        - \int_\Om f^*(\Fmap^\top \alpha^\lambda) \d \mu + \left< \alpha^\lambda, Y \right>_{L^2(\rho)} - \frac{\lambda}{2} \| \alpha^\lambda \|^2_{L^2(\rho)} \geq - \int_\Om f^*(\Fmap^\top \alpha^\dagger) \d \mu + \left< \alpha^\dagger, Y \right>_{L^2(\rho)} - \frac{\lambda}{2} \| \alpha^\dagger \|^2_{L^2(\rho)}.
    \end{align*}
    Substracting the two previous inequalities and simplifying gives:
    \begin{align*}
        \| \alpha^\lambda \|_{L^2(\rho)} \leq \| \alpha^\dagger \|_{L^2(\rho)}.
    \end{align*}
    Thus $\alpha^\lambda$ is bounded, uniformly over $\lambda > 0$. For the convergence part, note that since it is bounded in $L^2(\rho)$ it converges weakly to some $\alpha^0 \in L^2(\rho)$. Also taking the optimality condition for $\alpha^\lambda$ and taking the limit when $\lambda \to 0^+$ we obtain for every $\alpha \in L^2(\rho)$:
    \begin{align*}
        & - \int_\Om f^*(\Fmap^\top \alpha^\lambda) \d \mu + \left< \alpha^\lambda, Y \right>_{L^2(\rho)} - \frac{\lambda}{2} \| \alpha^\lambda \|^2_{L^2(\rho)} \geq - \int_\Om f^*(\Fmap^\top \alpha) \d \mu + \left< \alpha, Y \right>_{L^2(\rho)} - \frac{\lambda}{2} \| \alpha \|^2_{L^2(\rho)} \\
        \xrightarrow{\lambda \to 0+} &
        - \int_\Om f^*(\Fmap^\top \alpha^0) \d \mu + \left< \alpha^0, Y \right>_{L^2(\rho)} \geq - \int_\Om f^*(\Fmap^\top \alpha) \d \mu + \left< \alpha, Y \right>_{L^2(\rho)},
    \end{align*}
    which shows $\alpha^0$ is also a maximizer of the dual problem~\cref{eq:L_dual} when $\lambda = 0$ and, as a consequence, also satisfies the source condition~\cref{eq:source}.
    But, by minimality of the norm of $\alpha^\dagger$ and by weak lower semicontinuity of the norm we have:
    \begin{align*}
        \| \alpha^\dagger \|_{L^2(\rho)} \leq  \| \alpha^0 \|_{L^2(\rho)} \leq \liminf_{\lambda \to 0^+}  \| \alpha^\lambda \|_{L^2(\rho)} \leq  \| \alpha^\dagger \|_{L^2(\rho)}
    \end{align*}
    which shows that in fact $\alpha^\lambda \to \alpha^\dagger$ strongly in $L^2(\rho)$.
\end{proof}

\section{Numerics} \label{sec:numerics}

We report in this section numerical results.
First, to assess the validity of our theory, we tested the VarPro algorithm on simple low-dimensional examples with synthetic data:
experiments with a $1$-dimensional feature space are detailed in~\cref{subsec:numerics_1d} and supplementary experiments in $2$-d are detailed in~\cref{sec:numerics_2d}.
Those experiments indicate that, when the regularization is sufficiently low, the VarPro dynamic indeed enters an ultra-fast diffusion regime where the student feature distribution converges to the teacher's at a linear rate.
Moreover, if the stepsize is sufficiently small, the VarPro dynamic can also be efficiently approximated by a two-timescale learning strategy.

Finally, to investigate the large-scale applicability and generalization capabilities of the VarPro algorithm, we tested it on an image classification problem with the CIFAR10 dataset~\cite{krizhevsky2009learning} and compare its performances with other standard stochastic optimization methods.
Those results are detailed in~\cref{subsec:numerics_cifar}.

The code for reproducing the results is available at:
\url{https://github.com/rbarboni/VarPro}.

\input{numerics_1d_alt}

\subsection{VarPro for image classification on CIFAR10}
\label{subsec:numerics_cifar}

We conclude this section by performing experiments on an image classification task with the CIFAR10 dataset~\cite{krizhevsky2009learning}.
We thereby aim at testing the large scale applicability of the VarPro algorithm.
Note that applications of variable projection strategies to the training of deep neural network architectures were also studied in~\cite{newman2021train}.
However, such setting goes outside of the scope of the theory developed in this paper as the neural network can no longer be represented as a linear operator acting on measure.

We consider a \emph{residual neural network (ResNet)} architecture with $20$ layers and $0.27$M parameters, whose precise description can be found in~\cite[Sec. 4.2]{he2016deep}.
This model has a Euclidean parameter space $\Theta$ and for parameters $\theta \in \Theta$ and images $x \in \RR^{3 \times 32 \times 32}$ it produces features which we denote by $\resnet(\theta, x) \in \RR^M$, with $M = 64$.
We consider the last fully connected layer separately as a weight matrix $U \in \RR^{c \times M}$ with here $c=10$ the number of classes.
Overall, for parameters $(\theta, U) \in \Theta \times \RR^{c \times M}$ and an input image $x \in \RR^{3 \times 32 \times 32}$, the output of the model is given by:
\begin{align*}
    F_{(\theta, U)} (x) \eqdef \frac{1}{M} U \cdot \resnet(\theta, x) \in \RR^c
\end{align*}
To apply the VarPro algorithm we need to have an efficient way of computing the exact projection of the linear parameters $U$.
For this purpose and instead of a cross-entropy loss, we consider here simply the square error between the outputs of our model and the true labels converted to one-hot vectors $y \in \{0,1\}^c$.
In this manner, the training risk for a batch of data $\Bb$ and parameters $(\theta, U) \in \Theta \times \RR^{c \times M}$ reads:
\begin{align} \label{eq:resnet_risk}
    \Hat{\Rr}^\lambda_\Bb(\theta, U) \eqdef \frac{1}{2 |\Bb|} \sum_{(x,y) \in \Bb} \| F_{(\theta, U)}(x) - y \|^2 + \frac{\lambda}{2 M} \| U \|^2. 
\end{align}
For an initialization $(\theta^0, U^0) \in \Theta \times \RR^{c \times M}$, a stepsize $\tau > 0$ and a momentum parameter $\m > 0$ the training dynamic reads:
\begin{align} \label{eq:resnet_varpro}
    \forall k \geq 0, \quad
    \left\{
    \begin{array}{rcl}
        U^{k+1} & = & \m U^k + (1-\m) \Bar{U}^k   \\[5pt]
        \theta^{k+1} & = & \theta^k - \frac{ M \tau}{\lambda} \nabla_\theta \Hat{\Rr}^\lambda_{\Bb_k}(\theta^k, U^{k+1})
    \end{array}
    \right.
\end{align}
where $\Bb_k$ is the mini-batch at step $k$ and $\Bar{U}_k$ is the corresponding projection of the outer weights i.e.
$\Bar{U}^k \in \argmin_{U \in \RR^{c \times M}} \Hat{\Rr}^\lambda_{\Bb_k}(\theta^k, U)$.

Note the introduction of the momentum parameter $\m > 0$ which is here to compensate the variability of the projection $\Bar{U}_k$ w.r.t.\@ the sampling of mini-batches at each step.
Indeed, intuitively, for evaluation on test-data, rather than having a classifier computed only on the last mini-batch, it is preferable to have an average of the last computed classifiers.

\paragraph{Experimental setting}
In practice, we find it effective to consider a regularization strength $\lambda = 10^{-3}$, a momentum $\m = 0.9$ and a stepsize $\tau = 10^3$.
We consider different values of the batch size $|\Bb| \in \{ 64, 128, 256, 512, 1024 \}$.
We train our model by performing $110$ passes over the training set, evaluating the model accuracy on the test set in-between each pass.
The stepsize is divided by $2$ for the last $10$ passes on the training set.
Note that this setting allows for a fair comparison of performances with the results presented in~\cite[Sec. 4.2]{he2016deep} for the training of ResNets on the CIFAR10 dataset.
We also follow the same data-augmentation procedure.

\paragraph{Comparison with other stochastic optimization methods}

We compare the above described VarPro algorithm (\cref{eq:resnet_varpro}) with other stochastic optimization methods for the  minimization of the training risk in~\cref{eq:resnet_risk}.
We compare with standard \emph{Stochastic Gradient Descent (SGD)} on the full parameterization $(\theta, U) \in \Theta \times \RR^{c \times M}$ with momentum $\m = 0.9$ and stepsize $\tau = 10^{-3}$.
We also compare with the \emph{Shampoo} algorithm~\cite{gupta2018shampoo} which is a preconditioned gradient method\footnote{We used the implementation from \url{https://github.com/moskomule/shampoo.pytorch}} and set the learning rate to $\tau = 10^{-2}$.

\Cref{fig:resnet_risk} reports the evolution of the training risk (\cref{eq:resnet_risk}) along training.
One can observe that, in terms of minimization of the training risk, performances of VarPro at convergence are similar to the one of SGD and better than Shampoo.
Compared with these last two methods, the convergence speed of VarPro however seems to be slower during the first stages of training.
Behavior of the methods w.r.t.\@ the batch size is also different.
While the batch size has no or little influence on the convergence speed of SGD or Shampoo, one can observe that the VarPro algorithm tends to converge more slowly when the batch size increases.
Since this method is based on the exact resolution of a quadratic minimization problem on each mini-batch at each step, an explanation is thus that this subproblem becomes less well-conditioned when the size of the mini-batches increases.

\Cref{fig:resnet_risk} reports the evolution of the top-$1$ accuracy of the model on the test set.
All optimization methods seems to achieve the same generalization performances on the test set, that is more than $90\%$ accuracy, which is in par with the $91.25\%$ reported for the same model in~\cite{he2016deep}.
As before, one can observe the Varpro algorithm (\cref{eq:resnet_varpro}) seems to take more time to achieve the same accuracy.
Also, whereas SGD and Shampoo generalize better when the batch size is smaller, the converse happens for VarPro and one can see the generalization performance of our ResNet model trained with the VarPro algorithm deteriorates for smaller batch sizes.

\begin{figure}
    \centering
    \centerline{
    \includegraphics[scale=0.85]{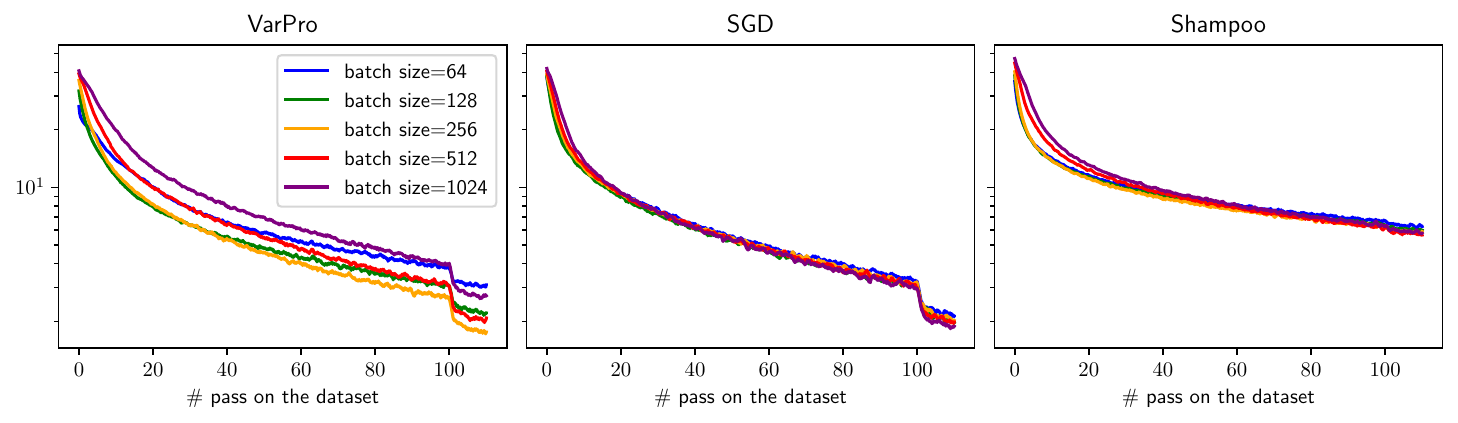}
    }
    \caption{Evolution of the training risk $\frac{1}{\lambda} \Hat{\Rr}^\lambda_\Bb$ (\cref{eq:resnet_risk}) along training for different batch sizes and different optimization methods.
    Plots are averages of the risk associated to each mini-batch encountered during one pass.
    VarPro corresponds to~\cref{eq:resnet_varpro}.}
    \label{fig:resnet_risk}
\end{figure}

\begin{figure}
    \centering
    \centerline{
    \includegraphics[scale=0.85]{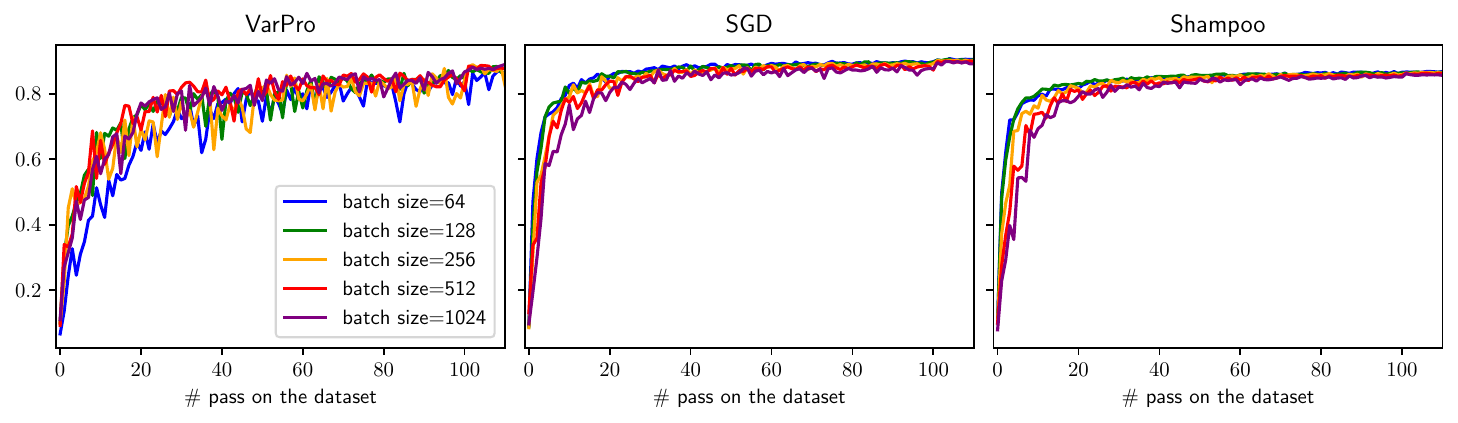}
    }
    \caption{Evolution of the top-$1$ accuracy along training for different batch sizes and different optimization methods. VarPro corresponds to~\cref{eq:resnet_varpro}.}
    \label{fig:resnet_accuracy}
\end{figure}

\section*{Conclusion}

In this work we have investigated the convergence of gradient based methods for the training of mean-field models of neural networks.
To this end, we have adopted a \emph{Variable Projection (VarPro)} strategy, which eliminates the linear parameters and reduces the training problem to the learning of the nonlinear features.
Using tools from the theory of Wasserstein gradient flows, we have shown theoretically that, when the regularization strength $\lambda$ vanishes, the training dynamic converges, under regularity assumptions, to solutions of \emph{weighted ultra-fast diffusion} PDEs (\cref{thm:gradient_flow_approximation}).
In such a low regularization regime, this allows establishing convergence of the learned feature distribution to the teacher's at a linear rate (\cref{thm:diffusion_convergence}).
Moreover, in presence of regularization, we also obtain a quantitative convergence result but with a slower algebraic rate (\cref{thm:algebraic_convergence}).

Our theoretical predictions are supported by numerical results on simple experiments with synthetic data.
One can observe that, when the regularization strength $\lambda$ is negligible compared to the tangent kernel, the VarPro and ultra-fast diffusion dynamics are similar and converge to the teacher feature distribution at a linear rate (\cref{fig:relu1d_lmbda_distance}).
Moreover, if the time step is sufficiently small, this dynamic is also recovered with a simple two-timescale gradient descent algorithm (\cref{fig:relu1d_2ts}).
Finally, experiments with a ResNet architecture on the CIFAR10 dataset show that a VarPro strategy can be easily adapted to the training of complex architectures on large datasets.

We conclude by mentioning possible future research directions:
\begin{itemize}
    \item On a theoretical perspective, our convergence results in~\cref{sec:convergence} hold under regularity assumptions on the training dynamic.
    It would be interesting to see if one can relax or ensure these assumptions, possibly by strengthening~\cref{ass:teacher_student}.
    \item We have considered here simple $2$-layer neural networks but it has been pointed out by several works that the learning of good nonlinear representations of the data also plays an important role in the training of deep architectures such as ResNets or Transformers~\cite{barboni2022global,barboni2024understanding,gao2024global}.
    It might thus be interesting to see in what extent the two-timescale approach could be extended to deeper architectures.
    A difficulty is that, in deep architectures, separability of the regression problem w.r.t.\@ linear and nonlinear variables of each layer is lost due to composition.
\end{itemize}

\clearpage

\section*{Acknowledgements}

The work of G. Peyr\'e was supported by the European Research Council (ERC project WOLF) and the French government under the management of Agence Nationale de la Recherche as part of the “France 2030” program, reference ANR-23-IACL-0008 (PRAIRIE-PSAI).
The work of F.-X. Vialard was supported by the Bézout Labex (New Monge Problems), funded by ANR, reference ANR-10-LABX-58.
This work was performed using HPC resources from GENCI-IDRIS (Grant 2025-AD011013400R3).

\printbibliography

\clearpage

\appendix

\section{Positive definite kernels and RKHS} \label{sec:rkhs}

We recall in this section basic properties on the theory of Reproducing Kernel Hilbert Spaces and refer to classical textbooks for a complete presentation of the topic~\cite{steinwart2008support,scholkopf2002learning}.
In this work we consider a mapping $\fmap : \Om \times \RR^d \to \RR$ as well as a probability measure $\rho \in \Pp(\RR^d)$. This choice of $\fmap$ and $\rho$ determines a \emph{symmetric, positive definite kernel}~\cite[Def. 4.15]{steinwart2008support} $\kappa :  \Om \times \Om \to \RR$ defined by:
\begin{align*}
    \forall \om, \om' \in \Om, \quad \kappa(\om, \om') \eqdef \int_X \fmap(\om, x) \fmap(\om', x) \d \rho(x) = \left< \fmap(\om,.), \fmap(\om', .) \right>_{L^2(\rho)}.
\end{align*}
Thus, associated to $\kappa$ is a (uniquely defined) structure of \emph{Reproducing Kernel Hilbert Space (RKHS)} $\Hh$ with scalar product $\left< ., . \right>_\Hh$, that is a Hilbert space of functions on $\Om$ such that~\cite[Def. 4.18]{steinwart2008support}: (i) $\kappa(\om, .) \in \Hh$ for every $\om \in \Om$ and (ii) the following \emph{reproducing property} holds:
\begin{align*}
    \forall h \in \Hh, \om \in \Om, \quad h(\om) = \left< h, \kappa(\om,.) \right>_\Hh.
\end{align*}
Following the definition of $\kappa$, $L^2(\rho)$ is a \emph{feature space} for $\Hh$ and $\fmap$ a \emph{feature map}~\cite[Def. 4.1]{steinwart2008support}.
Also, $\Hh$ can be isometrically identified as a subspace of $L^2(\rho)$ and the convolution with $\fmap$ is a partial isometry~\cite[Thm. 4.21]{steinwart2008support}. Precisely, we have $\Hh = \lbrace h : \Om \to \RR \; | \; \exists \alpha \in L^2(\rho), h = \int_{\RR^d} \fmap(.,x) \alpha(x) \d \rho(x)  \rbrace$ and the RKHS norm on $\Hh$ satisfies:
\begin{align} \label{eq:RKHS_characterization}
    \forall h \in \Hh, \quad \| h \|_\Hh = \inf \left\{ \| \alpha \|_{L^2(\rho)} \; | \; h = \int_{\RR^d} \fmap(.,x) \alpha(x) \d \rho(x)   \right\} .
\end{align}
In this paper, we always work with the following minimal assumption on the feature map $\fmap$:
\begin{assumption}[Assumption on $\fmap$] \label{ass:feature_map}
    \hfill \\
    The feature map $\fmap$ is in $L^2(\rho, \Cc^0)$. In particular, it implies the kernel $\kappa$ is continuous on $\Om \times \Om$.
\end{assumption}

\paragraph{Kernel embeddings and kernel discrepancy between measures}

The above assumption is sufficient to ensure $\Hh$ is included in $\Cc(\Om)$ and guarantees the existence of kernel embeddings for finite Borel measures~\cite{muandet2017kernel,gretton2012kernel}. For a measure $\nu \in \Mm(\Om)$ its \emph{kernel embedding} $M_\kappa(\nu)$ is defined as the unique element of $\Hh$ satisfying:
\begin{align} \label{eq:kernel_embedding}
    \forall h \in \Hh, \quad \int_\Om h \d \nu = \left< h, M_\kappa(\nu) \right>_\Hh.
\end{align}
Equivalently, the kernel embedding is given by the Bochner integral $M_\kappa(\nu) = \int_\Om \kappa(., \om) \d \nu(\om) \in \Hh$.
This embedding defines a discrepancy between measures by seeing them as element of the Hilbert space $\Hh$. For two measures $\nu, \nu' \in \Mm(\Om)$ the \emph{Maximum Mean Discrepancy (MMD)} between $\nu$ and $\nu'$ is defined as~\cite{muandet2017kernel,gretton2012kernel}:
\begin{align*}
    \MMD_\kappa(\nu, \nu') \eqdef \left\| M_\kappa(\nu) - M_\kappa(\nu') \right\|_\Hh.
\end{align*}
Alternatively, \cref{ass:feature_map} is sufficient to ensure the operator $\Fmap \star : \Mm(\Om) \to L^2(\rho)$ defined in~\cref{eq:Fmap_star} is bounded and by construction we have:
\begin{align} \label{eq:MMD}
    \MMD_\kappa(\nu, \nu') = \left( \iint_{\Om \times \Om} \kappa(\om, \om') \d (\nu-\nu')(\om) \d (\nu-\nu')(\om') \right)^{1/2} = \left\| \Fmap \star (\nu - \nu') \right\|_{L^2(\rho)} .
\end{align}
The discrepancy $\MMD_\kappa$ is in particular a distance between measures whenever the kernel $\kappa$ is \emph{universal}, that is when the associated RKHS $\Hh$ is dense in the space of continuous functions on $\Om$~\cite{micchelli2006universal,sriperumbudur2011universality}.
One can show this condition is equivalent to an injectivity assumption on $\Fmap \star$.

\begin{lem}[see also {\cite[Prop. 1]{micchelli2006universal}}] \label{lem:rkhs_universality}
    Let $\fmap$ satisfy~\cref{ass:feature_map}. Then $\Fmap \star : \Mm(\Om) \to L^2(\rho)$ is injective if and only if $\Hh$ is dense in the space $\Cc^0(\Om)$ of continuous functions over $\Om$. In this case, $\MMD_\kappa$ is a distance on $\Mm(\Om)$.
\end{lem}

\begin{proof}
    The fact the $\MMD$ is a distance on $\Mm(\Om)$ when $\Fmap \star$ is injective directly follows from~\cref{eq:MMD}.
    For the direct implication, assume $\Fmap \star$ is injective and consider some measure $\nu \in \Hh^\perp$ i.e. such that for every $h \in \Hh$ we have $\int h \d \nu = 0$. Then by the characterisation in~\cref{eq:RKHS_characterization} we have for every $\alpha \in L^2(\rho)$:
    \begin{align*}
        0 = \int_\Om \left( \int_{\RR^d} \fmap(\om,x) \alpha(x) \d \rho(x) \right) \d \nu(\om) = \left< \alpha, \Fmap \star \nu \right>_{L^2(\rho)}.
    \end{align*}
    Hence $\Fmap \star \nu = 0$, implying $\nu = 0$ and thus that $\Hh^\perp = \{ 0 \}$ i.e. $\Hh$ is dense in $\Cc^0(\Om)$ by Hahn-Banach theorem. For the converse implication, assume that $\Hh$ is dense in $\Cc^0(\Om)$ and consider some $\nu \in \Mm(\Om)$ s.t.\@ $\Fmap \star \nu = 0$. Then for every $\alpha \in L^2(\rho)$ we have $\left< \alpha, \Fmap \star \nu \right>_{L^2(\rho)} = 0$ and by similar calculations this implies $\nu \in \Hh^\perp$ i.e. $\nu = 0$.
\end{proof}

\paragraph{Kernel and integral operator}

In this work we have used properties of the RKHS $\Hh$ seen as a subspace of the Hilbert space $L^2(\mu)$ for probability measures $\mu \in \Pp(\Om)$.
For such a probability measure $\mu \in \Pp(\Om)$, it indeed follows from~\cref{ass:feature_map} that $\Hh$ is compactly embedded in $L^2(\mu)$~\cite[Lem. 2.3]{steinwart2012mercer}.
Also, the kernel defines an integral operator $J_\mu : L^2(\mu) \to L^2(\mu)$ given by:
\begin{align*}
    \forall f \in L^2(\mu), \quad J_\mu \cdot f = \int_\Om k(. , \om) f(\om) \d \mu(\om).
\end{align*}
Then $J_\mu$ is a compact, self-adjoint and positive operator and, by the spectral theorem, it can be diagonalized in an orthonormal basis $(e_i)_{i \geq 0}$ of $L^2(\mu)$ with associated eigenvalues $(\lambda_i)_{i \geq 0}$ s.t.\@ $\lambda_1 \geq \lambda_2 \geq ...\geq 0$. In particular, $J_\mu = \Fmap_\mu^\top \Fmap_\mu$ with $\Fmap_\mu : L^2(\mu) \to L^2(\rho)$ the \emph{feature operator} defined in~\cref{eq:feature_operator}, thus $(\sqrt{\lambda_i})_{i \geq 0}$ are the (right) singular values of $\Fmap_\mu$ (which is a compact operator) and, if $\Fmap \star$ is injective, then $\lambda_i > 0$ for every $i \geq 0$. Mercer's theorem gives a representation of the kernel $\kappa$ and of the associated RKHS $\Hh$ in terms of this eigenvalue decomposition~\cite[Thm. 4.51]{steinwart2008support}.

\begin{thm}[Mercer representation of RKHSs] \label{thm:mercer}
    Assume $\Fmap \star$ is injective and let $\mu \in \Pp(\Om)$ be a probability measure with full support on $\Om$. Consider $(\lambda_i)_{i \geq 0}$ and $(e_i)_{i \geq 0}$ the eigenvalue decomposition of the operator $J_\mu$. Then we have:
    \begin{align*}
        \forall \om, \om' \in \Om \times \Om, \quad \kappa(\om, \om') = \sum_{i \geq 0} \lambda_i e_i(\om) e_i(\om'),
    \end{align*}
    where the convergence is absolute and uniform over $\Om \times \Om$. Moreover
    \begin{align*}
        \Hh = \left\{ \sum_{i \geq 0} a_i \sqrt{\lambda_i} e_i, \, (a_i)_{i \geq 0} \in \ell^2(\NN) \right\}
    \end{align*}
    is the RKHS associated to the kernel $\kappa$ when provided with the scalar product $\left< .,. \right>_\Hh$ defined for $f = \sum_{i \geq 0} a_i \sqrt{\lambda_i}$ and $g = \sum_{i \geq 0} b_i \sqrt{\lambda_i}$ by
    $\left< f, g \right>_\Hh = \sum_{i \geq 0} a_i b_i$.
\end{thm}

\include{numerics_2d}

\end{document}

%% file: subsec1.1_alt.tex
\subsection{Mean-field neural networks and two-timescale learning}

We consider in this work neural networks with parameter in the \emph{parameter space} $\Om$, which is assumed to be either the $n$-dimensional torus $\TT^n = \RR^n / \ZZ^n$, or a closed, bounded and convex domain of $\RR^n$.
In the following, we are given a map $\fmap : \Om \times \RR^d \to \RR$, called \emph{feature map}, and for an integer $M \geq 1$, we define a \emph{single-hidden-layer (SHL)} neural network of width $M$ with \emph{inner weights} $\lbrace \om_i \rbrace_{1 \leq i \leq M} \in \Om^M$ and \emph{outer weights} $\lbrace u_i \rbrace_{1 \leq i \leq M} \in \RR^M$ as the map:
\begin{align} \label{eq:SHL}
    F_{\lbrace (\om_i, u_i) \rbrace} : x \in \RR^d \mapsto \frac{1}{M} \sum_{i=1}^M u_i \fmap(\om_i, x) \in \RR,
\end{align}
taking inputs in the \emph{input space} $\RR^d$ and returning values in the \emph{output space} $\RR$.
Thanks to the interchangeability of the indices and the normalisation factor $1/M$, the above model can be reparameterized in terms of the empirical distribution of the inner weights $\lbrace \om_i \rbrace_{1 \leq i \leq M}$.
Given an arbitrary probability distribution $\mu \in \Pp(\Om)$ on the space of inner weights and a measurable map $u \in L^1(\mu)$ we define:
\begin{align} \label{eq:SHL_mean_field}
    F_{\mu, u} : x \in \RR^d \mapsto \int_\Om u(\om) \fmap(\om, x) \d \mu(\om) \in \RR.
\end{align}
In particular, for the empirical distribution $\Hat{\mu} = \frac{1}{M} \sum_{i=1}^M \delta_{\om_i}$ and the outer weights $\Hat{u}(\om_i) = u_i$ we recover the finite width SHL $F_{\Hat{\mu}, \Hat{u}} = F_{\lbrace (\om_i, u_i) \rbrace}$.
Such a ``mean-field'' model of neural network has been proposed by several authors to study the training of neural networks at arbitrary large width~\cite{chizat2018global,mei2019mean,rotskoff2019global,sirignano2020mean}.

\paragraph{Supervised learning}

In the context of supervised learning, training a neural network consists in minimizing a \emph{training risk} associated to the evaluation of the model on some \emph{training data}.
We consider in this work a univariate regression setting where the neural network weights are trained for minimizing the mean square error with a \emph{target signal} $Y \in L^2(\rho)$ evaluated on training data with distribution $\rho \in \Pp(\RR^d)$.
For a regularization strength $\lambda > 0$ and $\mu \in \Pp(\Om)$, $u \in L^1(\mu)$ we define the training risk as:
\begin{align} \label{eq:risk}
    \Rr^\lambda(\mu, u) \eqdef \frac{1}{2} \Vert F_{\mu,u} - Y \Vert^2_{L^2(\rho)} + \lambda \Vert u \Vert^2_{L^2(\mu)},
\end{align}
where we assume $\Rr^\lambda(\mu, u) = + \infty$ if $u \notin L^2(\mu)$.
Training the neural network then amounts to finding parameters $(\mu, u) \in \argmin \Rr^\lambda$.

\paragraph{Example of applications}
Note that the mean-field neural network model of~\cref{eq:SHL_mean_field} can be seen as a linear model acting on (signed) measures.
Indeed, for $\mu \in \Pp(\Om)$ and $u \in L^1(\mu)$, we have $F_{\mu, u} = \Fmap \star (u\mu)$ where for every finite Borel measure $\nu \in \Mm(\Om)$ we define:
\begin{align} \label{eq:Fmap_star}
    \Fmap \star \nu \eqdef \int_\Om \fmap(\om, .) \d \nu(\om) \, .
\end{align}
In turn, minimization of functionals of the form in~\cref{eq:risk} with linear models acting on the space of measures have numerous applications depending on the choice of the feature map $\fmap$.
\begin{itemize}
    \item \underline{Two-layer perceptron:} The perceptron model is arguably the prototypical example of a neural network.
    It consists in considering $\Om \subset \RR^{d+1}$ and a feature map $\fmap : (\om, x) \mapsto \activation( \om^\top \Bar{x})$ where $\Bar{x} = (x,1) \in \RR^{d+1}$ and $\activation : \RR \to \RR$ is some nonlinear activation such as the \emph{Rectified Linear Unit (ReLU)} or hyperbolic tangent.
    Owing to their great expressivity~\cite{cybenko1989approximation}, this class of models is ubiquitous in applications where an unknown signal is to be recovered from data observations.
    \item \underline{Radial Basis Function (RBF) neural networks and signal deconvolution:}
    RBF neural networks \cite{pereyra2006variable,karamichailidou2024radial} is an example of a simple architecture in which the feature map consists of a translation invariant kernel $k$ i.e. $\Om \subset \RR^d$ and $\fmap : (\om, x) \mapsto k(\om-x)$.
    The network $F_{\mu, u}$ then implements a convolution with the kernel $k$ and minimization of the risk $\Rr^\lambda$ amounts to solve a form of deconvolution problem.
    This has important applications in signal processing where one wants to recover an unknown signal given noisy or filtered observations~\cite{de2012exact,duval2015exact}.
\end{itemize}

\paragraph{Training with gradient descent and two-timescale learning}

In supervised learning, minimization of the training risk is usually performed using first order optimization methods such as gradient descent or stochastic variants on the neural network's weights~\cite{bottou2018optimization}.

For a SHL of finite width $M \geq 1$ with weights $\{ (\om_i, u_i) \}_{1 \leq i \leq M} \in (\Om \times \RR)^M$ the associated risk is $\Hat{\Rr}^\lambda(\{ (\om_i, u_i) \}_{1 \leq i \leq M}) \eqdef \Rr^\lambda (\Hat{\mu}, \Hat{u})$, where $\Hat{\mu} = \frac{1}{M} \sum_{i=1}^M \delta_{\om_i}$ and $\Hat{u}(\om_i) = u_i$.
For an initialization $\{ (\om_i^0, u_i^0) \}_{1 \leq i \leq M}$, a step-size $\tau > 0$ and a timescale parameter $\eta > 0$, the \emph{gradient descent} dynamic reads:
\begin{align} \label{eq:gradient_descent_risk}
    \forall k \geq 0, \quad \forall i \in \{ 1, ..., M \}, \quad
    \left\{
    \begin{array}{rcl}
        \om_i^{k+1} &  = & \om_i^k - M \tau \nabla_{\om_i} \Hat{\Rr}^\lambda( \lbrace (\om_i^k, u_i^k) \rbrace_{1 \leq i \leq M} ) \\[8pt]
        u_i^{k+1} &  = & u_i^k -\eta M \tau \nabla_{u_i} \Hat{\Rr}^\lambda(\lbrace (\om_i^k, u_i^k) \rbrace_{1 \leq i \leq M} )
    \end{array}
    \right.
\end{align}
For the purpose of theoretical analysis we study here the limit of the gradient descent algorithm when the step-size $\tau$ tends to $0$.
For an initialization $\{ (\om_i(0), u_i(0)) \}_{1 \leq i \leq M}$, this \emph{gradient flow} dynamic reads:
\begin{align} \label{eq:gradient_flow_risk}
    \forall i \in \{ 1, ..., M \}, \quad
    \left\{
    \begin{array}{rcl}
        \frac{\d}{\d t} \om_i(t) &  = & - M \nabla_{\om_i} \Hat{\Rr}^\lambda( \lbrace (\om_i(t), u_i(t)) \rbrace_{1 \leq i \leq M} ) \\[8pt]
        \frac{\d}{\d t} u_i(t) &  = & -\eta M \nabla_{u_i} \Hat{\Rr}^\lambda(\lbrace (\om_i(t), u_i(t)) \rbrace_{1 \leq i \leq M} )
    \end{array}
    \right.
\end{align}
Note the role of the \emph{timescale parameter} $\eta > 0$ controlling the ratio of learning timescales between inner and outer weights.
When $\eta < 1$ the outer-weights $u_i$ are learned more ``slowly'' than the inner-weights $\om_i$ and conversely, when $\eta > 1$ the outer-weights $u_i$ are learned more ``quickly'' than the inner-weights $\om_i$.
In particular, the limiting training dynamics when $\eta \to +\infty$ correspond (formally) to the case where the outer weights are learned ``instantaneously'', that is, at each time $t \geq 0$, we have $\{ u_i(t) \}_{1 \leq i \leq M} \in \argmin_{u \in \RR^M} \Hat{\Rr}^\lambda(\lbrace (\om_i(t), u_i) \rbrace_{1 \leq i \leq M} )$.
Such limiting dynamics correspond to a variable projection algorithm.

\paragraph{Variable Projection}

The \emph{Variable Projection (VarPro)} algorithm performs elimination of the linear variable $u$ and enables here reducing the training of a neural network to the sole problem of learning the feature distribution.
Introduced in~\cite{golub1973differentiation} for the minimization of separable nonlinear least squares problems,
such a strategy has proven to be efficient in various applications~\cite{golub2003separable,osborne2007separable} including the training of neural networks~\cite{sjoberg1997separable,pereyra2006variable,newman2021train,karamichailidou2024radial}.
A reason for this popularity is that partial optimization over one variable can lead to a better conditioning of the Hessian~\cite{sjoberg1997separable,vialard2022partial}.

Exploiting here the linearity w.r.t. the outer weights in the definition of $F$, it is convenient to read a neural network's output $F_{\lbrace(\om_i, u_i) \rbrace}(x) = \frac{1}{M} \sum u_i \fmap(\om_i,x)$ as a linear combination of the \emph{features} $\lbrace \fmap(\om_i,x) \rbrace_{i=1}^M$.
From this point of view, neural networks should be compared to \emph{kernel methods} for which the features are built in advance and fixed during training, whereas only the weights of the linear combination are learned~\cite{hofmann2008kernel}.
In contrast, both inner weights $\lbrace \om_i \rbrace_{i=1}^M$ and outer weights $\lbrace u_i \rbrace_{i=1}^M$ of a neural network are usually trained.
In the following, we refer to the parameters $\om \in \Om$ as the neural network's \emph{features} and to $\mu \in \Pp(\Om)$ as the \emph{feature distribution}.
More generally in the mean-field limit, for $\mu \in \Pp(\Om)$ and $u \in L^1(\mu)$, we have:
\begin{align} \label{eq:feature_operator}
    F_{\mu, u} = \int_\Om \fmap(\om,.) u(\om) \d \mu(\om) =  \Fmap_\mu \cdot u 
\end{align}
where we introduced the \emph{feature operator} $\Fmap_\mu : u \in L^1(\mu) \mapsto \int_\Om u(\om) \fmap(\om, .) \d \mu(\om) \in L^2(\rho)$.
One can thus notice that the problem of minimizing the risk $\Rr^\lambda$ belongs to the class of \emph{separable nonlinear least squares problems} as, by definition, for a fixed inner weights distribution $\mu \in \Pp(\Om)$:
\begin{align*}
     \Rr^\lambda(\mu, u) = \frac{1}{2} \Vert \Fmap_\mu \cdot u - Y \Vert^2_{L^2(\rho)} + \lambda \Vert u \Vert^2_{L^2(\mu)}.
\end{align*}
Thus the problem of minimizing $\Rr^\lambda$ w.r.t. $u$ is a \emph{ridge regression problem} which can be efficiently numerically solved by inverting a linear system.
For $\lambda > 0$, there exists a unique solution $u^\lambda[\mu] \in \argmin_{u \in L^2(\mu)} \Rr^\lambda(\mu, u)$ given by $u^\lambda[\mu] \eqdef (\Fmap_\mu^\top \Fmap_\mu + 2\lambda)^{-1} \Fmap_\mu^\top Y$.
Plugging this in $\Rr^\lambda$ gives rise to a \emph{reduced risk} which we define for any $\mu \in \Pp(\Om)$ by:
\begin{align} \label{eq:L_lambda_min}
    \Ll^\lambda(\mu) \eqdef  \frac{1}{\lambda} \Rr^\lambda(\mu, u^\lambda[\mu]) = \min_{u \in L^2(\mu)} \frac{1}{2\lambda} \Vert \Fmap_\mu \cdot u - Y \Vert^2_{L^2(\rho)} + \Vert u \Vert^2_{L^2(\mu)}.
\end{align}
This definition also extends to the limiting case $\lambda \to 0^+$ by considering:
\begin{align} \label{eq:L0}
    \Ll^0(\mu) \eqdef \min_{\Fmap_\mu \cdot u = Y}  \Vert u \Vert^2_{L^2(\mu)}.
\end{align}
where the infimum is taken to be $+\infty$ whenever the signal $Y$ is not in the range of $\Fmap_\mu$.
In the case where $Y \in \Rg(\Fmap_\mu)$, this minimization problem admits a unique solution $u^\dagger [\mu] = \Fmap_\mu^\dagger \cdot Y$, where $\Fmap_\mu^\dagger$ is the generalized pseudo-inverse of $\Fmap_\mu$ restricted to $L^2(\mu)$, and $\Ll^0(\mu) = \| u^\dagger[\mu] \|^2_{L^2(\mu)}$.

The \emph{VarPro algorithm} consists here in performing \emph{gradient descent} over the reduced risk $\Ll^\lambda$.
For a neural network of finite width $M \geq 1$ with features $\lbrace \om_i \rbrace_{1 \leq i \leq M} \in \Om^M$, the associated reduced risk is $\Hat{\Ll}^\lambda(\lbrace \om_i \rbrace_{1 \leq i \leq M}) \eqdef \Ll^\lambda(\Hat{\mu})$, where $\Hat{\mu}$ is the empirical distribution $\Hat{\mu} = \frac{1}{M} \sum_{i=1}^M \delta_{\om_i}$.
For an initialization $\lbrace \om_i^0 \rbrace_{1 \leq i \leq M} \in \Om^M$ and a step-size $\tau > 0$, the VarPro dynamic reads:
\begin{align*}
    \forall k \geq 0, \forall i \in \lbrace 1, ..., M \rbrace, \quad \om_i^{k+1} = \om_i^k - M \tau \nabla_{\om_i} \hat{\Ll}^\lambda ( \lbrace \om_i^k \rbrace_{1 \leq i \leq M} )\,.
\end{align*}
As before, the \emph{gradient flow} of $\Hat{\Ll}^\lambda_f$ is the continuous counterpart of gradient descent when the step-size $\tau$ tends to $0$.
For an initialization $\lbrace \om_i(0) \rbrace_{1 \leq i \leq M} \in \Om^M$, it is defined for every time $t \geq 0$ as the solution $\lbrace \om_i(t) \rbrace_{1 \leq i \leq M} \in \Om^M$ to the ODE:
\begin{align} \label{eq:gradient_flow}
    \forall i \in \lbrace 1, ..., M \rbrace, \quad \frac{\d}{\d t} \om_i(t) = - M \nabla_{\om_i} \hat{\Ll}^\lambda ( \lbrace \om_i(t) \rbrace_{1 \leq i \leq M} )\,.
\end{align}
Note that the above gradient can be efficiently calculated numerically once optimization on the outer weights $u_i$ has been performed, for example by means of standard automatic differentiation libraries.
Indeed, if $\{ u_i(t) \}_{1 \leq i \leq M} \in \argmin_{u \in \RR^M} \Hat{\Rr}^\lambda(\lbrace (\om_i(t), u_i) \rbrace_{1 \leq i \leq M} ) $, then by the envelope theorem $\nabla_{\om_i} \Hat{\Rr}^\lambda( \lbrace (\om_i(t), u_i(t)) \rbrace_{1 \leq i \leq M} ) = \lambda \nabla_{\om_i} \hat{\Ll}^\lambda ( \lbrace \om_i(t) \rbrace_{1 \leq i \leq M} )$.
For the same reason, the above dynamic can be seen, at least formally, as the limit of the gradient flow dynamic~\cref{eq:gradient_flow_risk} over the (unreduced) risk $\Hat{\Rr}^\lambda$ when the timescale parameter $\eta$ tends to $+\infty$.
Thus, we equivalently refer to~\cref{eq:gradient_flow} as the \emph{VarPro gradient flow} or as the \emph{two-timescale regime of gradient flow}.

\paragraph{Wasserstein gradient flows and ultra-fast diffusions}

Relying on the mathematical framework provided by theory of gradient flows in the space of probability measures~\cite{ambrosio2008gradient,santambrogio2017euclidean}, we show in~\cref{sec:training} that the dynamic of the feature distribution when trained with gradient flow for the minimization of the reduced risk $\Ll^\lambda$ is solution to an advection PDE of the form:
\begin{align*}
    \partial \mu_t - \div (\mu_t \nabla \Ll^\lambda[\mu_t]) = 0
\end{align*}
for some nonlinear velocity field $\nabla \Ll^\lambda_f[\mu_t]$.
We study in~\cref{sec:convergence} the asymptotics of this equation when the training time $t$ tends to $+\infty$ and the regularization strength $\lambda$ tends to $0^+$.
We are more particularly interested in the case where the signal $Y$ itself can be exactly represented by a neural network.
We consider the following assumption:

\begin{assumption}[Teacher student setup] \label{ass:teacher_student}
    Let $\Fmap \star$ be defined by~\cref{eq:Fmap_star}.
    We assume that,
    \begin{enumerate}
        \item there exists a finite measure $\Bar{\nu} \in \Mm(\Om)$ s.t. $Y = \Fmap \star \Bar{\nu}$,
        
        \item the operator $\Fmap \star : \Mm(\Om) \to L^2(\rho)$ is injective.
    \end{enumerate}
    In this case, we refer to $\Bar{\nu} \in \Mm(\Om)$ as the \emph{teacher measure} and to $\Bar{\mu} \eqdef | \Bar{\nu}| / \| \Bar\nu \|_\TV \in \Pp(\Om)$ as the \emph{teacher (feature) distribution}.
\end{assumption}

\noindent
In  such a ``teacher-student'' framework, we are interested in determining to what extent the teacher feature distribution can be learned by the student neural network.
Observe that, under~\cref{ass:teacher_student}, $\Ll^0$ can be simply expressed in terms of the $\chi^2$-divergence between the teacher feature distribution $\Bar{\mu}$ and $\mu$. 
By definition $\chi^2(\Bar{\mu}|\mu) = \int_\Om \left| \frac{\d \Bar{\mu}}{\d \mu} - 1 \right|^2 \d \mu$ and it follows from~\cref{eq:L0_div} that:
\begin{align*}
	\Ll^0(\mu) = \int_\Om \left| \frac{\d \Bar{\nu}}{\d \mu} \right|^2 \d \mu = \| \Bar{\nu} \|^2_\TV \left( \int_\Om \left| \frac{\d \Bar{\mu}}{\d \mu} - 1 \right|^2 \d \mu +1 \right) = \| \Bar{\nu} \|^2_\TV ( \chi^2(\Bar{\mu} | \mu) +1 ) \,.
\end{align*}
%More generally, for $f(t) = |t|^r/(r-1)$ we have $\Ll^0_r(\mu) = \frac{\| \Bar{\nu} \|^2_\TV}{r-1} \int_\Om \left| \frac{\d \Bar{\mu}}{\d \mu} \right|^r \d \mu$ and 
The Wasserstein gradient flow of $\Ll^0$ then corresponds to a nonlinear diffusion equation of the form:
\begin{align} \label{eq:nonlinear_diffusion}
    \partial_t \mu = \div \left( \Bar{\mu} \nabla \left( \frac{\mu}{\Bar{\mu}} \right)^m \right)
\end{align}
with $m < 0$ and $\Bar{\mu} \in \Pp(\Om)$, referred to as \emph{ultra-fast diffusion} equation~\cite{iacobelli2019weighted}.
Note that this class of nonlinear diffusion equations stands out from the class of \emph{linear diffusion} and \emph{porous medium} equations (corresponding to the case $m \geq 1$~\cite{vazquez2006smoothing,vazquez2007porous}) by the fact that the exponent $m$ is negative and the diffusivity $\mu^{m-1}$ is singular at $0$.
In~\cite{iacobelli2016asymptotic,caglioti2016quantization,iacobelli2019weighted}, the study of solutions to~\cref{eq:nonlinear_diffusion} is motivated by the convergence analysis of algorithms for the quantization of measures.
In particular, \textcite{iacobelli2019weighted} show the well-posedness of~\cref{eq:nonlinear_diffusion} on the $d$-dimensional torus or on bounded convex domains with Neumann boundary conditions and prove convergence of solutions towards the stationary state $\Bar{\mu}$ in $L^2$.
We prove in~\cref{thm:gradient_flow_approximation} that Wasserstein gradient flows of our reduced risk $\Ll^\lambda$ converge towards solutions of the ultra-fast diffusion equation when the regularization strength $\lambda$ vanishes.

\begin{rem}
    Some remarks about~\cref{ass:teacher_student}:
    \begin{itemize}
        \item At fixed $\lambda > 0$, the teacher-student assumption that $Y =\Fmap \star \Bar{\nu}$ is not restrictive since one can always replace $Y$ by its orthogonal projection on the set $\lbrace \Fmap \star \nu, \, \nu \in \Mm(\Om) \rbrace$, thereby only modifying $\Ll^\lambda$ by subtracting a constant term. However, this assumption becomes crucial in the limit $\lambda \to 0^+$ to ensure the feasibility of the optimization problem in~\cref{eq:L0_f}.
        
        \item The injectivity assumption on $\Fmap \star$ ensures uniqueness of the reference measure $\Bar{\nu}$. In the limit where $\lambda \to 0^+$, this allows rewriting $\Ll^0$ only in terms of a divergence between $\Bar{\nu}$ and $\mu$ (\cref{eq:L0_div}).
        In the case $\lambda > 0$, $\Ll^\lambda$ is an infimal convolution between this divergence and a \emph{kernel discrepancy} (\cref{eq:L_mmd}) and the injectivity assumption ensures this discrepancy is a distance on the space of measures (\cref{lem:rkhs_universality}).
        It will be useful in~\cref{sec:convergence} to prove convergence of Wasserstein gradient flows of $\Ll^\lambda$ to solutions of the ultra-fast diffusion equation.
        In the case of a two-layer perceptron, the feature map is of the form $\fmap((w,b), x) = \activation (w^\top x + b)$ and the injectivity assumption is satisfied as soon as $\activation$ is not a polynomial and the data distribution has full support on $\RR^d$ (\cite[Thm. III.4]{sun2019random}).
        
    \end{itemize}
\end{rem}

%% file: numerics_1d_alt.tex
\subsection{Single-hidden-layer neural networks with $1$-dimensional feature space} \label{subsec:numerics_1d}

We tested the VarPro algorithm for the training of a simple SHL with features on the $1$-dimensional sphere $\SS^1$.
The feature space is here $\Om = \SS^1$, the data dimension is $d=2$, and the feature map is given by $\fmap : (\om, x) \in \SS^1 \times \RR^2 \mapsto \relu(\om^\top x)$ where $\relu : t \in \RR \mapsto \max(0,t)$ is the \emph{Rectified Linear Unit} activation.
Recalling~\cref{eq:SHL}, we thus consider a SHL of width $M$  defined for inner weights $\{ \om_i \}_{i=1}^M \in (\SS^1)^M$ and outer weights $\{u_i \}_{i=1} \in  \RR^M$ by:
\begin{align} \label{eq:SHL_relu1d}
    F_{\{(\om_i, u_i)\}} : x \in \RR^2 \mapsto \frac{1}{M} \sum_{i=1}^M u_i \relu(\om_i^\top x) \, .
\end{align}
We consider a target signal $Y$ that is given by a teacher network of width $\Bar{M}$:
\begin{align*}
    \forall x \in \RR^2, \quad Y(x) = \frac{1}{\Bar{M}} \sum_{i=1}^{\Bar{M}} \relu( \Bar{\om}_i^\top x) \, .
\end{align*}
The teacher feature distribution is hence $\Bar{\mu}_\gamma = \frac{1}{\Bar{M}} \sum_{i=1}^{\Bar{M}} \delta_{\Bar{\om}_i}$ with i.i.d. features $\Bar{\om}_i \sim \mu_\gamma$ where, for $\gamma > 0$, we consider $\mu_\gamma \eqdef \left( \frac{2}{3} \delta_{\om_1^*} + \frac{1}{3} \delta_{\om_2^*} \right) \star \pi_\gamma$.
The target feature modes are here fixed to $\om_1^* = 0$ and $\om_2^* = 0.4 \pi$ and $\pi_\gamma \in \Pp(\SS^1)$ is the distribution with density:
\begin{align} \label{eq:mu_gamma}
    \pi_\gamma(\om) \propto \frac{1}{1+ \gamma \sin^2(\om / 2)}, \quad \forall \om \in \SS^1 \, , 
\end{align}
where by abuse of notation we identify $\om \in \SS^1$ with the corresponding angle in $\nicefrac{\RR}{2\pi\ZZ}$.
In particular, the parameter $\gamma \geq 0$ controls the shape of the distribution $\mu_\gamma$ and the concentration around its modes:
when $\gamma = 0$, $\mu_\gamma$ is the uniform distribution and, when $\gamma \to \infty$, we have $\mu_\gamma \to \mu_\infty \eqdef \frac{2}{3} \delta_{\om_1^*} + \frac{1}{3} \delta_{\om_2^*}$.
Plots of the density $\mu_\gamma$ and of the corresponding teacher signal are shown in~\cref{fig:teacher_signal}.
Finally, we consider the input data $x$ to be distributed according to an empirical distribution $\Hat{\rho} = \frac{1}{N} \sum_{i=1}^N \delta_{x_i}$ with i.i.d. standard Gaussian samples $x_i \sim \Nn(0, \Id)$.

Recall that, for a regularization function $f: \RR \to \RR$ and a regularization strength $\lambda > 0$, the reduced risk defined by~\cref{eq:L_lambda_f} associated to the features $\{ \om_i \}_{i=1}^M \in (\SS^1)^M$ reads:
\begin{align} \label{eq:reduced_risk_SHL}
    \Hat{\Ll}^\lambda_f (\{ \om_i \}_{i=1}^M)
    = \min_{u \in \RR^M} \frac{1}{2 \lambda N} \sum_{j=1}^N \left| F_{\{(\om_i, u_i)\}}(x_j) - Y(x_j) \right|^2 + \frac{1}{M} \sum_{i=1}^M f(u_i).
\end{align}
In this setting, the \emph{VarPro algorithm} is the time discretization of the particle evolution~\cref{eq:gradient_flow} and consists in performing gradient descent over the reduced risk $\Hat{\Ll}^\lambda_f$:
\begin{align} \label{eq:gradient_descent_SHL}
    \forall i \in \lbrace 1, ..., M \rbrace, \forall k \geq 0, \quad \om_i^{k+1} = \om_i^k - M \tau \nabla_{\om_i} \hat{\Ll}^\lambda_f ( \lbrace \om_i^k \rbrace_{1 \leq i \leq M} ) \,.
\end{align}
where $\tau > 0$ is some stepsize parameter and $\{ \om_i^0 \}_{i=1}^M \in (\SS^1)^M$ is some random  initialization.
We consider here an uniform initialization with i.i.d.\@ $\om_i^0 \sim \Uu(\SS^1)$.

\begin{figure}
    \centering
    \centerline{
    \includegraphics[scale=0.75,trim={0 0 0 0},clip]{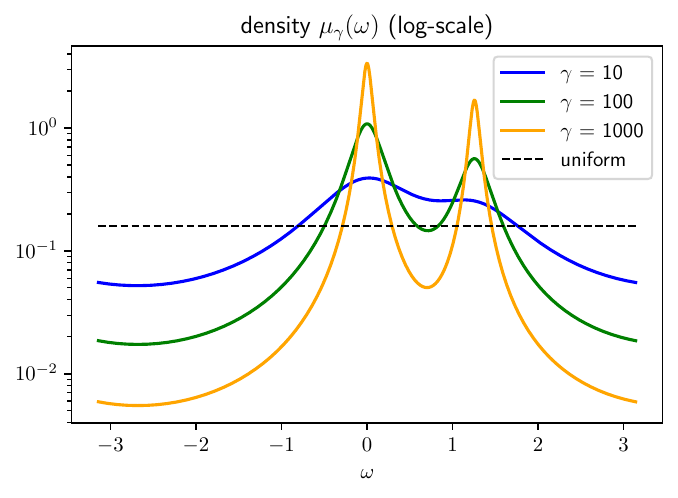}
    \includegraphics[scale=0.38,trim={0 0 10 0},clip]{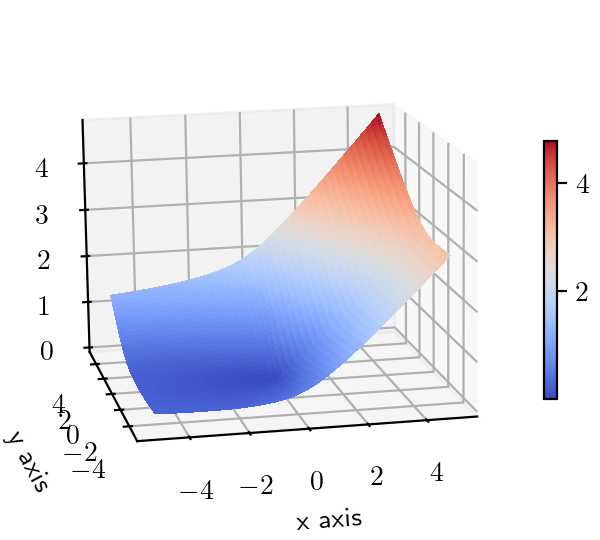}
    }
    \caption{Left: density of the teacher distributions $\mu_\gamma$ for $\gamma \in \{ 10, 100, 1000 \}$. Right: corresponding teacher signal for $\gamma = 100$.}
    \label{fig:teacher_signal}
\end{figure}

\paragraph{Experimental setting}

We test the performance of the VarPro algorithm (\cref{eq:gradient_descent_SHL}) for the training of SHLs (\cref{eq:SHL_relu1d}) of varying width $M \in \{ 32, 128, 512, 1024 \}$.
We use either the ``biased'' quadratic regularization $f_b : t \mapsto \frac{1}{2} t^2$, for which the minimizer of the reduced risk differs from $\Bar{\mu}_\gamma$, or the ``unbiased'' quadratic regularization $f_u : t \mapsto \frac{1}{2}|t-1|^2$, for which the minimizer of the reduced risk is the teacher distribution $\Bar{\mu}_\gamma$ (c.f.~\cref{sec:minimizer}), and we consider varying regularization strength $\lambda \in \{10^{-1}, 10^{-2}, 10^{-3}, 10^{-4} \}$.
We also consider different teacher distributions $\Bar{\mu}_\gamma$ by changing the parameter $\gamma \in \{10, 100, 1000 \}$.
In order to stick with our theoretical results, we consider a number of data samples $N = 4096 \gg M$, such that the injectivity assumption in~\cref{ass:teacher_student} is satisfied, and we consider the teacher has a width $\Bar{M} = 4096 \gg M$, such that the approximation $\Bar{\mu}_\gamma \simeq \mu_\gamma$ holds.
Finally, to closely model the gradient flow equation~\cref{eq:empirical_gradient_flow} we consider a stepsize $\tau = 2^{-10}$.

\paragraph{Qualitative comparison with ultra-fast diffusion on $\SS^1$}

Conveniently, the choice of the $1$-dimensional domain $\SS^1$ enables the use of standard numerical schemes to solve the weighted ultra-fast diffusion equation~\cref{eq:grad_flow_L0}.
This setting thus allows for comparison of the solutions to ultra-fast diffusion computed with high accuracy on a fine grid --- we use here the ``LSODA'' integration method~\cite{hindmarsh1983odepack} --- and the training dynamics computed with our VarPro method with particles~\cref{eq:gradient_descent_SHL}.
The two dynamics can be compared in~\cref{fig:density_evolution}.
Qualitatively, one can observe a close resemblance between the two dynamics, especially around the modes of the target distribution $\mu_\gamma$ where the densities progressively concentrates.
While the learned feature distribution seems to concentrate less than the exact solution, this is likely due to the convolution with a gaussian kernel which is used to plot the density.
However, the dynamics seems to differ more on the sides of the plots.
These are indeed regions where the density $\mu_t$ becomes very low and thus where approximation of the velocity field $\nabla \left( \frac{\mu_\gamma}{\mu_t} \right)^2$ likely suffers from numerical instabilities.

\begin{figure}
    \centering
    \includegraphics[scale=0.9,trim={20 15 0 0}]{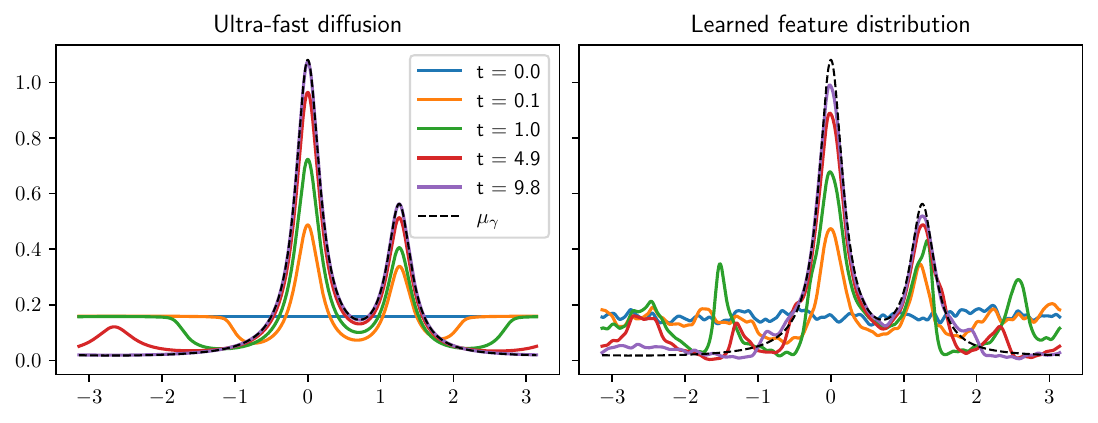}
    \caption{Left: Solution $\mu_t$ to the ultra-fast diffusion~\cref{eq:grad_flow_L0} equation with exponent $r=2$ and weights $\mu_\gamma$, $\gamma=100$. Right: Evolution of the feature distribution learned by gradient descent on a SHL of width $M = 1024$ for the minimization the reduced risk $\Hat{\Ll}^\lambda_f$ with regularization function $f_b : t \mapsto \frac{1}{2} t^2$ and $\lambda = 10^{-4}$ (c.f.~\cref{eq:reduced_risk_SHL,eq:gradient_descent_SHL}). The density is obtained by convolving the empirical feature distribution $\Hat{\mu}$ with a gaussian kernel of variance $\sigma^2 = (0.03)^2$ and the plots are averages over $6$ independent runs.}
    \label{fig:density_evolution}
\end{figure}

\paragraph{Neural networks of varying width}

We investigate the behavior of the gradient descent dynamic for the minimization of the reduced risk (\cref{eq:gradient_descent_SHL}) when varying the width $M$ of the neural network.
For this purpose we consider the teacher distribution $\Bar{\mu}_\gamma \simeq \mu_\gamma$ with $\gamma = 100$, fix the regularization strength to $\lambda = 10^{-3}$ and consider SHLs of varying width $M \in \{ 32, 128, 512, 1024 \}$ with regularization either $f_b$ or $f_u$.

In this setting, \cref{fig:relu1d_width_risk} reports evolution of the reduced risk $\Hat{\Ll}^\lambda_f$ along iterations of gradient descent.
In the case of the biased regularization $f_b$, the reduced risk monotonically decreases to the same (strictly positive) value for every width.
This is normal since one should expect the feature distribution to converge to a minimizer $\Bar{\mu}_\gamma^\lambda \neq \Bar{\mu}_\gamma$ for which the reduced risk is strictly positive.
On the contrary, in the case of the unbiased regularization $f_u$, the reduced risk monotonically decreases to different values depending on the width $M$.
Indeed, in this case the gradient descent is expected to converge to the true teacher distribution $\Bar{\mu}_\gamma \simeq \mu_\gamma$ and these different values corresponds to different levels of discretization of $\mu_\gamma$.
Also, in this case, the convergence speed seems to increase with the width.

\begin{figure}
    \centering
    \includegraphics[scale=0.95,trim={20 10 0 0}]{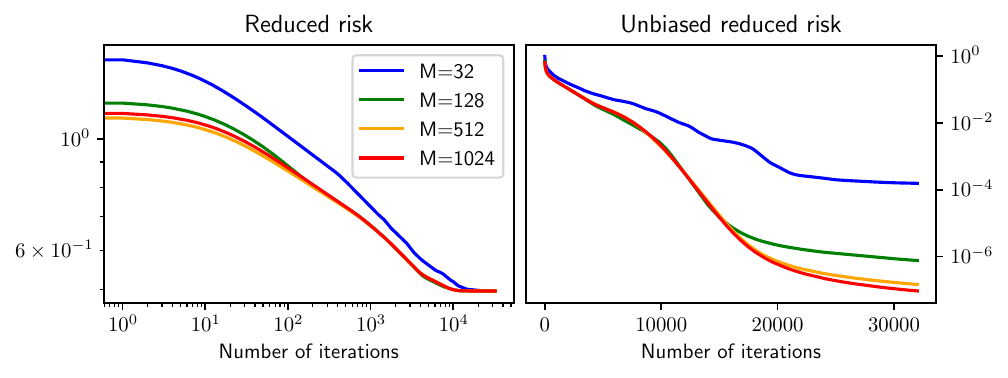}
    \caption{Evolution of the reduced risk $\Hat{\Ll}^\lambda_f$ (\cref{eq:gradient_descent_SHL}) along iterations of gradient descent for a SHL of width $M \in \{32, 128, 512, 1024 \}$. The regularization strength is $\lambda = 10^{-3}$ and the regularization function is either $f_b : t \mapsto \frac{1}{2} t^2$ (left) or $f_u : t \mapsto \frac{1}{2} |t-1|^2$ (right). Plots are averages over $6$ independent runs.}
    \label{fig:relu1d_width_risk}
\end{figure}

In~\cref{fig:relu1d_width_distance}, we report the evolution of a MMD distance between the learned feature distribution and two references which are the teacher distribution $\Bar{\mu}_\gamma \simeq \mu_\gamma$ and the exact ultra-fast diffusion dynamic.
We used the MMD distance~\cref{eq:MMD} associated to the energy-distance kernel $\kappa(\om, \om') = -\|\om - \om'\|$.
In coherence with what was observed before, in the case of the unbiased regularization $f_u$, the distance to the teacher distribution decreases monotonically to some value which is lower when the width increases.
Illustrating our~\cref{thm:algebraic_convergence}, this shows gradient descent converges to a feature distribution discretizing the teacher distribution.
On the contrary, when considering the biased regularization $f_b$, the positive regularization strength introduces a bias.
In turn, plots of the distance to the diffusion dynamic show this distance decreases with the width, which is normal since a higher number of features corresponds to a better discretization.
These plots also show that gradient descent stays close from the diffusion limit, as predicted by~\cref{thm:diffusion_convergence}.

\begin{figure}
    \centering
    \includegraphics[scale=1,trim={10 10 0 0}]{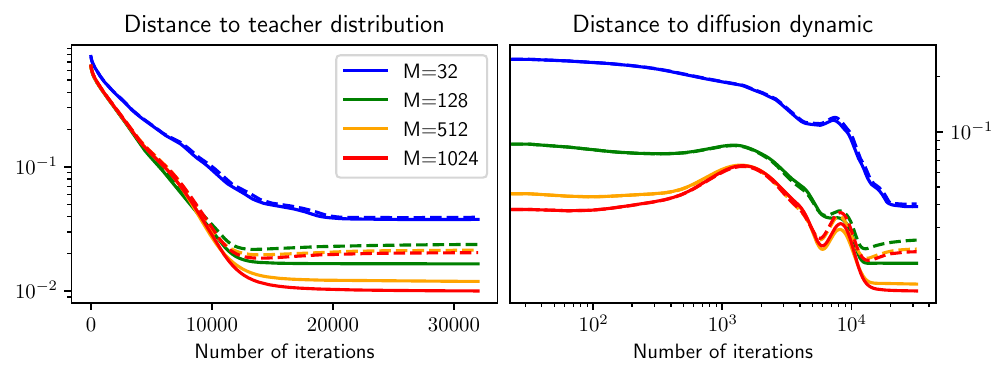}
    \caption{
    Evolution of the MMD distance to the teacher distribution and to the diffusion dynamic along iterations of gradient descent over the reduced risk $\Hat{\Ll}^\lambda_f$ (\cref{eq:gradient_descent_SHL}) for a SHL of width $M \in \{32, 128, 512, 1024 \}$.
    Left: distance to the teacher distribution $\Bar{\mu}_\gamma \simeq \mu_\gamma$ ($\gamma = 100$).
    Right: distance to the diffusion dynamic.
    The regularization strength is $\lambda = 10^{-3}$ and the regularization function is either $f_b : t \mapsto \frac{1}{2} t^2$ (dashed) or $f_u : t \mapsto \frac{1}{2} |t-1|^2$ (plain).
    Plots are averages over $6$ independent runs.}
    \label{fig:relu1d_width_distance}
\end{figure}

\paragraph{Role of the regularization strength $\lambda$}

We now investigate the role played in the gradient descent dynamic by the regularization strength $\lambda > 0$.
For this purpose, we consider a neural network of fixed width $M=1024$ and train it with gradient descent for the minimization of the reduced risk $\Hat{\Ll}^\lambda_f$ for varying values $\lambda \in \{10^{-1}, 10^{-2}, 10^{-3}, 10^{-4} \}$ of the regularization strength.

Evolution of the MMD distance between the learned feature distribution and respectively the teacher feature distribution and the diffusion dynamic are shown in~\cref{fig:relu1d_lmbda_distance}.
On the plots of distance to the teacher distribution, one can first observe that the bias introduced in the case of the regularization $f_b$ decreases with the regularization strength $\lambda$.
This illustrates well our~\cref{prop:convergence_minimizers}, showing convergence of minimizers of the reduced risk towards the true teacher distribution when the regularization strength vanishes.
In the case of the unbiased regularization $f_u$, one can observe a difference of behavior between low regularization regimes $\lambda \in \{10^{-2}, 10^{-3}, 10^{-4} \}$ and large regularization $\lambda = 10^{-1}$.
While in the former case convergence seems to operate at a linear rate, which is the convergence rate of the diffusion limit (\cref{thm:diffusion_convergence}), in the latter the convergence rate is significantly slower which could indicate an algebraic rate as predicted by~\cref{thm:algebraic_convergence}.
Indeed, $\lambda = 10^{-1}$ is the order of magnitude of the most significant eigenvalues of the tangent kernel $K_\mu$ (numerically, the spectrum of $K_\mu$ is, in descending order, $\Sp(K_\mu) \simeq \left( 0.2, 0.1, 0.1, 0.02, ... \right)$).
Recalling that the risk can be expressed in terms of $(K_\mu + \lambda)^{-1}$~(\cref{eq:kernel_learning}), an explanation is thus that, the unregularized reduced risk is well approximated only when $\lambda \ll K_\mu$.
In contrast, in the high regularization regime ($\lambda \gtrsim K_\mu$), the reduced risk receive more influence from the MMD distance term than from the $f$-divergence term in~\cref{eq:L_mmd} and gradient flows of MMD distances are known to be associated with slower convergence rates.

Finally, plots of the distance between the gradient flow and ultra-fast diffusion dynamics show this distance is lower and stays also lower for longer time when the regularization strength decreases.
This supports the ``local uniform in time convergence'' behavior predicted by~\cref{thm:gradient_flow_approximation}.
Note however that this result says nothing about the long time behavior of the dynamic, which is why the number of iterations is displayed in log-scale.

\begin{figure}
    \centering
    \includegraphics[scale=0.95,trim={20 10 0 0}]{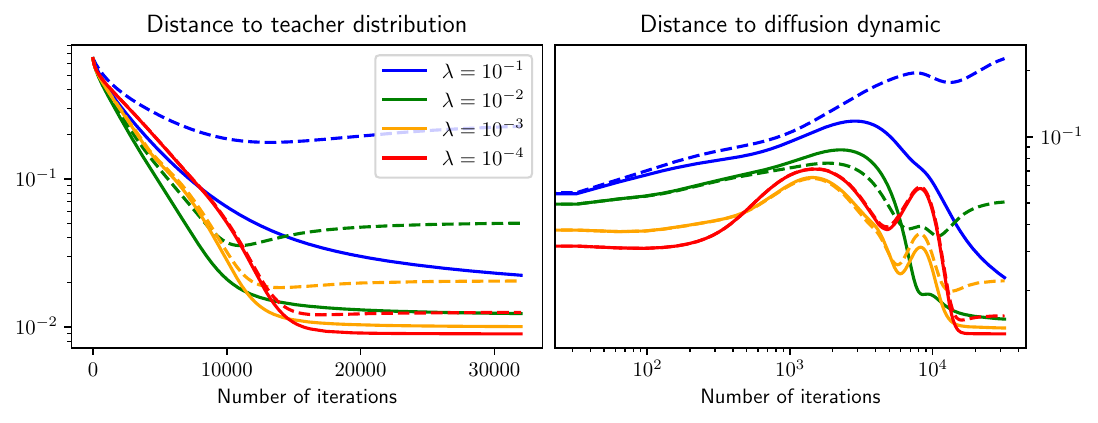}
    \caption{
    Evolution of the MMD distance to the teacher distribution and to the diffusion dynamic along iterations of gradient descent over the reduced risk $\Hat{\Ll}^\lambda_f$ (\cref{eq:gradient_descent_SHL}) for a SHL of width $M=1024$ with regularization $\lambda \in \{10^{-1}, 10^{-2}, 10^{-3}, 10^{-4} \}$.
    Left: distance to the teacher distribution $\Bar{\mu}_\gamma \simeq \mu_\gamma$ ($\gamma=100$).
    Right: distance to the diffusion dynamic.
    The regularization function is either $f_b : t \mapsto \frac{1}{2} t^2$ (dashed) or $f_u : t \mapsto \frac{1}{2} |t-1|^2$ (plain).
    Plots are averages over $6$ independent runs.}
    \label{fig:relu1d_lmbda_distance}
\end{figure}

\paragraph{Role of the shape of the teacher distribution}

We investigate the role played by the shape of the distribution $\mu_\gamma$, controlled by the parameter $\gamma$.
We consider teacher distributions $\Bar{\mu}_\gamma \simeq \mu_\gamma$ for $\gamma \in \{10, 100, 1000 \}$ and  train a neural network of fixed width $M=1024$ with gradient descent over the reduced risk (\cref{eq:gradient_descent_SHL}) with the unbiased regularization $f_u$ and $\lambda=10^{-4}$.
Plot of the log-densities $\mu_\gamma$ are shown in~\cref{fig:teacher_signal}.
In particular the distribution $\mu_\gamma$ approximates the atomic distribution $\mu_\infty = \frac{2}{3} \delta_{\om_1^*} + \frac{1}{3} \delta_{\om_2^*}$ in the limit $\gamma \to \infty$.

Plots of the evolution of the reduced risk, of the distance to the teacher distribution and of the distance to the ultra-fast diffusion dynamic are shown in~\cref{fig:relu1d_gamma}.
One can clearly observe that the convergence speed of gradient descent is affected by the parameter $\gamma$.
In particular, looking at the distance to the teacher distribution, every curve exhibits a linear convergence rate but this convergence rate deteriorates when $\gamma$ increases.
This supports the conclusions of~\cref{thm:diffusion_convergence} in which the convergence rate of ultra-fast diffusion towards the target distribution is exponentially bad in the log-density ratio $\log(\mu_\gamma / \mu_0)$ (see also~\cref{rmk:convergence_rate}).
Finally, on can observe in the last plot that gradient descent deviates more quickly from the diffusion dynamic when $\gamma$ increases.
When $\gamma$ is large, there are indeed regions where the density $\mu_t$ will become very low, hence leading to numerical instabilities when estimating the velocity field $\nabla \left( \frac{\mu_\gamma}{\mu} \right)^2$.

\begin{figure}
    \centering
    %\hspace{-10em}
    \centerline{
    \includegraphics[scale=0.95,trim={10 10 0 0}]{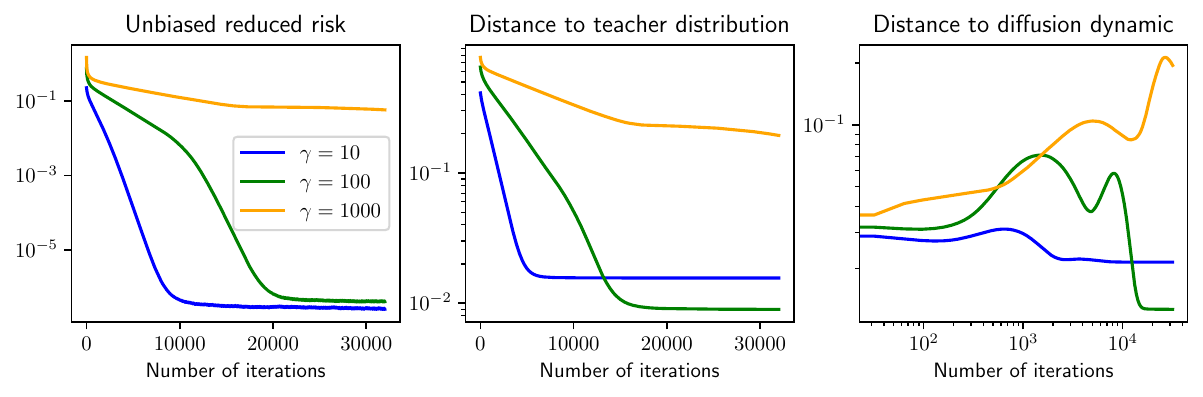}
    }
    \caption{Gradient descent over the reduced risk (\cref{eq:gradient_descent_SHL}) for a SHL of width $M=1024$ with unbiased regularization $f_u : t \mapsto \frac{1}{2} |t-1|^2$, $\lambda = 10^{-4}$ and teacher distribution $\Bar{\mu}_\gamma \simeq \mu_\gamma$ for $\gamma \in \{10, 100, 1000 \}$.
    Left: Evolution of the reduced risk.
    Middle: Evolution of the MMD distance to the teacher distribution $\Bar{\mu}_\gamma \simeq \mu_\gamma$.
    Right: Distance to the ultra-fast diffusion dynamic.
    Plots are averages over $6$ independent runs.}
    \label{fig:relu1d_gamma}
\end{figure}

\paragraph{Comparison with two-timescale gradient descent}

Since performing exact projection of the outer layer at every gradient step might have a prohibitive algorithmic cost, it is interesting to compare the VarPro algorithm with the two-timescale gradient descent which consist in affecting a different learning rate to the inner and outer weights of the neural network.
For a regularization function $f: \RR \to \RR$ and a regularization strength $\lambda > 0$ we recall that the risk defined by~\cref{eq:R_lambda_f} associated to the parameters $\{ (\om_i, u_i) \}_{i=1}^M \in (\SS^1 \times \RR)^M$ reads:
\begin{align} \label{eq:R_lambda_SHL}
    \frac{1}{\lambda} \Hat{\Rr}^\lambda_f(\{ (\om_i, u_i) \}_{i=1}^M) =  \frac{1}{2 \lambda N} \sum_{j=1}^N \left| F_{\{(\om_i, u_i)\}}(x_j) - Y(x_j) \right|^2 + \frac{1}{M} \sum_{i=1}^M f(u_i).
\end{align}
Then, for a timescale parameter $\eta > 0$, we implement the two-timescale gradient descent algorithm defined by :
\begin{align} \label{eq:2ts_gradient_descent_SHL}
    \forall i \in \lbrace 1, ..., M \rbrace, \forall k \geq 0, \quad
    \left\{
    \begin{array}{rcl}
        \om_i^{k+1} & = & \om_i^k - \frac{M \tau}{\lambda} \nabla_{\om_i} \hat{\Rr}^\lambda_f ( \lbrace (\om_i^k, u_i^k) \rbrace_{1 \leq i \leq M} ) , \\[5pt]
        u_i^{k+1} & = & u_i^k - \frac{\eta}{\lambda} \nabla_{u_i} \hat{\Rr}^\lambda_f ( \lbrace (\om_i^k, u_i^k) \rbrace_{1 \leq i \leq M} ) .
    \end{array}
    \right.
\end{align}
As for the VarPro algorithm (\cref{eq:gradient_descent_SHL}), we take the stepsize parameter $\tau = 2^{-10}$ and $\{ \om_i^0 \}_{i=1}^M \in (\SS^1)^M$ is some random  initialization with i.i.d. $\om_i^0 \sim \Uu(\SS^1)$.
For a fair comparison with VarPro, we first perform one projection step before training such that the outer weights initialization verifies:
\begin{align*}
    u^0 \in \argmin_{u \in \RR^M} \hat{\Rr}^\lambda_f ( \lbrace (\om_i^0, u_i) \rbrace_{1 \leq i \leq M} ).
\end{align*}

Concerning the timescale parameter $\eta$, we find it empirically efficient to set it to $\eta = \lambda M$.
Lower values of $\eta$ leads to slower training and higher values to numerical instabilities.
An explanation for this is that, in the case of a quadratic regularization, by~\cref{eq:R_lambda_SHL} the risk as a function of the outer weights $u \in \RR^M$ reads:
\begin{align*}
    \frac{1}{\lambda} \Hat{\Rr}^\lambda_f(  \lbrace (\om_i, u_i) \rbrace_{1 \leq i \leq M} ) = \frac{1}{2 \lambda N} \sum_{j=1}^N \left| \frac{1}{M} ( \Hat{\Fmap} \cdot u )_j - Y(x_j) \right|^2 + \frac{1}{2M} \sum_{i=1}^M u_i^2,
\end{align*}
where $\Hat{\Fmap} \in \RR^{N \times M}$ is some feature matrix depending on the features $\lbrace \om_i \rbrace_{1 \leq i \leq M}$.
Numerically, one observes $\frac{1}{N M}\lambda_{\max} (\Hat{\Fmap}^\top \Hat{\Fmap}) \simeq 1$, such that $\eta ^{-1} = \frac{1}{\lambda M} \simeq \lambda_{\max}(\lambda^{-1} \nabla^2_{u,u} \Hat{\Rr}^\lambda_f)$ indeed corresponds to the smoothness constant of the ridge regression problem w.r.t. $u$.

\begin{rem}
    Note that, as explain above, for numerical stability, one can not consider an arbitrarily large time-scale parameter $\eta$ and we fix here $\eta = \lambda M$.
    In this setting, the ratio between the lerning rates of inner and outer weights is given by $\frac{\eta}{M \tau} = \frac{\lambda}{\tau}$.
    Therefore, we can only expect to be in the two-timescale regime, i.e. when the two-timescale gradient descent is a good approximation of VarPro, if the stepsize $\tau$ is chosen s.t. $\tau \ll \lambda$.
    
    We stress that, for low-regularization regimes, this can be numerically prohibitive and VarPro, i.e. exact optimization of the outer weights at each step, can provide an efficient alternative to gradient descent.
    Interestingly, we in fact observe in our case that, as soon as $\tau \gg \lambda$ and thus $\eta \gg M\tau$ (which for examples happens here for $\lambda = 10^{-4}$), the VarPro algorithm (\cref{eq:gradient_descent_SHL}) efficiently learns the teacher feature distribution (see e.g. \cref{fig:relu1d_lmbda_distance}), whereas two-timescale gradient descent (\cref{eq:2ts_gradient_descent_SHL}) does not converge.
\end{rem}

In this setting we train SHLs of varying width using either the VarPro algorithm (\cref{eq:gradient_descent_SHL}) or the two-timescale gradient descent algorithm (\cref{eq:2ts_gradient_descent_SHL}) and report results in~\cref{fig:relu1d_2ts}.
As predicted, one can observe the two dynamics are very close in the case case of a sufficiently high regularization, here $\lambda \geq 10^{-2}$, for which we have  $\eta \gg M \tau$.
This supports the fact that, in this regime, the VarPro dynamic can be obtained as the two-timescale limit of gradient descent.
On the other hand, the two dynamics significantly differ in the low regularization regime $\lambda = 10^{-3}$ for which we have $\eta = \lambda M \simeq M \tau$.
In this case, independently of the width $M$, the VarPro algorithm converges at a linear rate, while two-timescale gradient descent is slower and even seems to introduce a bias in the learned feature distribution.
An explanation is that, in this regime, the two-timescale gradient descent quickly deviates from the ultra-fast diffusion dynamic, which one can observe in the last column of~\cref{fig:relu1d_2ts}.
Overall, the most favorable setting seems to be when $\lambda = 10^{-2}$.
Indeed, in this case $\eta = \lambda M \gg \tau M$ s.t. two-timescale gradient descent efficiently emulates the VarPro dynamic, while $\lambda \ll \| K_\mu \|_{op} \simeq 0.5$, the spectral norm of tangent kernel, s.t. both dynamics benefit from the linear convergence rate of ultra-fast diffusion (see also~\cref{fig:relu1d_lmbda_distance}).

\begin{figure}
    \centering
    %\hspace{-10em}
    \centerline{
    \includegraphics[scale=0.9]{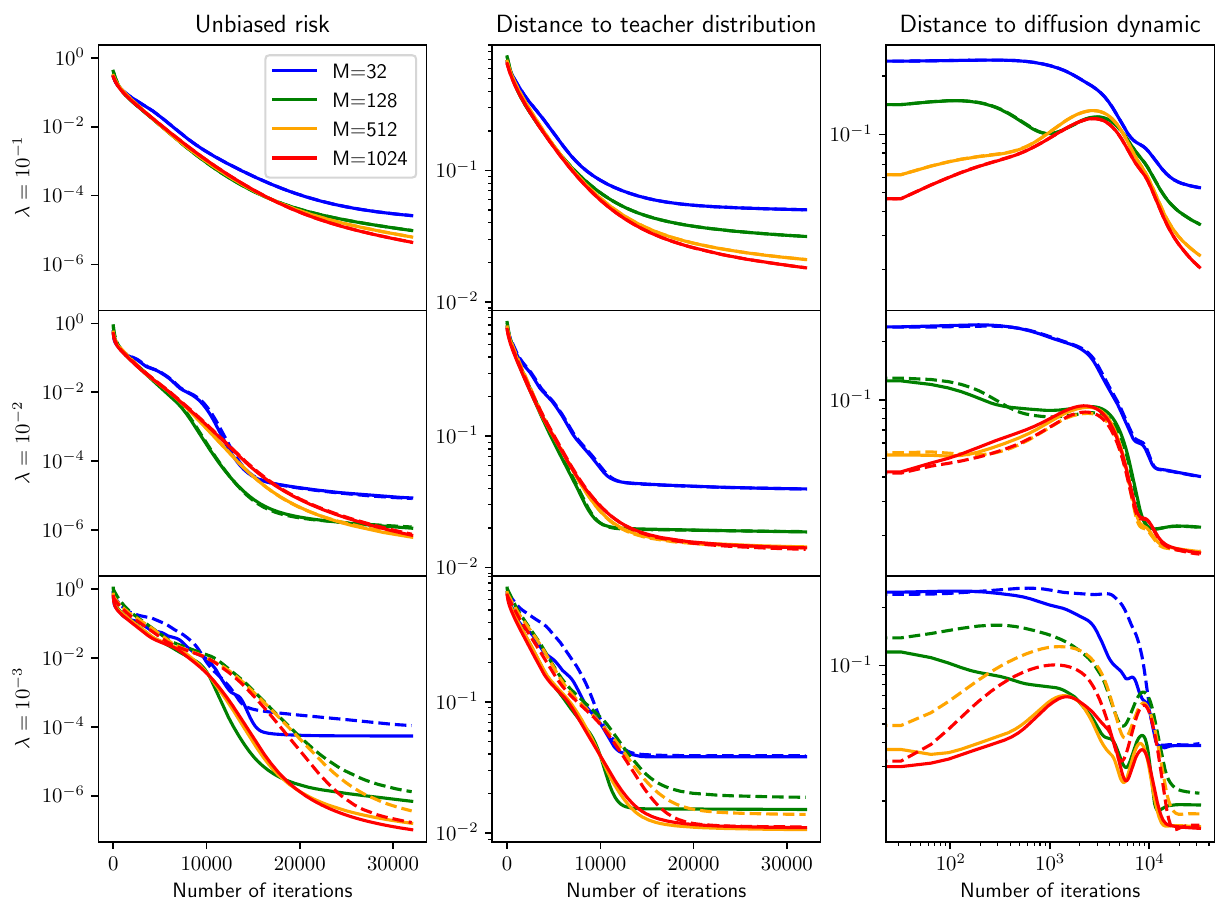}
    }
    \caption{VarPro (\cref{eq:gradient_descent_SHL}, plain lines) and two-timescale gradient descent (\cref{eq:2ts_gradient_descent_SHL}, dashed lines) over the risk for SHLs of varying width $M \in \{ 32, 128, 512, 1024 \}$ with unbiased regularization function $f_u : t \mapsto \frac{1}{2} |t-1|^2$ and regularization strength $\lambda = 10^{-1}$ (top), $\lambda = 10^{-2}$ (middle) or $\lambda = 10^{-3}$ (bottom).
    The teacher distribution is $\Bar{\mu}_\gamma \simeq \mu_\gamma$ with $\gamma = 100$.
    Left: Evolution of the risk.
    Middle: Evolution of the MMD distance to the teacher distribution.
    Right: Distance to the ultra-fast diffusion dynamic.
    Plots are averages over $6$ independent runs.}
    \label{fig:relu1d_2ts}
\end{figure}

%% file: numerics_2d.tex
\section{Radial basis function neural network on the $2$-dimensional torus} \label{sec:numerics_2d}

We performed numerical experiments to test the performance of the VarPro algorithm for the training of \emph{Radial Basis Function (RBF)} neural networks.
%Variable projections algorithms for this kind of architecture have also been studied in~\cite{pereyra2006variable,karamichailidou2024radial}.
Notably, due to the particular structure of the architecture, the learning problem corresponds to performing a deconvolution, which has important applications in signal processing~\cite{de2012exact,duval2015exact}.

The feature space is here the $2$-dimensional torus $\Om =  \nicefrac{\RR^2}{4\ZZ^2} \subset \RR^2$ and the data dimension is $d=2$.
The RBF neural network architecture performs the convolution with a kernel $k : \RR^2 \to \RR$ and corresponds to considering the feature map $\fmap : (\om, x) \mapsto k(\om-x)$.
We will use here the Laplace kernel
    $k : x \in \RR^2 \mapsto 8 \exp(-\frac{1}{2} \| [x] \|)$, 
where $[x]$ represents the projection of $x$ in $\Om =  \nicefrac{\RR^2}{4\ZZ^2}$.
For features $\{ \om_i \}_{i=1}^M \in (\Om)^M$ and outer weights $\{u_i \}_{i=1} \in  \RR^M$ the RBF neural network model reads:
\begin{align} \label{eq:convolutional_model}
    F_{\{(\om_i, u_i)\}} : x \in \RR^2 \mapsto \frac{1}{M} \sum_{i=1}^M u_i k(\om_i-x) = (k \star \Hat{\nu})(x),
\end{align}
where $\Hat{\nu} = \frac{1}{M} \sum_{i=1}^M u_i \delta_{\om_i} \in \Mm(\Om)$ and $\star$ is the convolution operator.
We consider a teacher feature distribution $\Bar{\mu}_\gamma = \frac{1}{\Bar{M}} \sum_{i=1}^{\Bar{M}} \delta_{\Bar{\om}_i}$ for features $\{ \Bar{\om}_i \}_{1 \leq i \leq \Bar{M}} \in \Om^{\Bar{M}}$ and the target signal $Y$ is thus:
\begin{align*}
    \forall x \in \RR^2, \quad Y(x) = \frac{1}{\Bar{M}} \sum_{i=1}^{\Bar{M}} k(\om_i-x) = (k \star \Bar{\mu}_\gamma).
\end{align*}
The teacher features are i.i.d. with $\Bar{\om}_i \sim \mu_\gamma = ( \frac{1}{2} \delta_{\om^*_1} + \frac{1}{2} \delta_{\om^*_2} ) \star \pi_\gamma$, where $\om_1^* = (-1,0)$, $\om^*_2 = (1, 1)$ are two target modes and
$\pi_\gamma$ is the product measure with density:
\begin{align*}
    \forall (z_1, z_2) \in \nicefrac{\RR^2}{4\ZZ^2}, \quad \pi_\gamma(z_1, z_2) \propto \frac{1}{1+ \gamma \sin^2(z_1 \pi / 4)} \times \frac{1}{1+ \gamma \sin^2(z_2 \pi / 4)} \, .
\end{align*}
The parameter $\gamma$ controls the shape of the distribution $\mu_\gamma$ such that in the large $\gamma$ limit one recovers $\mu_\gamma \simeq \mu_\infty \eqdef \frac{1}{2} \delta_{\om^*_1} + \frac{1}{2} \delta_{\om^*_2}$.
A scatter plot of the teacher measure $\Bar{\mu}_\gamma$ and of the resulting teacher signal for $\gamma = 100$ are shown in~\cref{fig:teacher_signal_2d}.
Finally, we consider the input data $x$ to be distributed according to an empirical distribution $\Hat{\rho} = \frac{1}{N} \sum_{i=1}^N \delta_{x_i}$ with i.i.d. standard gaussian samples $x_i \sim \Nn(0, \Id)$.
In this setting we consider training the model by performing a VarPro algorithm i.e. gradient descent over the reduced risk, as in~\cref{eq:gradient_descent_SHL}.

\begin{figure}
    \centering
    \includegraphics[scale=0.85]{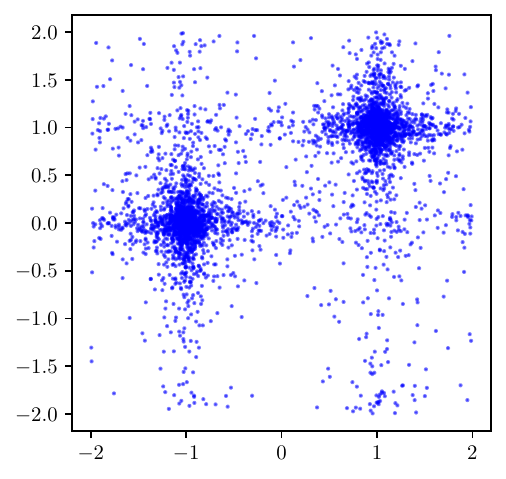}
    \includegraphics[scale=0.9,trim={10 10 0 0}]{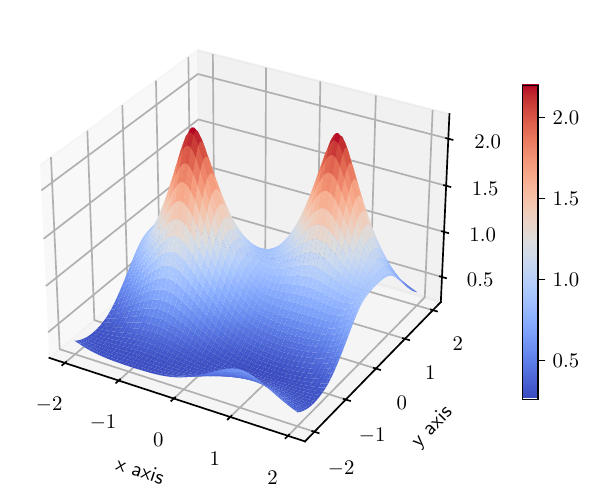}
    \caption{Left: scatter plot the empirical teacher distribution $\Bar{\mu}_\gamma$ for $\gamma = 100$. Right: corresponding target signal.}
    \label{fig:teacher_signal_2d}
\end{figure}

\paragraph{Experimental setting}

We test the performances of the VarPro algorithm (\cref{eq:gradient_descent_SHL}) for the training of a RBF neural network (\cref{eq:convolutional_model}) of varying width $M \in \{ 32, 128, 512, 1024 \}$.
We use either the ``biased'' quadratic regularization $f_b : t \mapsto \frac{1}{2} t^2$ or the ``unbiased'' quadratic regularization $f_u : t \mapsto \frac{1}{2}|t-1|^2$ and we consider varying regularization strength $\lambda \in \{10^{-1}, 10^{-2}, 10^{-3} \}$.
We consider different teacher distributions $\Bar{\mu}_\gamma$ by changing the parameter $\gamma \in \{100, +\infty \}$.
As in~\cref{subsec:numerics_1d}, we consider a number of data samples $N = 4096 \gg M$, a teacher of width $\Bar{M} = 4096 \gg M$ (s.t. the approximation $\Bar{\mu}_\gamma \simeq \mu_\gamma$ holds) and a stepsize $\tau = 2^{-10}$.

\paragraph{Student of varying width}

As before, we first investigate the role played by the width $M$ of the student in the training dynamic.
For this purpose, we fix the regularization strength to $\lambda = 10^{-3}$ and consider training RBF neural networks (\cref{eq:convolutional_model}) of varying width $M \in \{32, 128, 512, 1024 \}$ with the teacher distribution $\Bar{\mu}_\gamma$, $\gamma = 100$.

\Cref{fig:convolution_width_risk} reports evolution of the biased and unbiased reduce risk during training and~\cref{fig:convolution_width_distance} reports evolution of the distance to the teacher distribution $\Bar{\mu}_\gamma$.
As for our $1$-dimensional experiments, one can observe that the VarPro algorithm converges to lower values of the unbiased reduced risk when the width of the student increases.
In turn, at convergence, this corresponds to learned feature distributions that approximate the teacher distribution with different levels of discretization.
On the contrary, using the biased regularization $f_b : t \mapsto \frac{1}{2} t^2$ introduces a bias in the learned distribution.

\begin{figure}
    \centering
    \includegraphics[scale=1,trim={20 15 0 0}]{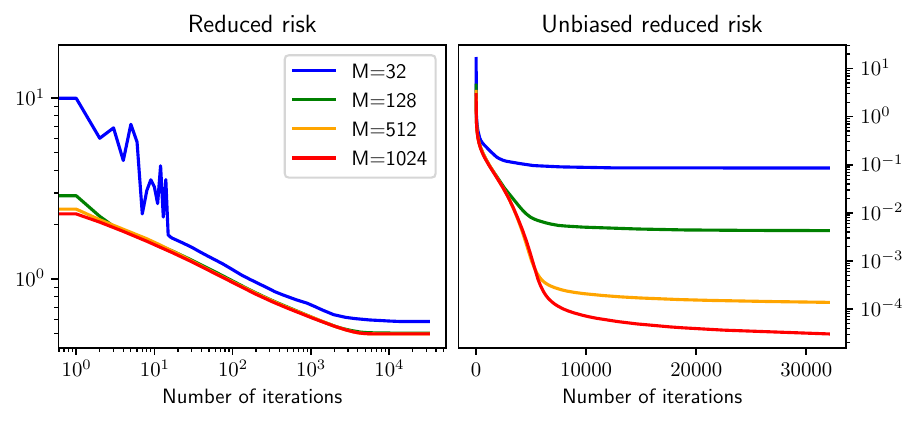}
    \caption{Evolution of the reduced risk along iterations of gradient descent for a RBF neural network (\cref{eq:convolutional_model}) of width $M \in \{32, 128, 512, 1024 \}$.
    The regularization strength is $\lambda = 10^{-3}$ and the regularization function is either $f_b : t \mapsto \frac{1}{2} t^2$ (left) or $f_u : t \mapsto \frac{1}{2} |t-1|^2$ (right). Plots are averages over $6$ independent runs.}
    \label{fig:convolution_width_risk}
\end{figure}

\begin{figure}
    \centering
    \includegraphics[scale=1,trim={10 15 0 0}]{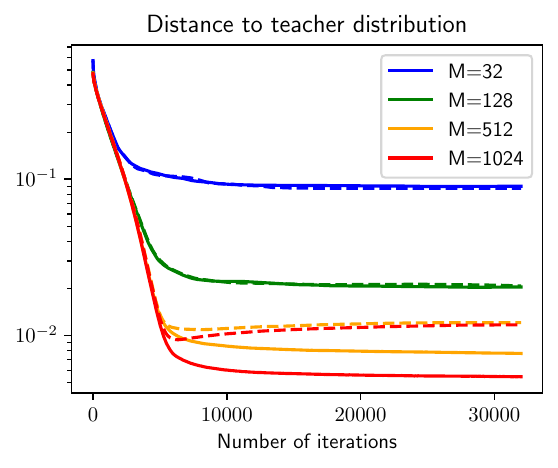}
    \caption{
    Evolution of the MMD distance to the teacher distribution $\Bar{\mu}_\gamma$ along gradient descent over the reduced risk for a RBF neural network (\cref{eq:convolutional_model}) of width $M \in \{32, 128, 512, 1024 \}$.
    The regularization strength is $\lambda = 10^{-3}$ and the regularization function is either $f_b : t \mapsto \frac{1}{2} t^2$ (dashed) or $f_u : t \mapsto \frac{1}{2} |t-1|^2$ (plain).
    Plots are averages over $6$ independent runs.}
    \label{fig:convolution_width_distance}
\end{figure}

\paragraph{Role of the regularization strength $\lambda$}

We now investigate the role of the regularization strength $\lambda > 0$.
We thus consider training RBF neural networks (\cref{eq:convolutional_model}) of fixed width $M = 1024$ with the teacher distribution $\Bar{\mu}_\gamma$, $\gamma = 100$, and we perform gradient descent over the reduced risk (\cref{eq:gradient_descent_SHL}) with varying regularization $\lambda \in \{10^{-1}, 10^{-2}, 10^{-3} \}$.
Note that here, compared to our $1$-dimensional, the case of regularization lower than $\lambda = 10^{-3}$ is numerically impracticable, at least with our choice of stepsize $\tau = 2^{-10}$.

Evolution of the distance to the teacher distribution $\Bar{\mu}_\gamma$ along training is reported in~\cref{fig:convolution_lmbda_distance}.
As in our $1$-dimensional experiments, one can observe that the convergence speed gets slower when the regularization strength increases.
There is also a significant change of behavior between $\lambda = 10^{-1}$ and $\lambda \in \{ 10^{-2}, 10^{-3}\}$.
In the former case the convergence seems to exhibit an algebraic rate, supporting the conclusions of~\cref{thm:algebraic_convergence}, while in the latter the convergence rate is linear, indicating a behavior closer to the ultra-fast diffusion limit (\cref{thm:diffusion_convergence}).
As for the $1$-dimensional case, this can be explained by the fact that $\lambda = 10^{-1}$ is the order of magnitude of the most significant eigenvalues of the tangent kernel $K_\mu$ in~\cref{eq:kernel_learning}.
Thus, for higher values of $\lambda$, one enters in a high regularization regime where the reduced risk receive more influence from the MMD distance term than from the $f$-divergence term in~\cref{eq:L_mmd}.
One also observes that the bias introduced in the case of the regularization $f_b : t \mapsto \frac{1}{2} t^2$ vanishes with the regularization strength $\lambda$, supporting the conclusions of our~\cref{prop:convergence_minimizers}.

\begin{figure}
    \centering
    \includegraphics[scale=1,trim={20 10 0 0}]{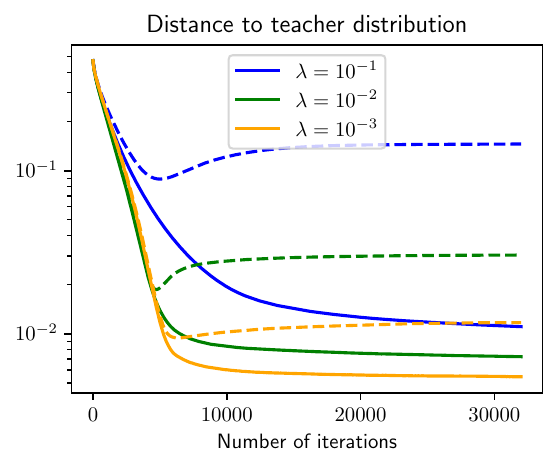}
    \caption{Evolution of the MMD distance to the teacher distribution $\Bar{\mu}_\gamma$ along gradient descent over the reduced risk (\cref{eq:gradient_descent_SHL}) for a RBF neural network (\cref{eq:convolutional_model}) of width $M=1024$ with regularization $\lambda \in \{10^{-1}, 10^{-2}, 10^{-3}, 10^{-4} \}$.
    The regularization function is either $f_b : t \mapsto \frac{1}{2} t^2$ (dashed) or $f_u : t \mapsto \frac{1}{2} |t-1|^2$ (plain).
    Plots are averages over $6$ independent runs.}
    \label{fig:convolution_lmbda_distance}
\end{figure}

\paragraph{Role of the shape of the teacher distribution}

Finally, we investigate the impact of the shape of the teacher distribution on the VarPro dynamic.
We are particularly interested in the limit $\gamma = +\infty$ in which the teacher distribution is given by $\Bar{\mu}_\gamma = \frac{1}{2} \delta_{\om^*_1} + \frac{1}{2} \delta_{\om^*_2}$.
While such setting is not covered by our theory (in particular the ultra-fast diffusion equation is not necessarily well-posed), it is of interest to see if the VarPro algorithm is able to recover sparse feature representations.

In this context, we consider teacher distributions $\Bar{\mu}_\gamma$ for $\gamma \in \{100, + \infty \}$ and train RBF neural networks (\cref{eq:convolutional_model}) of fixed width $M=1024$ with gradient descent over the reduced risk (\cref{eq:gradient_descent_SHL}) with the unbiased regularization $f_u : t \mapsto \frac{1}{2} |t-1|^2$ and $\lambda = 10^{-3}$.
Plots of the evolution of the reduced risk and of the MMD distance to the teacher distribution are reported in~\cref{fig:convolution_gamma}.
As in our $1$-dimensional experiments, one can see that the convergence speed of VarPro deteriorates when $\gamma$ increases, both in terms of convergence of the risk and in terms of convergence of the learned feature distribution to the teacher's.
In case of a sparse teacher distribution ($\gamma = +\infty$), convergence towards the teacher seems to not necessarily be governed by a linear rate.
Indeed, as the teacher distribution is not absolutely continuous, one could expect the comparison with ultra-fast diffusion dynamics to no longer hold.
However, \cref{thm:algebraic_convergence} could still apply, leading to an algebraic convergence rate.

\begin{figure}
    \centering
    %\hspace{-10em}
    \centerline{
    \includegraphics[scale=1,trim={10 15 0 0}]{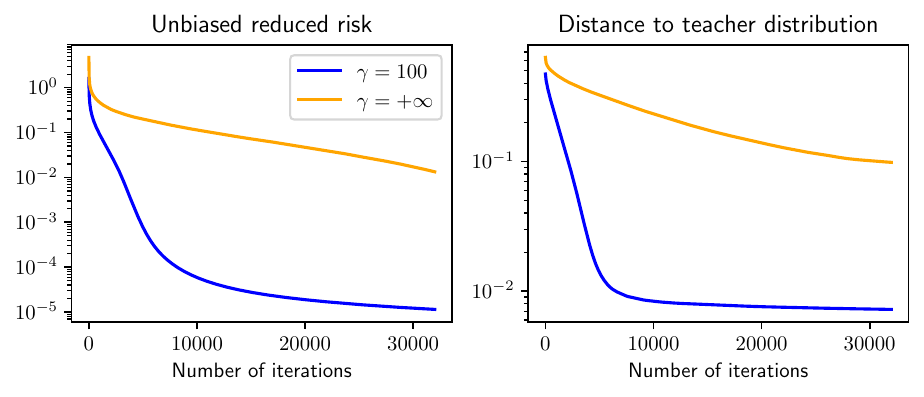}
    }
    \caption{Gradient descent over the reduced risk (\cref{eq:gradient_descent_SHL}) for a RBF neural network (\cref{eq:convolutional_model}) of width $M=1024$ with unbiased regularization $f_u : t \mapsto \frac{1}{2} |t-1|^2$, $\lambda = 10^{-3}$ and teacher distributions $\Bar{\mu}_\gamma$ for $\gamma \in \{100, +\infty \}$.
    Left: Evolution of the unbiased reduced risk.
    Right: Evolution of the MMD distance to the teacher distribution $\Bar{\mu}_\gamma$.
    Plots are averages over $6$ independent runs.}
    \label{fig:convolution_gamma}
\end{figure}